\documentclass[runningheads]{llncs}

\usepackage{xr}
\usepackage{graphicx}
\usepackage[disable]{todonotes}
\usepackage{url}
\usepackage{algorithm}
\usepackage{algorithmic}
\usepackage{caption}
\usepackage{booktabs}
\usepackage{microtype}

\usepackage{amssymb}
\usepackage{mathtools}
\newtheorem{mydef}{Definition}

\usepackage{cleveref}

\renewcommand{\vec}[1]{\mathbf{#1}}

\newcommand{\rewritten}{\todo[inline,color=green]{This section has been paraphrased}}

\newcommand{\done}{\todo[inline,color=green!50!red]{This section seems finalized}}

\newcommand{\forcameraready}[1]{}

\begin{document}

\forcameraready{
\section*{General todos}

Check chapter starting on own page where reasonable. Check also references.

\subsection{Additional Refs}
Lun Du, Yun Wang, Guojie Song, Zhicong Lu, and Junshan Wang, ‘Dynamic network embedding: An extended approach for skip-gram based network embedding.’, in IJCAI, pp. 2086–2092, (2018).

Jundong Li, Harsh Dani, Xia Hu, Jiliang Tang, Yi Chang, and Huan Liu, ‘Attributed network embedding for learning in a dynamic environment’, in Proceedings of the 2017 ACM on Conference on Information and Knowledge Management, pp. 387–396. ACM, (2017).

Linhong Zhu, Dong Guo, Junming Yin, Greg Ver Steeg, and Aram Galstyan, ‘Scalable temporal latent space inference for link prediction in dynamic social networks’, IEEE Transactions on Knowledge and Data Engineering, 28(10), 2765–2777, (2016).

Jianxin Ma, Peng Cui, and Wenwu Zhu, ‘Depthlgp: learning embeddings of out-of-sample nodes in dynamic networks’, in Thirty-Second AAAI Conference on Artificial Intelligence, (2018).

Some smoothness measure is defined in:

Mark Heimann, Haoming Shen, and Danai Koutra, ‘Node representation learning for multiple networks: The case of graph alignment’, CoRR, abs/1802.06257, (2018).

\newpage

}

\title{Updating Embeddings for Dynamic Knowledge Graphs}

\author{
Christopher Wewer\inst{1} \and
Florian Lemmerich\inst{2}\and
Michael Cochez\inst{3}
}
\authorrunning{
	C. Wewer, et al.
}
\institute{
	RWTH Aachen University, Germany\and
	University of Passau, Germany\and
	Vrije Universiteit Amsterdam, the Netherlands
}
\maketitle              %
\noindent Part of the work was performed while F. Lemmerich and M. Cochez were affiliated with RWTH Aachen and Fraunhofer FIT, respectively.

\begin{abstract}

\rewritten
Data in Knowledge Graphs often represents part of the current state of the real world. 
Thus, to stay up-to-date the graph data needs to be updated frequently.
To utilize information from Knowledge Graphs, many state-of-the-art machine learning approaches use embedding techniques.
These techniques typically compute an embedding, i.e., vector representations of the nodes as input for the main machine learning algorithm. If a graph update occurs later on --- specifically when nodes are added or removed --- the training has to be done all over again.
This is undesirable, because of the time it takes and also because downstream models which were trained with these embeddings have to be retrained if they change significantly.
In this paper, we investigate embedding updates that do not require full retraining and evaluate them in combination with various embedding models on real dynamic Knowledge Graphs covering multiple use cases.
We study approaches that place newly appearing nodes optimally according to local information, but notice that this does not work well.
However, we find that if we continue the training of the old embedding, interleaved with epochs during which we only optimize for the added and removed parts, we obtain good results in terms of typical metrics used in link prediction.
This performance is obtained much faster than with a complete retraining and hence makes it possible to maintain embeddings for dynamic Knowledge Graphs.

\end{abstract}

\section{Introduction} \label{introduction}
\rewritten

\forcameraready{A large part of the results presented in this work were part of the thesis of the first author \cite{thesisChristopher}}
In recent years, Knowledge Graphs (KGs) have become a go-to option for the representation of relational data.
While they are useful to model relationships between entities with labeled, directed edges between vertices, downstream machine learning models typically require numerical features as input.
In order to exploit these data sources, Knowledge Graph embedding methods have been devised.
However, these embedding models are static in the sense that they only consider a single fixed state of a KG.
In practice, we observe that almost every KG is evolving over time with regular additions and deletions of entities, relations, and triples; they are dynamic.
It is not desirable and in some cases infeasible to recompute the entire embedding from scratch after each minor change of the underlying graph, specifically since this would also require the retraining of all downstream models, as the new embedding would differ significantly from the old one.
Therefore, there is a high demand for an efficient and effective update procedure to initialize the vector representations of new entities and relations as well as to refresh the outdated embedding according to added and deleted triples.

A key goal for that purpose is to achieve high-quality approximations, that is, the updates obtained should perform similarly to full embedding recomputations. In this paper, we present a novel approach for updating KG embeddings with respect to changes in the graphs with a focus on this primary objective. 
First, we find that initializing vector representations of new entities and relations optimally according to local information does not maintain the link prediction performance over time.
Next, as we observe the need of negative evidence, pre-training the new vectors while fixing the remaining embedding improves the initialization.
The final approach is able to keep the quality stable over time by continuing the training of the old embedding and biasing it towards the newly added and deleted triples.
This way, we are able to update embeddings for evolving KGs much faster compared to a complete retraining, enabling the maintenance of embeddings in dynamic contexts.

While we focus on link prediction performance in this paper, we acknowledge that also other properties such as stability (smoothness of updates to preserve downstream models), low time complexity, or scalability with respect to large-scale KGs are important and provide respective experimental evaluations in the supplementary material.
We based our experiments on the OpenKE framework~\cite{openke}. All code used for our experiments can be found in our git repository\footnote{\url{https://anonymous.4open.science/r/cc361916-2988-4b1a-93c4-e96dd51278a3/}}.

The remainder of the paper is structured as follows:
After introducing static KGs and their embedding models, we formalize our notion of dynamic KGs in \cref{sec:dynamic_kgs}.
\Cref{sec:maintaining_embeddings} covers the requirements for the update procedure in detail. %
Our approach and supporting experiments are described in \cref{sec:approach}.
Finally, we provide an evaluation on real-world dynamic KGs.

\section{Preliminaries} %

\rewritten
\done

Before we move to our notion of dynamic Knowledge Graphs, we introduce (static) KGs and different embedding models.
We consider two groups of similar embedding models, whose learning procedure can be generalized such that our dynamic updating strategy can be used in combination with any of them.

\subsection{Static Knowledge Graphs}
Knowledge Graphs are a fundamental tool for the representation of relational data and nowadays widely established in practice~\cite{2003.02320}.
In a KG, relationships between entities are represented in form of labeled edges between vertices, often called triples $(h, r, t)$, where $h$ is the head entity, $r$ the relation, and $t$ the tail entity\footnote{In this work, we use terminology commonly used in link prediction work. %
}.
In this paper, we use the following general definition:
\begin{mydef}\label{KG_def}
	A Knowledge Graph $\mathcal{G}=(V, E)$ over a set of relations $\mathcal{R}$ is defined as a pair of vertices $V$ and edges $E\subseteq V\times\mathcal{R}\times V$.
\end{mydef}

\subsection{Static Embedding Methods}
As most downstream machine learning models expect numerical feature vectors as input, KGs cannot be directly used as a source of training data.
Therefore, it is common to embed the KG into a continuous vector space first, for which several embedding models have been proposed in recent years.
Also the link prediction task, which we focus on in this paper, has been successfully solved using these models.
Wang et al.~\cite{Survey} categorized two groups of embedding models: translational distance and semantic matching models.

\emph{Translational Distance (TD)} Models represent relations by a translation between participating entities; this translation is usually relation type specific.
Furthermore, they consider a triple to be likely true if the distance between a translated head entity and a potential tail entity is relatively small.
In this paper, we consider several translational models, namely \textsc{TransE}~\cite{TransE}, \textsc{TransH}~\cite{TransH}, and \textsc{TransD}~\cite{TransD}.
These methods differ in how they represent entities and how the translation is done.

\emph{Semantic Matching (SM)} Models like \textsc{DistMult}~\cite{DistMult}, \textsc{RESCAL}~\cite{RESCALpaper}, and ANALOGY \cite{Analogy} interpret the representation of entities and relations directly by regarding the latent semantics of the embedding space.
If two entities are in a relation with each other, then it is expected that their representations are correlated in a specific way.
We limit ourselves to these techniques, because they either have shown good overall performance (see \cite{pykeen}) or, in the case of \textsc{ANALOGY}, subsume other embedding models like \textsc{ComplEx} and \textsc{HolE} \forcameraready{Add references} as special cases.

The different methods are described in more detail in the supplement \cref{sup-embedding_models}.
Several more methods have been created to embed RDF graphs. Specifically, models like \textsc{RDF2Vec}~\cite{RDF2Vec} or \textsc{KGloVe}~\cite{RDFGloVe} perform or simulate random walks to extract the context of entities in the graph. These contexts are then used as the input of a language model.
Such methods do not train a link prediction objective and accordingly, the embeddings do also not perform well on that task.
Hence, they are out of scope for the current work.

\subsection{Training of Embedding Models}
To train embeddings and later evaluate them, a part of the triples in the graph is set aside to form the evaluation (also called test) set $T$.
The rest is split into the training set $S$ and the validation set $W$.
Then, computing embeddings of entities and relations is formulated as an optimization problem with regard to the training data.
There are two types of loss functions that are usually employed in combination with the considered methods.%
For translational models, we would typically use
a pairwise ranking loss
\begin{equation}\label{pairwise_ranking_loss}
	\underset{\Theta}{\mathrm{min}}\sum_{(h,r,t)\in S}\sum_{(h',r,t')\in S'} [\gamma-f_{r}(h,t)+f_{r}(h',t')]_+
\end{equation}
and for semantic matching models
the logistic loss
\begin{equation}\label{logistic_loss}
	\underset{\Theta}{\mathrm{min}}\sum_{(h,r,t)\in S\cup S'} \mathrm{log}(1+\mathrm{exp}(-y_{(h,r,t)}\cdot f_{r}(h,t)))
\end{equation}
to find the embedding $\Theta$. 

Both minimization problems use the Stochastic Local Closed World Assumption\cite{pykeen} to account for non-existing triples.
As a source of negative evidence, we sample from
\begin{equation}\label{corruptedTriples}
	S'_{(h,r,t)}=(\{(h',r,t)|h'\in V\} \cup \{(h,r,t')|t'\in V\})\setminus S %
\end{equation}
where $h'$ and $t'$ are chosen such that the newly formed triples are not in the training set\footnote{We do not employ specific hard negative sampling strategies in this work.}.
As a result, the embedding is trained to not only reproduce the positive triples in the KG but also to explicitly not mistake negative triples as positive ones.

The remaining symbols in these optimization problems can be understood as follows:
$y_{(h,r,t)}$ is $1$ if $(h,r,t)$ holds and $-1$ otherwise. 
$\gamma>0$ is a margin hyperparameter, $[x]_+ = \mathrm{max}(0, x)$ 
and $f_{r}$ the model dependent scoring function for the relation $r$ (see also~\cite{Survey}).
Intuitively, the scoring functions $f_r:V\times V \rightarrow\mathbb{R}$ give an indication on how plausible it is that $h$ is related to $t$ through relation $r$.
Hence, a higher score should be obtained if the triple is in the graph, compared to when it is not.

Finding a solution for the optimization problem is done numerically by performing stochastic gradient descent.
We start with initializing the embeddings randomly. 
Then, in each epoch, the training set is split into batches and combined with a set of negative triples.
To avoid overfitting, we evaluate the performance of the embeddings on the validation set after every $valid\_steps$ epochs, and apply early stopping if the quality does not improve for a fixed number of epochs.

Note that for several methods, there are trivial, undesired solutions to the optimization problems.
One such example is making the embedding vectors larger to reduce the loss.
To avoid these solutions, we add a soft constraint to the loss function.
Further specifics on the optimization and soft constraints can be found in the supplementary material in \cref{sup-embedding_models}.

\section{Dynamic Knowledge Graphs}\label{sec:dynamic_kgs}
\rewritten
\done
Knowledge Graphs as formalized in Definition~\ref{KG_def} are static, i.e., they represent the knowledge at a specific point in time.
However, in practice there are many applications in which the KG is evolving.
In this section, we discuss two broad use cases and define dynamic Knowledge Graphs. 

\subsection{Use Cases}\label{sec:use_cases}
One category of evolving KGs are large-scale knowledge bases that are incrementally populated by 
extracting knowledge from external sources.
Examples include DBpedia~\cite{DBpedia}, extracted from structured content on Wikipedia, or Google's Knowledge Vault~\cite{kv}, which leverages multiple extraction techniques for different forms of structured and unstructured data in order to populate the graph~\cite{GKG}.
Another similar approach is the Never-Ending Language Learner (NELL)~\cite{NEL}, whose dynamic KG we use for our evaluation.
Besides the addition of newly extracted triples, the often noisy generation of such a large-scale KG involves the deletion of manually or automatically detected false positives.

As the second use case of dynamic KGs, we consider data streams, like these coming from sensors.
These sensors include social sensing, as described in the Semantic Sensor Network Ontology\footnote{See also \url{https://www.w3.org/TR/vocab-ssn/}}; 
different kinds of user interactions can be represented as directed labeled edges in an evolving KG.
Therefore, machine learning models using data from such a KG can benefit from all relationships simultaneously.
With two datasets from the MathOverflow forum~\cite{Mathoverflow} and the Twitter social network~\cite{Higgs}, we also cover this use case in our final evaluation.

In practice, almost every KG has some kind of evolution because of initial incompleteness.
Moreover, in most cases this evolution is smooth with small changes at a time rather than drastic modifications of large subgraphs.
Therefore, it is desirable to also have corresponding smoothly evolving embeddings.

\subsection{Definition of Dynamic Knowledge Graphs}

Before we propose a method to update the embeddings of dynamic Knowledge Graphs, we first formalize them. %
To provide applicability for as many use cases as possible, we keep the definitions general and also assume a dynamic ontology, i.e., the set of relations can change over time.
\begin{mydef}
	A dynamic Knowledge Graph at time step $t$, $\mathcal{G}^{(t)}=(V^{(t)}, E^{(t)})$, with a set of relations $\mathcal{R}^{(t)}$ is defined as a pair of:
	\begin{itemize}
		\item vertices $V^{(t)}$ at time step t and
		\item edges $E^{(t)}\subseteq V^{(t)}\times \mathcal{R}^{(t)} \times V^{(t)}$ at time step t.
	\end{itemize}
\end{mydef}
Hence, we observe different states (also called snapshots) of the same KG at discrete time steps.
This discretization is necessary in order to consider the differences between two consecutive time steps, which we define as follows:
\begin{mydef}\label{definition_change}
	The change of the dynamic Knowledge Graph $\mathcal{G}^{(t+1)}=(V^{(t+1)}, E^{(t+1)})$ over the set of relations $\mathcal{R}^{(t+1)}$ at time step t+1 compared to $\mathcal{G}^{(t)}=(V^{(t)}, E^{(t)})$ over $\mathcal{R}^{(t)}$ at time step t is defined by:
	\begin{itemize}
		\item the added $\Delta V^{+}:=V^{(t+1)}\setminus V^{(t)}$ and deleted $\Delta V^{-}:=V^{(t)}\setminus V^{(t+1)}$ vertices,
		\item the added $\Delta E^{+}:=E^{(t+1)}\setminus E^{(t)}$ and deleted $\Delta E^{-}:=E^{(t)}\setminus E^{(t+1)}$ of edges,
		\item the added $\Delta \mathcal{R}^{+}:=\mathcal{R}^{(t+1)}\setminus\mathcal{R}^{(t)}$ and deleted $\Delta \mathcal{R}^{-}:=\mathcal{R}^{(t)}\setminus\mathcal{R}^{(t+1)}$ relations.

	\end{itemize}
\end{mydef}
After the observation of a new snapshot, these sets form the input of our later update procedures.

\section{Maintaining Embeddings for a Dynamic KG}\label{sec:maintaining_embeddings}%
\rewritten
\done

In this work, our goal is to update an initial KG embedding according to the graph's changes over time.
We first discuss the properties that we expect from this update. 
Since our approach was developed in an incremental fashion, we will provide intermediate validation results.
These results will show assets and drawbacks supporting our next steps.
For this, we already briefly specify the used metrics, our complexity analysis setup, and an unbiased benchmark dataset.

\subsection{Desirable Properties}
We distinguish four properties one might want an update algorithm for embeddings to have:
good quality and stability of the embeddings, a low time complexity of the algorithm, and scalability to larger graphs.
The main focus of the results included in this work is on the quality of the embeddings. 
Detailed results for stability as well as complexity and scalability are included in the supplementary material.

\subsubsection{Quality}
To measure the quality of embeddings, we use link prediction as the standard evaluation task for embedding models in previous work~\cite{Survey}.
More precisely, we perform entity prediction such that the embedding has to enable the reproduction of the correct head or tail entity given the remaining entries of a triple.
For this, we employ the following three widely used metrics:
\begin{itemize}
    \item \textbf{Mean Rank (MR)} - the average rank of the ground truth entity,
    \item \textbf{Mean Reciprocal Rank (MRR)} - the arithmetic mean of all inverse ranks
    \item \textbf{Hits@10} - the proportion of triples for which the ground truth entity is in the top 10 predictions
\end{itemize}
All of these metrics are applied in a filtered setting ignoring other correct predictions than the ground truth, as there can be multiple triples in the KG differing in only one entity.
Our updated embeddings are said to have good quality in case their link prediction performance is close to the one obtained from complete recalculations of the embedding on the updated graph. %

\subsubsection{Stability}
It is reasonable to assume that the simultaneous training of downstream models on an evolving embedding requires some form of input stability.
If the update steps would induce drastic changes to the embedding, it is unlikely that the machine learning model is still able to recognize previously learned patterns.
Therefore, we consider an embedding to be stable, if its vector representations do not change significantly between consecutive time steps.
To our knowledge, there are no techniques to specifically measure this notion of embedding stability so far.
Hence, we provide the definition of an own stability metric together with corresponding experimental results in the supplement \cref{sup-sec:stability_metric,sup-sec:stability_results}.

\subsubsection{Time Complexity}
In order to compare the computing time of our update procedure with plain recalculations of the entire embedding, we provide a time complexity analysis in the supplementary material in \cref{sup-sec:time_complexity_analysis,sup-sec:time_complexity_local_optmimum,sup-sec:time_complexity_online_method}.
Since our approach can be used in combination with several embedding models, each inducing a different time complexity, we keep our analysis general by counting the number of evaluations of the model-specific scoring function as well as the number of gradient steps during training.
Overall, the complexity analysis shows that our approach scales well in the size of the changes.

\subsubsection{Scalability}
If a graph is small, then retraining the complete embedding from scratch is something which could be done every time the graph gets updated.
So, the methods used to update embeddings should not focus on small graphs, but rather be scalable to larger ones.
The time complexity already gives an indication that our approach scales well.
The Higgs dataset (description in \cref{sec:evaluation_on_real_datasets}) is used for the scalability experiments included in the supplementary material \cref{sup-results_scalability} and supports that claim in practice.

\subsection{Dynamic Training, Validation and Test Sets}\label{sec:dynamic_train_valid_test}
The usual training of an embedding and our later update procedure have in common that they are supervised learning tasks involving separate training, validation, and test sets.
Since these sets are usually considered to be static like the KG itself, we have to adjust them to the scenario of dynamic KGs.
Therefore, we introduce a dynamic training set $S^{(t)}$, a dynamic validation set $W^{(t)}$, and a dynamic test set $T^{(t)}$ with $S^{(t)}\mathbin{\dot{\cup}}W^{(t)}\mathbin{\dot{\cup}}T^{(t)}=E^{(t)}$.
The probabilities that a new triple is added to one of these sets are given by the split proportions of the initial data in order to approximately maintain the relative sizes.
However, in the case of adding a temporarily deleted triple, it will always be assigned to the same set again such that the training, validation, and test sets are disjunct at any point in time, and no test or validation set leakage occurs.
Finally, we split the change of edges according to definition~\ref{definition_change} into $\Delta E_{S}^{+}$, $\Delta E_{W}^{+}$, and $\Delta E_{T}^{+}$ with $\Delta E_S^{+}\mathbin{\dot{\cup}}\Delta E_{W}^{+}\mathbin{\dot{\cup}}\Delta E_{T}^{+}=\Delta E^{+}$ for the additions and $\Delta E_{S}^{-}$, $\Delta E_{W}^{-}$, and $\Delta E_{T}^{-}$ with $\Delta E_{S}^{-}\mathbin{\dot{\cup}}\Delta E_{W}^{-}\mathbin{\dot{\cup}}\Delta E_{T}^{-}=\Delta E^{-}$ for the deletions.
\Cref{fig:dynamic_train_valid_test} gives an overview of the complete partitioning for two consecutive snapshots of a dynamic KG.
\begin{figure}[!h]
	\centering
	\includegraphics[width=0.75\textwidth]{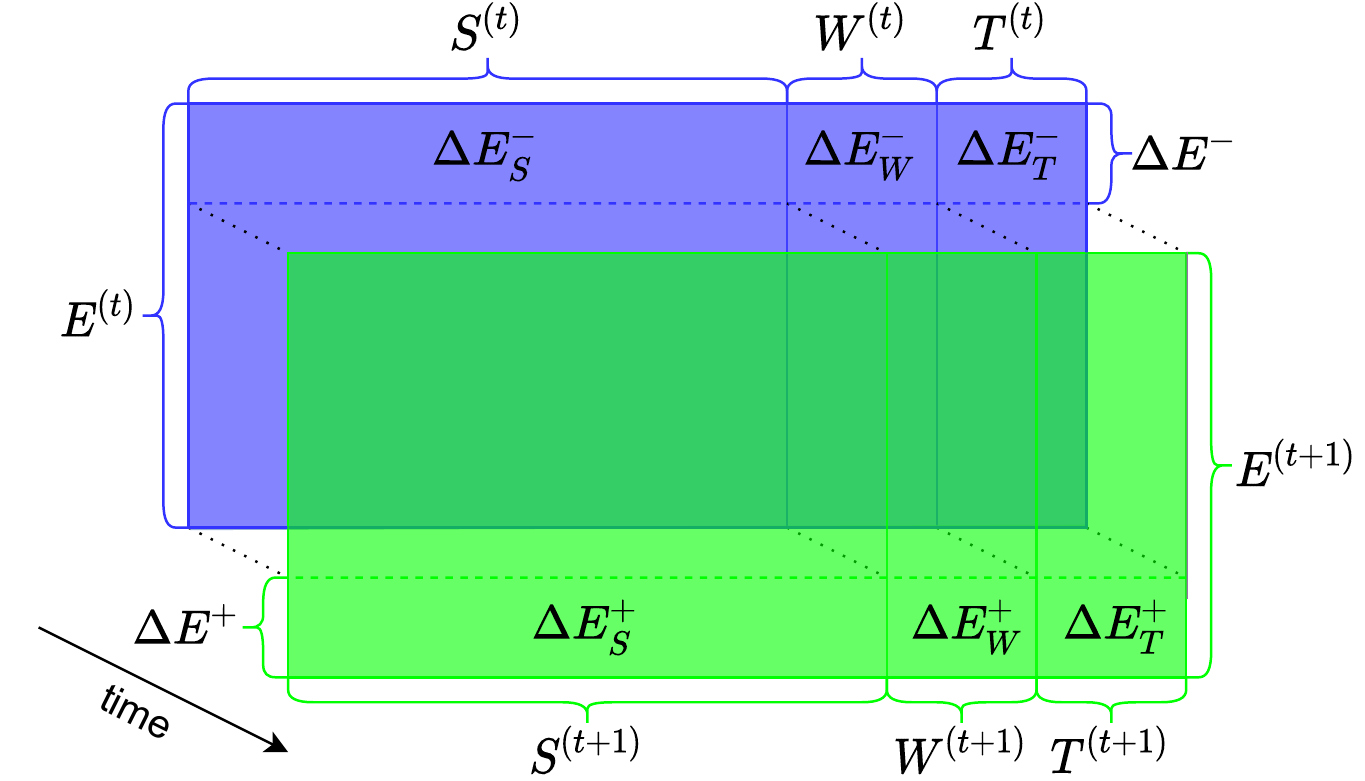}
	\caption[Visualization of the data partitioning into dynamic training, validation, and test sets]{Partitioning of the edges at time steps $t$ (blue) and $t+1$ (green). Considering the complete width, the sets of additions and deletions are shown on the left and the right side, respectively. Above and below, we can see the partitioning into dynamic training, validation, and test sets. 
	The non-overlapping boxes indicate the final split of additions and deletions into parts of training, validation, and test sets.}\label{fig:dynamic_train_valid_test}
\end{figure}

\subsection{Synthetic Dataset}

During the design of our approach, we decided to generate an artificial dynamic KG based on the widely used benchmark dataset FB15K.
We generated twenty snapshots by randomly sampling triples consisting of likewise sampled entities and relations.
This way, we obtain an unbiased dynamic KG independent of certain use cases with additions and deletions of entities, relations, and triples.
The exact generation method as well as detailed statistics about the change of this synthetic dynamic KG are given the supplementary material.
Note that this dataset was only used during the design stage. The evaluation in \cref{sec:evaluation_on_real_datasets} is based on real datasets.

\section{Our Approach to Update Embeddings}\label{sec:approach}
\rewritten
\done
As the change of a dynamic KG involves additions of new entities $\Delta V^{+}$ and relations $\Delta\mathcal{R}^{+}$, for which the outdated embedding does not contain any vector representations, we have to find an appropriate initialization for them.
This is the first part of our embedding updates.

\subsection{Initializing New Entities and Relations}
Usually, KG embeddings are initialized randomly.
In the case of dynamic KGs, we do not have to start from scratch but can leverage the current state of the embedding in order to find best fitting representations of added entities and relations (or elements for short).
To be more precise, for each new element $e$ we first need to determine the sets of informative triples $I(e)$, i.e., triples that only include the respective element as a non-embedded entry.
However, uninitialized entities can have (possibly uninitialized) relations between them.
In general, we do allow additions and deletions of whole subgraphs, which is likely to be a common case in practice, e.g., in interaction graphs.
Therefore, an appropriate initialization order is required.

Besides the set of informative triples for each new element $e$, we also determine their sets of uninformative triples $U(e)$, i.e., triples that include at least one additional non-embedded entry.
If we initialize an element, we decide to use its uninformative triples as informative triples for other added elements afterwards.
With this, an intuitive approach would be to initialize elements with many informative and few uninformative triples earlier than elements with only few informative and many uninformative ones.
We do so by using the number of informative triples divided by the number of uninformative triples as the priority of each element in a priority queue.
This way, we treat elements independently of the absolute numbers of informative and uninformative triples.
However, there are also other possible options for the priority, e.g., only the number of informative triples.
Finally, after each initialization, we update the sets and therefore the priorities for each element.

Since the informative triples for an entity $v$ differ in whether it is the head or the tail, we further split them into incoming $I(v)$ and outgoing ones $O(v)$.
While this initialization order will remain the same throughout this paper, we propose different methods for the actual initialization step in the following.
All of them rely on the sets $I(r)$ for a new relation as well as $I(v)$ and $O(v)$ for a new entity.

\subsubsection{Global Optima for Positive Triples}\label{sec:global_optima}
Regarding the scoring function, comparatively simple embedding models allow the direct computation of the best fitting initializations based on informative triples.
As an example, we consider \textsc{TransE}~\cite{TransE} in the following and provide a similar approach for \textsc{TransH}~\cite{TransH} in the supplementary material in \cref{sup-sec:global_optimum_transH}.

Since \textsc{TransE} models the existence of a triple $(h,r,t)$ in the KG by requiring the constraint $\vec{h}+\vec{r}\approx\vec{t}$ for the corresponding embeddings, it is easy to derive the desired embedding of one entry given the vector representations of the other two.
The original embedding model employs a squared error in order to enforce these constraints.
Therefore, the optimal initialization for a new relation $r$ can be computed as the arithmetic mean over the desired embeddings for the correct representation of each informative triple in $I(r)$:
\begin{equation}\label{transE_relations_positives}
	\vec{r}=\frac{1}{|I(r)|}\left(\sum_{(h,r,t)\in I(r)}\vec{t}-\vec{h}\right)
\end{equation}
For the vector initialization of a new entity $v$, we have to treat incoming $I(v)$ and outgoing  edges $O(v)$ differently in order to fulfill the constraints of the form $\vec{h}+\vec{r}\approx\vec{v}$ or $\vec{v}+\vec{r}\approx\vec{t}$, respectively.
The solution is given in \cref{transE_entities_positives}.
\begin{equation}\label{transE_entities_positives}
	\vec{v}=\frac{1}{|I(v)|+|O(v)|}\left[\left(\sum_{(h,r,v)\in I(v)}\vec{h}+\vec{r}\right)+\left(\sum_{(v,r,t)\in O(v)}\vec{t}-\vec{r}\right)\right]
\end{equation}
We validate the link prediction performance for this initialization (pos) on the synthetic dynamic KG of FB15K against a random initialization (ran) and simple averages over all entities or relations (ave) as baselines.
The results for the MR and Hits@10 metrics are shown in \cref{TransE_init}.

\begin{figure}
	\includegraphics[width=\textwidth]{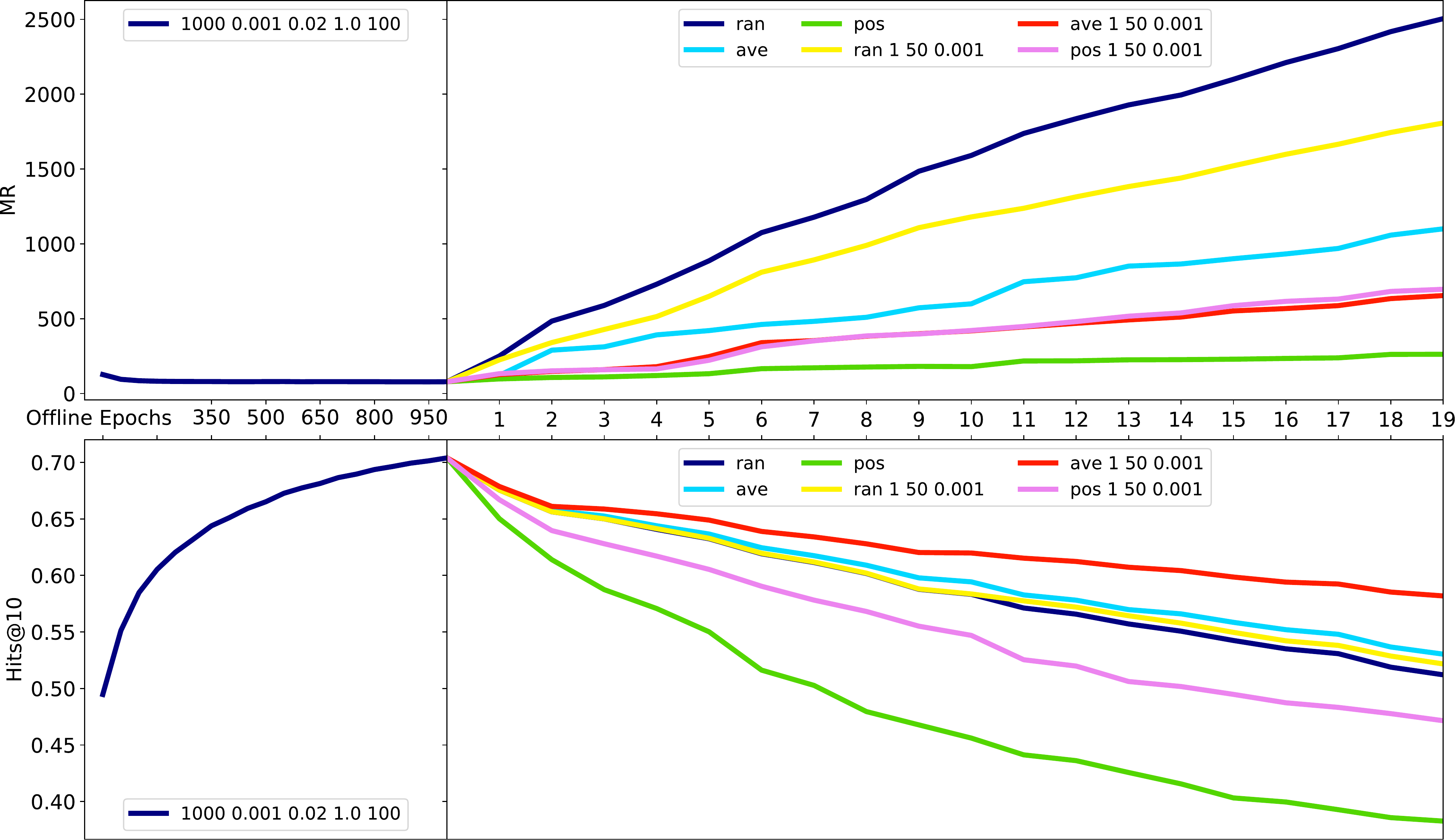}
	\caption[Link prediction performance in the MR (top) and Hits@10 metric (bottom) of all presented initialization approaches for \textsc{TransE}]{Link prediction performance in the MR (top) and Hits@10 metric (bottom) of all presented initialization approaches for \textsc{TransE}. The test scores during the training of the initial embedding are given on the left side, while the scores during the evolution of the dynamic KG are shown on the right. This general structure is used for all similar diagrams throughout this paper. We provide the detailed parameter settings for the initial embedding in the supplementary material.
    The method for settings with additional numeric parameters are explained in the following paragraphs.
} \label{TransE_init}
\end{figure}

Since the random initialization performs consistently worse than the one using averages, we only consider the latter baseline in the following.%
Our approach is able to keep the MR significantly lower than this baseline.
This is an expected behaviour, because the MR is sensitive to outliers introduced by an initialization with the average for the newly added entities and relations.

However, the comparison in terms of the Hits@10 metric indicate converse results.
Using our approach, we loose the initial prediction accuracy much faster than with the baseline initialization.
As indicated by the fast increase of the MR, the new entity and relation embeddings from the average initialization are so far away from most other vector representations that they do not affect their Hits@10 negatively.
In contrast to that, the initialization with the global optima for the positive evidence introduces new entity embeddings, which are likely to be mistaken for correct predictions.
Overall, this approach seems to distribute the loss in link prediction performance over all vector representations rather than concentrating it on the newly added ones, as it is the case for the baseline.

\subsubsection{Incorporating Negative Evidence}\label{sec:negative_evidence}
The training of the initial embedding incorporates negative samples. Hence, we experimented with different approaches to do so in the initialization as well, in order to improve the Hits@10 results.
While informative triples from the KG determine desired vector representations of new entities and relations, negative samples can indicate undesired positions in the embedding space.
To avoid wrong link predictions, new embedding vectors should be initialized ``far away'' from these positions.
However, such constraints cannot be enforced in a direct way.
Further discussion, an experiment on this, and experiments with an iterative approach, none of which could satisfactory solve the problem, can be found in the supplementary material in \cref{sup-sec:negative_evidence}.

Negative sampling introduces (soft) constraints preventing a globally optimal analytical solution to the initialization problem.
Therefore, the intuitive approach is to make use of numerical optimization, which we discuss in the next subsection.

\subsubsection{Moving Towards Local Optima}\label{sec:local_optima}
In order to apply gradient descent for the initialization, we have to start with some vector representations of the new entities and relations.
For this, we simply reuse our initial approach (pos) as well as the baseline (ave).
During the initialization of a certain element, we fix the embeddings of all other ones (indicated by underlined variables) to avoid spreading the error of our initial vector representation.
Therefore, this form of initialization can be understood as a pre-training for all added entities and relations.
For the initialization of a relation $r$, the desired number of negative samples per positive triple $n$, and an embedding model employing the pairwise ranking loss (similarly for the logistic loss), we minimize 
\begin{equation}\label{eq:train_init_rel}
	\sum_{(h,r,t)\in I(r)} \left[\mathrm{max}\left(0, \gamma -f_r(\underline{h},\underline{t})+\frac{1}{n}\sum_{(h',r,t')\in N_{(h,r,t)}}f_r(\underline{h'},\underline{t'})\right)\right]
\end{equation}
where $N_{(h,r,t)}\subseteq S'_{(h,r,t)}$ is a subset of the corrupted triples for $(h,r,t)$ from \cref{corruptedTriples} with $|N_{(h,r,t)}|=n$.
Similarly, an entity $v$ is initialized by minimizing the loss given in \cref{eq:train_init_ent}, where $I'_{(h,r,v)}$ and $O'_{(v,r,t)}$ are defined analogously to $N_{(h,r,t)}$ by only corrupting the head $h$ or the tail $t$, respectively.
\begin{align}\label{eq:train_init_ent}
	\begin{split}
		&\sum_{(h,r,v)\in I(v)} \left[\mathrm{max}\left(0, \gamma -f_{\underline{r}}(\underline{h},v)+\frac{1}{n}\sum_{(h',r,v)\in I'_{(h,r,v)}}f_{\underline{r}}(\underline{h'},v)\right)\right] \\
		+&\sum_{(v,r,t)\in O(v)} \left[\mathrm{max}\left(0, \gamma -f_{\underline{r}}(v,\underline{t})+\frac{1}{n}\sum_{(v,r,t')\in O'_{(v,r,t)}}f_{\underline{r}}(v,\underline{t'})\right)\right]
	\end{split}
\end{align}
Note that this general initialization can be used in combination with any specific embedding model, if we use the baseline average pre-initialization.

The validation results of our additional pre-training are also given in \cref{TransE_init}, indicated by the appended numbers of negative samples, epochs, and the learning rate.
We provide the technical details in the supplementary material.
While the training improves the Hits@10 results of our initial approach, it increases the MR significantly over time.
In contrast to that, the average pre-initialization benefits from the gradient descent epochs regarding both metrics.
Overall, this configuration achieved the best initialization results on our synthetic KG.

So far, we have only considered how to initialize embeddings for new entities and relations.
To complete our update procedure, we look into refreshing the old embeddings based on added and deleted triples.

\subsection{Refreshing Old Embeddings}\label{sec:updating_old_embeddings}
For the update of an old embedding according to the change of triples in the KG, we propose to bias the usual training procedure towards these additions and deletions.
Therefore, we introduce the concept of change specific epochs.

\subsubsection{Change Specific Epochs}\label{sec:change_specific_epochs}
The main difference between change specific epochs and the general ones used in the training of the initial embedding is the restriction to the added and deleted triples.
For the set of added edges $\Delta E_S^{+}$, we adopt the sampling of a corrupted triple from the usual training procedure.
However, it is not sufficient to only embed the new positive evidence. 
We have to explicitly consider the deletions as well.
Otherwise, removed triples could be mistaken for still being part of the KG.
Therefore, we first restrict the set of deleted triples $\Delta E_S^{-}$ to the subset $\{(h,r,t)\in\Delta E_S^{-} | r\notin\mathcal{R}^{-} \land h,t\notin V^{-}\}$, because removed edges induced by deletions of entities and relations cannot be predicted to be contained in the KG.
Similar to the sampling of corrupted triples, we sample a corrected triple for each removed edge $(h',r,t')$ from the set:
\begin{equation}\label{correctTriples}
	S^{(t+1)}_{(h',r,t')}=\{(h',r,t)\in S^{(t+1)}\}\cup\{(h,r,t')\in S^{(t+1)}\},
\end{equation}
If this set is empty for our negative triple, we instead directly sample from the training set $S^{(t+1)}$.
Full details of the procedure can be found in algorithm \ref{sup-alg:changespecs} in the supplementary material.

\subsubsection{Complete Update Procedure}
For a complete update, we first run our initialization for new entities and relations including the pre-training before refreshing the old embeddings.
Moreover, we carefully mix the introduced change specific epochs with general ones running over the entire training set.
We do this because otherwise the bias towards the change of the dynamic KG becomes too strong.

\section{Evaluation on Real Datasets}\label{sec:evaluation_on_real_datasets}
\rewritten
\done

For the evaluation of our approach, we have selected three real-world dynamic KGs covering the use cases from \cref{sec:use_cases}.
After a brief introduction of these datasets, we describe our final experimental setup and findings.

\subsection{Datasets}

In order to obtain a temporal sequence of snapshots, we apply the same procedure to all three datasets.
After sorting all triples by time stamp, we use a sliding window 
which always contains half of the dataset
and a stride 
of size $5\%$ 
to create twenty equal-sized snapshots.
As a result, the step from one snapshot to the next one involves the deletion of its oldest $10\%$ triples and the addition of as many new ones.
Moreover, the first and the final state of the dynamic KG are not overlapping.
Finally, each snapshot is partitioned into $90\%$ training, $5\%$ validation, and $5\%$ test data, avoiding test set leakage, as described in \cref{sec:dynamic_train_valid_test}.
We apply this procedure on three graphs containing additional time information.

\paragraph{NELL} For the use cases of incrementally generated knowledge bases, we consider the high-confidence triples\footnote{http://rtw.ml.cmu.edu/rtw/resources} extracted by the Never-Ending Language Learner (NELL)~\cite{NEL}.
We only use entities and relations occurring in at least five triples, because otherwise the KG would be too sparse for a meaningful embedding.

\paragraph{MathOverflow}The second dataset\footnote{https://snap.stanford.edu/data/sx-mathoverflow.html} contains three types of interactions between users from the MathOverflow forum (see the Stack Exchange Data Dump~\cite{Mathoverflow}).
This is an interaction graph described in the group of use cases above.

\paragraph{Higgs}In order to evaluate the scalability of our approach, we used the large-scale Higgs Twitter dataset\footnote{https://snap.stanford.edu/data/higgs-twitter.html}.
Its dynamic KG represents four types of interactions between Twitter users related to the discovery of the elusive Higgs boson.
The preprocessing details and the complete scalability results are in the supplement.

\subsection{Experimental Setup}
For the final evaluation, we run the complete update procedure with all six different embedding models on the datasets and draw observations from the behavior over time.
As a baseline for the link prediction performance, we recalculated an embedding for each snapshot from scratch (\emph{Recalc} in the figures).
This time, we only consider the Mean Reciprocal Rank (MRR) as a compromise between MR and Hits@10.

The goal of this paper is not to achieve optimal link prediction performance via an extensive hyperparameter search, but to maintain the embedding quality over the evolution of the KG.
Therefore, we adopt the general configurations used by Han et al.~\cite{openke} on the static FB15K dataset %
and refer to the work of Ali et al.~\cite{pykeen} for more dataset specific training settings.

In order to find a good combination of change specific and general epochs, we conducted a small grid search on our synthetic dynamic KG using \textsc{TransE} and \textsc{DistMult} as representatives for translational distance and semantic matching models, respectively.

\subsection{Results and Observations}

A first important observation is that the replacement of general epochs with change specific ones does only improve the embedding quality over time if we use a smaller learning rate for these epochs.
Moreover, in our best configuration, the ratio of change specific epochs ($20$) to general ones ($180$) is rather small.
These findings support our recommendation to use change specific epochs carefully.
We further observed that a good (general) learning rate during the updates is about one fifth of the learning rate used during the training of the initial embedding.
We provide the details about our parameter tuning together with a mapping from configurations (graph names) to hyperparameters in \cref{sup-sec:experimental_setup} of the supplementary material.
The results for the translational distance and semantic matching models on the dynamic KGs of NELL and MathOverflow are given in \cref{TD_NELL_MathOverflow_MRR,SM_NELL_MathOverflow_MRR}, respectively.

\paragraph{NELL} 
For the NELL dataset, our updates for \textsc{TransE} and \textsc{TransH} maintain the same embedding quality than simple recalculations.
The worse performance of \textsc{TransD} might indicate overfitting for entities with a small connectivity.
This claim is further supported by our observation of a substantially larger MR, whereas the Hits@10 score is similar for all three models.
Since our approach relies on the initial embedding, it is able to maintain its quality but not match the performance of recalculations using \textsc{TransD}.

Similar behaviors can be observed for the three semantic matching models.
While our approach is able to successfully maintain the performance of \textsc{DistMult} and \textsc{Analogy}, the recalculated embeddings using \textsc{RESCAL} outperform the updated ones significantly.%

\paragraph{MathOverflow} 
For the dynamic KG of MathOverflow, all three translational distance models achieve comparable scores.
The updates do not only keep the quality stable but follow the positive trend of the recalculated embeddings improving the MRR results over time.
However, for the semantic matching models the performance of the recalculated embeddings is improving faster and the updated ones are lagging behind.

\begin{figure}[p]
	\centering
	\centerline{\includegraphics[width=1.25\textwidth]{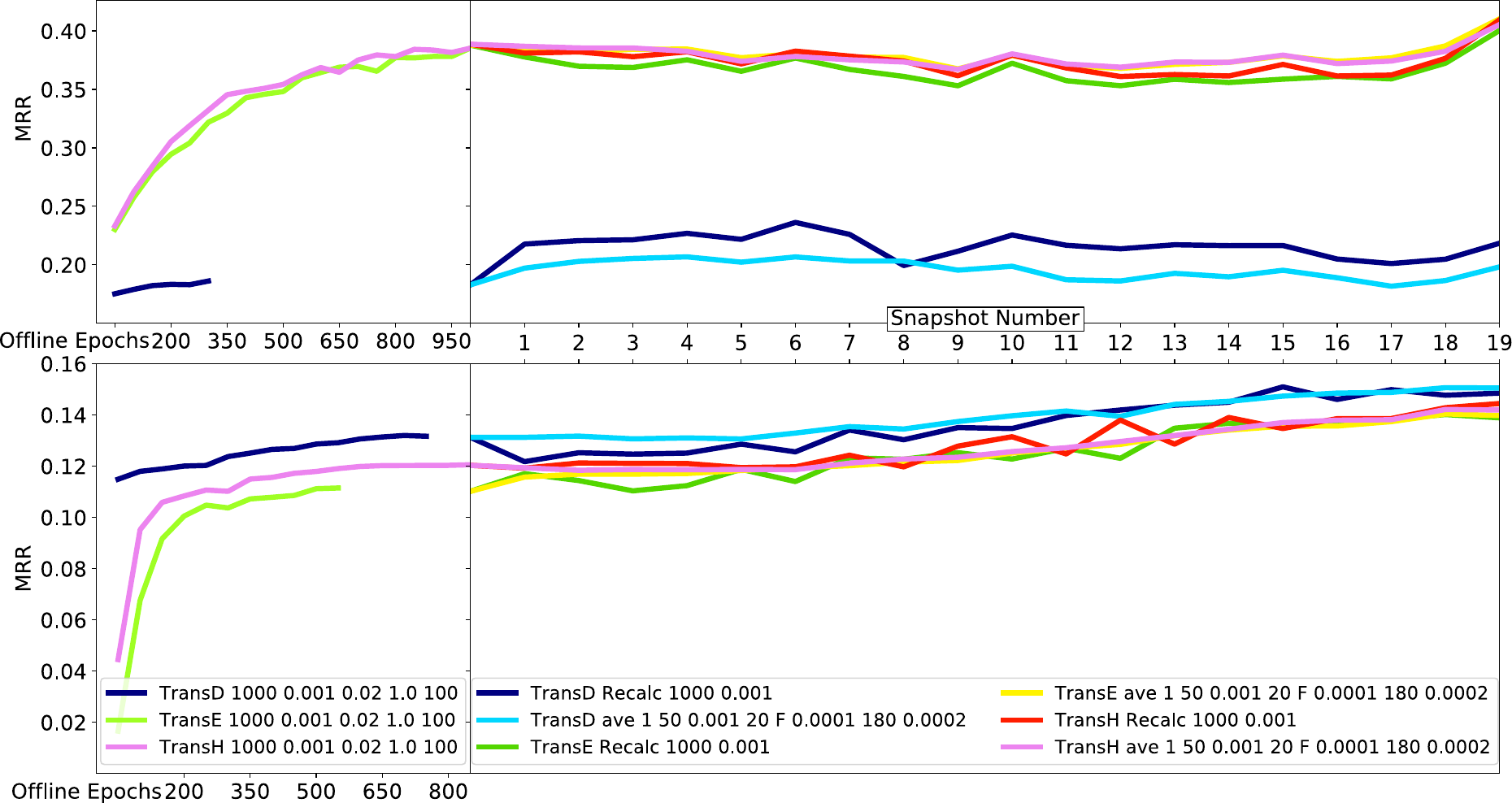}}
	\caption[Link prediction performance in the MRR metric for translational distance models on NELL (top) and MathOverflow (bottom)]{Link prediction performance in the MRR metric for translational distance models on NELL (top) and MathOverflow (bottom).	
For the initial embedding we use early stopping, hence not all methods use the same amount of epochs.
}\label{TD_NELL_MathOverflow_MRR}
\end{figure}
\begin{figure}[p]
	\centering
	\centerline{\includegraphics[width=1.25\textwidth]{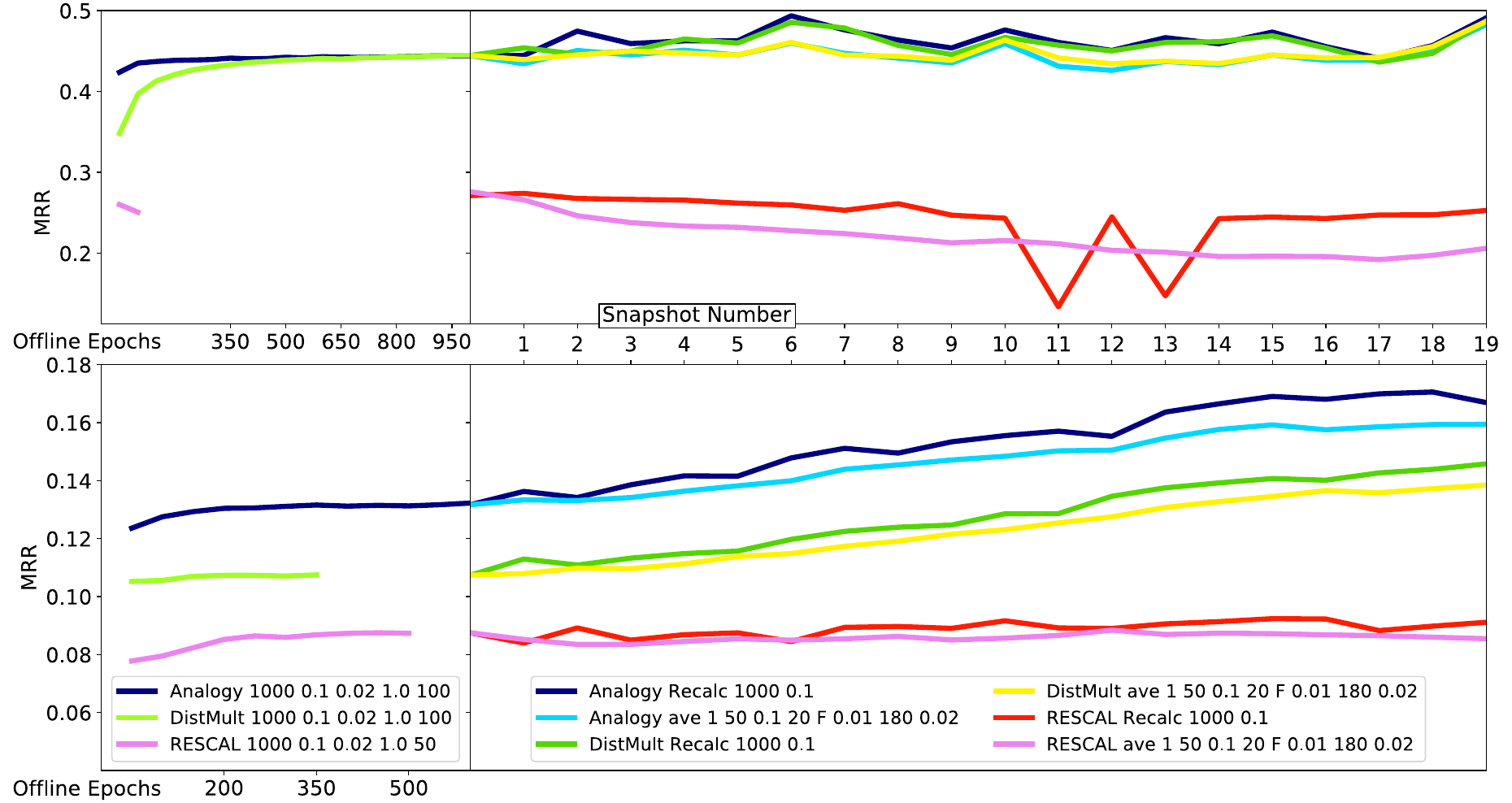}}
	\caption[Link prediction performance in the MRR metric for semantic matching models on NELL (top) and MathOverflow (bottom)]{Link prediction performance in the MRR metric for semantic matching models on NELL (top) and MathOverflow (bottom).
}\label{SM_NELL_MathOverflow_MRR}
\end{figure}

Overall, our update procedure is able to preserve the embedding quality over the evolution of real dynamic KGs, especially for translational distance models.

\todo[inline]{MC The scalability experiments could go to the supplement. These are the only ones that make use of the higgs dataset. so perhaps summarize here. }

\section{Conclusion and Outlook}
\rewritten
\forcameraready{CW Add standard acknowledgment sentence as requirement of using the RWTH Cluster in the Conclusion --> How to anonymize this?}
\todo[inline]{CW Add something about scalability although it is in the supplement? --- yes}

\todo[inline]{MC Mention that in a specific case, we could use dataset specific hyperparameter tuning to get most out of a specific dataset}

In this paper, we first extended the usual notion of static Knowledge Graphs to the much more realistic definition of dynamic Knowledge Graphs changing over time.
By allowing additions and deletions of entities, relations, and triples, our formal framework is general enough to cover many use cases in practice.
Based on these definitions, we proposed an update procedure to refresh outdated KG embeddings instead of recomputing them from scratch.
We first initialize added entities and relations and then mix the standard optimization strategy with epochs that solely train on new and removed triples.
We kept this approach general such that it can be used in combination with many translational distance and semantic matching embedding methods.
Our evaluation for six different methods on multiple real-world dynamic Knowledge Graphs showed that our approach is able to maintain the initial embedding quality over time.

We consider this work to be the first step towards enabling continuously updating embeddings and downstream models.
An initial indication that this is possible because of stability of the embeddings is part of the supplementary material, but substantial further investigation is needed.

\todo[inline]{MC This might need some further future directions.}

\forcameraready{Expand references if space permits}

\bibliographystyle{splncs04}
\bibliography{bibliography}

\end{document}


\maketitle

\begin{abstract}
This is the supplementary material for the paper `Updating Embeddings for Dynamic Knowledge Graphs'.
This supplement contains details on 
\begin{enumerate}
	\item The embedding methods used,
	\item the training and initialization algorithms for both static and dynamic knowledge graphs,
	\item a description on optimal positive initialization for TransH, the inclusion of negative evidence, and their results,
	\item the description of a more advanced iterative approach to incorporate positive and negative evidence.
	\item full detail on the synthetic dataset used for the preliminary experiments,
	\item full descriptions of the used datasets,
	\item full detail on the set of hyperparameters of the methods, 
	\item details on metrics for stability and their results,
	\item time complexity analysis of the algorithms, and
	\item results for an additional experiment demonstrating the scalability of the method.
\end{enumerate}

\end{abstract}

\tableofcontents

\section{Details on Embedding Models}\label{embedding_models}

\subsection{Translational Distance Models}

\subsubsection{TransE}
As the most representative translational distance model, \textsc{TransE}~\cite{TransE} learns one vector representation for each entity and one for each relationship.
Given a triple $(h,r,t)$ of head $h$, relation $r$ and tail $t$, \textsc{TransE} learns representations $\vec{h}, \vec{r}, \vec{t}\in\mathbb{R}^{d}$ with the constraints $\vec{h}+\vec{r}\approx\vec{t}$, if $(h,r,t)$ is in the KG, and $\vec{h}+\vec{r}$ should be far away from $\vec{t}$ otherwise.
Therefore, the scoring function of \textsc{TransE} is defined as
\begin{equation}
	f_{r}(h,t)= -||\vec{h}+\vec{r}-\vec{t}||_{1/2}, 
\end{equation}
where $||\vec{x}||_{1/2}$ denotes the $L_1$ or the $L_2$ norm of the vector $\vec{x}$.
It is a very straightforward model, as vertices are represented by points in the vector space with vectors instead of edges between them representing the relations.
However, in order to model the different relationships, every edge of a certain relation type in the KG is represented by the same vector yielding errors, which can be visualized as edges starting at the head entity but pointing not directly at the tail entity.
Training an embedding with \textsc{TransE} by minimizing the pairwise ranking loss from~\cref{main-pairwise_ranking_loss} in the main paper follows the intuition that this vector should point closer to the correct tail embedding than to another entity, which does not stand in this relation to the head.

A drawback of \textsc{TransE}'s simplicity is its lack of expressiveness, when it comes to reflexive, one-to-many, many-to-one and many-to-many relations~\cite{TransH}.
As an example for many-to-one relations, if we have the facts that Mark Zuckerberg, Andrew McCollum, Dustin Moskovitz, Eduardo Saverin and Chris Hughes are all founders of Facebook, their vector representations are constrained to be the same, since there is only one embedding for Facebook and one for the $founderOf$ relation.
Therefore, \textsc{TransE} has been extended by \textsc{TransH} to overcome those issues by improving the expressiveness while keeping its efficiency.

\subsubsection{TransH}
In order to improve \textsc{TransE}'s capacity regarding reflexive, 1-to-N, N-to-1 and N-to-N relationships, \textsc{TransH}~\cite{TransH} introduces relation-specific hyperplanes in the continuous vector space. 
While entities are still represented by a single vector, \textsc{TransH} models every relationship $r$ as a vector $\vec{r}$ on a hyperplane given by its normal vector $\vec{w}_r\in\mathbb{R}^d$.
Given a triple $(h,r,t)$, the embeddings of both entities are first projected onto the relation-specific hyperplane.
Hence, their projections can be calculated as follows:
\begin{equation}
	\vec{h}_\bot = \vec{h}-\vec{w}_r^\top\vec{h}\vec{w}_r
	\quad\mathrm{and}\quad
	\vec{t}_\bot = \vec{t}-\vec{w}_r^\top\vec{t}\vec{w}_r
\end{equation}
These projections are then required to fulfill the constraints of \textsc{TransE} leading to the following scoring function:
\begin{equation}
	f_r(h,t)=-||\vec{h}_\bot+\vec{r}-\vec{t}_\bot||_2^2
\end{equation}
Like \textsc{TransE}, \textsc{TransH} is trained by minimizing the pairwise ranking loss ~\cref{main-pairwise_ranking_loss} in the main paper.
However, in order to enforce that the embedding $\vec{r}$ of a relation $r\in\mathcal{R}$ is on its specific hyperplane, the soft constraint
\begin{equation}
	C\sum_{r\in\mathcal{R}}\left[\frac{(\vec{w}_r^\top\vec{r})^2}{||\vec{r}||_2^2}-\epsilon^2\right]_+
\end{equation}
is added to the loss function, where $C, \epsilon\in\mathbb{R}_{>0}$ are additional hyperparameters. 
As a result of introducing relation-specific hyperplanes and computing the projections of entity embeddings, \textsc{TransH} allows an entity to have different representations depending on the particular relationship. This characteristic helps in dealing with the flaws of \textsc{TransE}.
Going back to the example of a many-to-one relation in the last section about \textsc{TransE}, \textsc{TransH} only constrains the embeddings of the founders to have the same projection onto the hyperplane of the $founderOf$ relationship. Therefore, these vector representations can still differ from each other depending on the rest of the founders' relations.

In the chronological order of introduced translational distance models, \textsc{TransR}~\cite{TransR} was the next approach, which shares a similar idea with \textsc{TransH}, but introduces relation-specific spaces instead of hyperplanes. The intuition behind relation-specific spaces is that entities and relations are completely different objects and so they should be separated somehow. Because of its higher complexity, which makes it less scalable in terms of large KGs, we decided to skip this approach and include \textsc{TransD}~\cite{TransD}, which is a direct improvement of it.

\subsubsection{TransD}
\textsc{TransD} keeps the idea of relation-specific spaces by representing each entity and relation by two vectors, the first one capturing the meaning and the second one to construct projection matrices for pairs of entities and relations.
Given a triple $(h,r,t)$, \textsc{TransD} learns vectors $\vec{h}, \vec{h}_p, \vec{t}, \vec{t}_p\in\mathbb{R}^{d_V}$ and $\vec{r}, \vec{r}_p\in\mathbb{R}^{d_\mathcal{R}}$, where subscript $p$ marks the projection vectors and $d_V, d_\mathcal{R}\in\mathbb{N}$ are the dimensions of the entity or relation space, respectively.
With the help of the projection vectors, two mapping matrices $\vec{M}_{rh}, \vec{M}_{rt}\in\mathbb{R}^{d_{\mathcal{R}}\times d_{V}}$ can be defined as
\begin{equation}
	\vec{M}_{rh}=\vec{r}_p\vec{h}_p^\top+\vec{I}^{d_\mathcal{R}\times d_V}
	\quad\mathrm{and}\quad
	\vec{M}_{rt}=\vec{r}_p\vec{t}_p^\top+\vec{I}^{d_\mathcal{R}\times d_V}, 
\end{equation}
where $\vec{I}^{d_\mathcal{R}\times d_V}\in\mathbb{R}^{d_\mathcal{R}\times d_V}$ is the identity matrix, i.e. $\vec{I}^{d\times d}$ with $d=\mathrm{min}(d_{\mathcal{R}},d_{V})$ extended with zeros in the direction of the larger dimension.
Therefore, each pair of an entity and a relation has its own specific matrix, which is more reasonable than projecting each entity the same way as for \textsc{TransH} considering a certain relation.
For the example of $(\texttt{mark\_zuckerberg, founderOf, facebook})$, the head and the tail entity contain different types and attributes like being a human or a company, which could be more or less important for the $founderOf$ relation.
With the mapping matrices, the entity embeddings are projected as follows
\begin{equation}
	\vec{h}_\bot = \vec{M}_{rh}\vec{h}
	\quad\mathrm{and}\quad
	\vec{t}_\bot = \vec{M}_{rt}\vec{t}
\end{equation}
and normalized afterwards.
Finally, the scoring function is defined in the same manner as for \textsc{TransH}.
As for both other translational distance models, the KG embedding is trained by minimizing the pairwise ranking loss.

\subsection{Semantic Matching Models}

\subsubsection{DistMult}
\textsc{DistMult}~\cite{DistMult} is arguably the simplest semantic matching model.
Like \textsc{TransE}, \textsc{DistMult} represents each entity and relation by a single vector.
For a triple $(h,r,t)$, the scoring function is defined as
\begin{equation}
	f_r(h.t)=\vec{h}^\top\mathrm{diag}(\vec{r}) \vec{t} = \sum_{i=0}^{d-1}[\vec{r}]_{i}\cdot [\vec{h}]_{i}\cdot [\vec{t}]_{i}, 
\end{equation}
where $[\vec{x}]_{i}$ denotes the $i$-th entry of the vector $\vec{x}$.
The intuition is that the components of the relation embedding $\vec{r}$ describe, whether the corresponding entries in the entity embeddings $\vec{h}$ and $\vec{t}$ should match or not.
If $[\vec{r}]_{i}$ is positive, $[\vec{h}]_{i}$ and $[\vec{t}]_{i}$ should be quite similar such that the product of the entities' components and therefore the product of all three components is positive, whereas a negative entry in the relation representation constrains the entities to have different values (i.e. different signs) for that latent component.
However, since $\vec{h}^\top\mathrm{diag}(\vec{r}) \vec{t}=\vec{t}^\top\mathrm{diag}(\vec{r}) \vec{h}\ \forall\vec{h},\vec{r},\vec{t}\in\mathbb{R}^{d}$, \textsc{DistMult} models all relations as symmetric ones, which is an obvious drawback in the case of KGs.
\textsc{DistMult} is trained by optimizing the logistic loss.

\subsubsection{RESCAL}
\textsc{RESCAL}~\cite{RESCALslides} %
is a generalization of \textsc{DistMult} without the restriction of a diagonal matrix for the relations.
Therefore, each entity $v$ is again represented by a vector $\vec{v}\in\mathbb{R}^d$, whereas a relation $r$ is associated with a matrix $\vec{M}_{r}\in\mathbb{R}^{d\times d}$.
The scoring function for a triple $(h,r,t)$, which is plugged into the pairwise-ranking loss in the offline method, is then defined equally to the one of \textsc{DistMult}:
\begin{equation}
	f_r(h,t) = \vec{h}^\top\vec{M}_{r}\vec{t}=\sum_{i=0}^{d-1}\sum_{j=0}^{d-1}[\vec{M}_{r}]_{ij}\cdot [\vec{h}]_{i} \cdot [\vec{t}]_{j}
\end{equation}
On the one hand, \textsc{RESCAL} allows the modeling of non-symmetric relations, but on the other hand, it introduces $\mathcal{O}(d^2)$ parameters to learn for a single relation, whereas all previously described embedding models only require $\mathcal{O}(d)$ parameters.

\subsubsection{Analogy}
\textsc{Analogy}~\cite{Analogy} is an extension of \textsc{RESCAL} with the purpose to model analogical properties of entities and relations~\cite{Survey}.
A simple analogy is for example ``\texttt{mark\_zuckerberg} is to \texttt{facebook} as \texttt{bill\_gates} is to \texttt{microsoft}''.
If there are the same paths between the vertices for Mark Zuckerberg and facebook as between the ones for Bill Gates and Microsoft in a KG, it is likely to say that Mark Zuckerberg has a similar relation to facebook as Bill Gates to Microsoft.
Moreover, analogical structures rely on the compositional equivalence of relations, i.e. the exact order of relations on a certain path in the KG is irrelevant~\cite{Analogy}.
As relations are still represented by linear maps, this additional property requires some further constraints for the matrices.
While the scoring function is the same as for \textsc{RESCAL}
\begin{equation}
	f_r(h,t)=\vec{h}^\top\vec{M}_r\vec{t}
\end{equation}
with $\vec{h},\vec{t}\in\mathbb{R}^{d}$ and $\vec{M}_r\in\mathbb{R}^{d\times d}$, \textsc{Analogy} restricts the linear maps of relations to a normal and commutative family of matrices.
The normality and mutually commutativity are defined as follows:
\begin{equation}
	\begin{split}
		\mathrm{normality}: \vec{M}_r\vec{M}_r^\top &= \vec{M}_r^\top\vec{M}_r, \forall r\in\mathcal{R} \\
		\mathrm{commutativity}: \vec{M}_r\vec{M}_{r'} &= \vec{M}_{r'}\vec{M}_r, \forall r,r'\in\mathcal{R}
	\end{split}
\end{equation}
It has been shown that all normal matrices of a commutative family can be simultaneously block-diagonalized such that the resulting almost-diagonal matrices enable a low complexity of $\mathcal{O}(d)$ to compute the gradients for the optimization of the logistic loss.
Therefore, \textsc{Analogy} is computational cheaper than \textsc{RESCAL}.
Furthermore, \textsc{Analogy} subsumes \textsc{DistMult} as well as further embedding models \textsc{ComplEx}~\cite{ComplEx} and \textsc{HolE}~\cite{HolE} as special cases, which is the reason why we leave out the two latter ones in this paper.
\textsc{DistMult} is a restriction of \textsc{Analogy} to diagonal matrices, which are both normal and mutually commutative~\cite{Analogy}.
\textsc{ComplEx} employs a complex embedding space $\mathbb{C}^{d}$ and can be represented by \textsc{Analogy} with an embedding dimension $2d$ and a restriction of the relation matrices to a certain block-diagonal subgroup.
\textsc{HolE} incorporates circulant matrices for its relation representations, which again are normal and mutually commutative.

\subsection{Learning TDSM Offline}

The detailed training algorithm for offline models is given in algorithm \ref{TDSM_offline}.

\begin{algorithm}[!h]
	\caption{Learning TDSM (offline)} \label{TDSM_offline}
	\begin{algorithmic}[1]
		\REQUIRE Training set $S$, validation set $W$, entities $V$, relation set $\mathcal{R}$, margin $\gamma$, number of epochs $num\_epoch$, number of batches $num\_batch$, validation steps $valid\_steps$
		\STATE \textbf{initialize} $\vec{r} \leftarrow$ uniform($-\sqrt{\frac{6}{|\mathcal{R}|+d}}$, $\sqrt{\frac{6}{|\mathcal{R}|+d}}$) for each $r\in\mathcal{R}$
		\STATE \textcolor{white}{\textbf{initialize}} $\vec{v} \leftarrow$ uniform($-\sqrt{\frac{6}{|V|+d}}$, $\sqrt{\frac{6}{|V|+d}}$) for each entity $v\in V$
		\FORALL[$num\_epoch$ training steps]{$epoch=1, ..., num\_epoch$}
		\STATE $S \leftarrow$ shuffle($S$) \label{start_general}
		\STATE $batches \leftarrow$ divide($S, num\_batch$) \COMMENT{divide into $num\_batch$ batches}
		\FORALL{$S_{batch}\in batches$}
		\STATE $T_{batch}\leftarrow\emptyset$  \COMMENT{initialize the set of pairs of triples}
		\FORALL{$(h,r,t)\in S_{batch}$}
		\STATE $(h',r,t')\leftarrow$ sample($S_{(h,r,t)}^{\prime}$) \COMMENT{sample a corrupted triple (see \cref{main-corruptedTriples} in the main paper)}
		\STATE $T_{batch}\leftarrow T_{batch}\cup \{((h,r,t), (h',r,t'))\}$
		\ENDFOR
		\STATE Update embeddings w.r.t. \[\sum_{((h,r,t), (h',r,t'))\in T_{batch}}\nabla\left[\mathrm{log}(1+\mathrm{exp}(-f_r(h,t)))+\mathrm{log}(1+\mathrm{exp}(f_r(h',t')))\right]\] or \[\sum_{((h,r,t), (h',r,t'))\in T_{batch}}\nabla\left[\gamma - f_{r}(h,t) + f_{r}(h',t')\right]_+\]\label{TDSM_offline_optimization}
		\ENDFOR
		\IF{$valid\_steps \mid epoch$}
		\STATE validate($\Theta, W$) \COMMENT{validate every $valid\_steps$ epochs}
		\ENDIF
		\ENDFOR
	\end{algorithmic}
\end{algorithm}

\section{Maintaining Embeddings of a Dynamic KG}\label{chap:Evaluation_Tasks_and_Synthetic_Dynamic_Knowledge_Graphs}
In this section, we introduce further metrics and provide additional details on our evaluation setup.

\subsection{Triple Classification}
Another evaluation task to measure the quality of an embedding is triple classification, which is the binary classification task of deciding, whether a given triple $(h,r,t)$ is correct or not~\cite{TransD}. Since the validation and test sets do not contain any negative triples, they have to be constructed.
Therefore, the same setting as in~\cite{triple_classification} is used, but only the tail entity is replaced for each triple.
To judge, whether the given triple $(h,r,t)$ is correct or not, its score $f_r(h,t)$ is again calculated and the triple is considered to be correct, if and only if its score is larger than a certain relation-specific threshold $\delta_r$~\cite{TransD}.
As the only metric, we employ the average classification accuracy, i.e. the proportion of correct classified triples.
The thresholds $\delta_r\ \forall r\in\mathcal{R}$ are obtained by maximizing the classification accuracy for each relation on the validation set.

\subsection{Stability Metric}\label{sec:stability_metric}
Therefore, we define the Normalized Mean Change (NMC) for the entity representations (analogously for relations) of two consecutive embeddings $\Theta^{(t)}$ and $\Theta^{(t+1)}$ of a dynamic KG $\mathcal{G}$ as
\begin{equation}\label{NMC}
	NMC(t+1)=\frac{1}{dist_\varnothing}\sum_{v\in V_{sh}}\frac{\mathrm{dist}(\vec{v}^{(t+1)}, \vec{v}^{(t)})}{\sum_{w\in V_{sh}}\mathrm{dist}(\vec{v}^{(t+1)}, \vec{w}^{(t+1)})}
\end{equation}
where $V_{sh}=V^{(t+1)}\cap V^{(t)}$ is the set of shared entities at both time steps, $\vec{v}^{(t)}$ the embedding of entity $v$ at time step $t$ and \[dist_\varnothing=\frac{1}{|V_{sh}|}\sum_{(v, w)\in V_{sh}\times V_{sh}}\mathrm{dist}(\vec{v}^{(t+1)}, \vec{w}^{(t+1)})\] the average pairwise distance between shared entities with respect to the distance metric $\mathrm{dist}:\mathbb{R}^{d}\times \mathbb{R}^{d}\rightarrow\mathbb{R}$. 

The NMC describes an average relative distance between the vector representations before and after an update of the embedding.
Since these distances can vary significantly due to local clusters of semantically similar entities inside a KG embedding, we introduce a local and a global normalization factor.
For this, we first normalize the distances for each individual entity by the mean distance to all other entities after the update, before multiplying the global normalization factor $dist_\varnothing$ to enable comparisons between different embeddings.
\Cref{NMC} is already in reduced form, as the fractions $\frac{1}{|V_{sh}|}$ in the denominator of each summand and before the outer sum cancel each other out.
While the metric can be used in combination with any distance measure, we employ the standard Euclidean distance.

\subsection{Time Complexity Analysis}
\label{sec:time_complexity_analysis}
In order to check, whether our approach fulfills the third property, which we aim for, we need to analyze the time complexity.
Since the whole approach combines several algorithms, we do this step-wise and summarize the final time complexity of the online method in the end.
However, to provide comparability to the complexity of the offline method, the analysis has to be independent of the actual embedding method.
Therefore, we need to count a suitable type of operations.
As we ignore the intermediate validation steps as well as the sampling of corrupted triples, there are basically two operations, which determine the runtime of our system: evaluations of the scoring function and the backpropagations.
While the evaluation of the scoring function for a certain triple and a fixed embedding model always consumes the same time, the time for a backpropagation further depends on the number of parameters (or the size of the gradients), which is proportional to the number of entities and relations, whose vector representations are trained (except for relations in \textsc{RESCAL}).
We use the term gradient step for the update of one vector representation and the term backpropagation for a combination of many gradient steps in a parallel fashion.
Hence, we provide two worst-case complexities separately: one for the evaluations of the scoring functions and the other one for the gradient steps.
Besides the complexity, we also compare the runtime of our approach to simple recalculations for a large-scale dynamic KG in the extensive evaluation in \cref{sec:runtime}.

\subsection{Synthetic Dynamic Knowledge Graphs}
In order to validate our approaches and draw conclusions for the next improvements, we decided to create a synthetic dynamic KG for the purpose of a general and unbiased data set.
The idea is to have a general direction about the performance of our approaches independent of the particular use case.
Therefore, we use a static KG as the one of FB15K, which has been used as the benchmark for link prediction for all covered embedding models, and turn it into a dynamic one.
We generate twenty snapshots by sampling $99.5\%$ of all entities and relations and taking $95\%$ of all triples only consisting of these entities and relations for each single snapshot.
As a result, we obtain a dynamic KG with additions and deletions of entities, relations, subgraphs consisting of the new entities and relations as well as triples without a new entity or relation.
The detailed statistics about the differences between consecutive snapshots are given in table~\ref{FB15K_20_0.995rel_0.995ent_0.95}.
On average, every snapshot contains $582,995$ triples with $1,336$ relations and $14,874$ entities separated into $81.6\%$ training, $8.4\%$ validation and $10.0\%$ test data, as this is the distribution of the original FB15K benchmark.

\begin{table}[!h]
	\caption[Statistics of FB15K $20\_0.995\mathrm{rel}\_0.995\mathrm{ent}\_0.95$]{The cardinality of changes between two consecutive snapshots of the synthetic dynamic KG $20\_0.995\mathrm{rel}\_0.995\mathrm{ent}\_0.95$ on FB15K} \label{FB15K_20_0.995rel_0.995ent_0.95}
	\resizebox{\textwidth}{!}{
		\begin{tabular}{|c||cc|cc|cc|cc|cc|}
			\hline
			\textbf{Step} & $|\Delta E^{+}_{S}|$ & $|\Delta E^{-}_{S}|$ & $|\Delta E^{+}_{W}|$ & $|\Delta E^{-}_{W}|$ & $|\Delta E^{+}_{T}|$ & $|\Delta E^{-}_{T}|$ & $|\Delta\mathcal{R}^{+}|$ & $|\Delta\mathcal{R}^{-}|$ & $|\Delta V^{+}|$ & $|\Delta V^{-}|$ \\
			\hline
			\hline
			$0-1$ & $29{,}207$ & $34{,}805$ & $3{,}132$ & $3{,}580$ & $3{,}605$ & $4{,}207$ & $9$ & $7$ & $75$ & $76$ \\
			$1-2$ & $35{,}155$ & $26{,}969$ & $3{,}611$ & $2{,}775$ & $4{,}203$ & $3{,}350$ & $7$ & $9$ & $76$ & $76$ \\
			$2-3$ & $27{,}169$ & $29{,}400$ & $2{,}808$ & $2{,}983$ & $3{,}387$ & $3{,}621$ & $8$ & $6$ & $75$ & $76$ \\
			$3-4$ & $29{,}217$ & $29{,}271$ & $2{,}980$ & $2{,}996$ & $3{,}608$ & $3{,}710$ & $7$ & $8$ & $76$ & $75$ \\
			$4-5$ & $29{,}261$ & $31{,}725$ & $3{,}001$ & $3{,}266$ & $3{,}711$ & $3{,}996$ & $8$ & $7$ & $76$ & $76$ \\
			$5-6$ & $31{,}898$ & $27{,}680$ & $3{,}253$ & $2{,}859$ & $4{,}007$ & $3{,}411$ & $7$ & $9$ & $76$ & $76$ \\
			$6-7$ & $27{,}801$ & $28{,}714$ & $2{,}864$ & $2{,}969$ & $3{,}438$ & $3{,}545$ & $9$ & $8$ & $75$ & $75$ \\
			$7-8$ & $28{,}133$ & $32{,}230$ & $2{,}930$ & $3{,}311$ & $3{,}462$ & $3{,}846$ & $7$ & $8$ & $76$ & $79$ \\
			$8-9$ & $32{,}591$ & $27{,}528$ & $3{,}358$ & $2{,}758$ & $3{,}921$ & $3{,}299$ & $9$ & $10$ & $79$ & $77$ \\
			$9-10$ & $27{,}525$ & $32{,}000$ & $2{,}752$ & $3{,}195$ & $3{,}290$ & $3{,}857$ & $10$ & $7$ & $77$ & $77$ \\
			$10-11$ & $31{,}926$ & $27{,}350$ & $3{,}188$ & $2{,}834$ & $3{,}856$ & $3{,}357$ & $6$ & $8$ & $75$ & $79$ \\
			$11-12$ & $27{,}680$ & $29{,}573$ & $2{,}885$ & $3{,}084$ & $3{,}385$ & $3{,}516$ & $9$ & $8$ & $80$ & $76$ \\
			$12-13$ & $29{,}473$ & $27{,}819$ & $3{,}035$ & $2{,}797$ & $3{,}527$ & $3{,}403$ & $8$ & $9$ & $76$ & $74$ \\
			$13-14$ & $27{,}816$ & $28{,}070$ & $2{,}823$ & $2{,}926$ & $3{,}386$ & $3{,}406$ & $9$ & $7$ & $74$ & $77$ \\
			$14-15$ & $28{,}071$ & $29{,}570$ & $2{,}913$ & $3{,}044$ & $3{,}434$ & $3{,}593$ & $7$ & $7$ & $77$ & $76$ \\ 
			$15-16$ & $29{,}513$ & $27{,}403$ & $3{,}047$ & $2{,}745$ & $3{,}551$ & $3{,}369$ & $7$ & $13$ & $76$ & $74$ \\
			$16-17$ & $27{,}387$ & $30{,}110$ & $2{,}740$ & $3{,}203$ & $3{,}375$ & $3{,}754$ & $13$ & $7$ & $74$ & $74$ \\
			$17-18$ & $30{,}034$ & $29{,}194$ & $3{,}181$ & $3{,}015$ & $3{,}731$ & $3{,}617$ & $7$ & $12$ & $75$ & $76$ \\
			$18-19$ & $29{,}282$ & $26{,}526$ & $3{,}043$ & $2{,}769$ & $3{,}646$ & $3{,}198$ & $12$ & $9$ & $75$ & $77$ \\
			\hline
		\end{tabular}
	}
\end{table}

\newpage
\section{Our Approach to Update Embeddings}
In this section, we provide further details about our approach as well as more concepts for the initialization of new entity and relation embeddings.
\subsection{Initializing New Entities and Relations}
The pseudocode for our proposed initialization order is given in \cref{TDSM_init}.
\begin{algorithm}
	\caption{Initialization Order for TDSM} \label{TDSM_init}
	\begin{algorithmic}[1]
		\REQUIRE Training set $S^{(t+1)}\subset E^{(t+1)}$ at time $t+1$ and the changes $\Delta V^{+}$ and $\Delta \mathcal{R}^{+}$
		\STATE $Q\leftarrow\emptyset$\COMMENT{initialize the priority queue $Q$}
		\FORALL[insert all new relations into $Q$]{$r\in\Delta\mathcal{R}^{+}$}
		\STATE $I(r)\leftarrow\{(h,r,t)\in S^{(t+1)} | h,t\notin\Delta V^{+}\}$%
		\STATE $U(r)\leftarrow\{(h,r,t)\in S^{(t+1)} | h\in\Delta V^{+} \lor t\in\Delta V^{+}\}$%
		\STATE $\mathrm{insert}(Q,(r,\frac{|I(r)|}{|U(r)|+\epsilon}))$%
		\ENDFOR
		\FORALL[insert all new entities into $Q$]{$v\in\Delta V^{+}$}
		\STATE $I(v)\leftarrow \{(h,r,v)\in S^{(t+1)} | h\notin\Delta V^{+} \land r\notin\Delta\mathcal{R}^{+}\}$%
		\STATE $O(v)\leftarrow \{(v,r,t)\in S^{(t+1)} | r\notin\Delta\mathcal{R}^{+} \land t\notin\Delta V^{+}\}$%
		\STATE $U(v)\leftarrow \{(h,r,t)\in S^{(t+1)} | h=v \lor t=v\}\setminus(I(v)\cup O(v))$
		\STATE $\mathrm{insert}(Q,(v,\frac{|I(v)|+|O(v)|}{|U(v)|+\epsilon}))$%
		\ENDFOR
		\WHILE[initialize each element in $Q$]{$Q\neq\emptyset$}
		\STATE $x\leftarrow\mathrm{extractMax}(Q)$\COMMENT{take item with highest priority}\label{extractMin}
		\IF[$x$ is a relation]{$x\in\Delta\mathcal{R}^{+}$}
		\IF{$I(x)\neq\emptyset$}
		\STATE initRelation($x$, $I(x)$)\COMMENT{initialize $\vec{x}$}\label{TDSM_init_relation}
		\FORALL[update triple sets of entities in $Q$]{$v\in Q\cap\Delta V^{+}$}
		\STATE $I(v)\leftarrow I(v)\cup\{(h,x,v)\in S^{(t+1)} | h\notin Q\}$
		\STATE $O(v)\leftarrow O(v)\cup\{(v,x,t)\in S^{(t+1)} | t\notin Q\}$
		\ENDFOR
		\ENDIF
		\FORALL[update priority of entities in $Q$]{$v\in Q\cap\Delta V^{+}$}
		\STATE $U(v)\leftarrow U(v)\setminus (\{(h,x,v)\in S^{(t+1)} | h\notin Q\}\cup\{(v,x,t)\in S^{(t+1)} | t\notin Q\})$
		\STATE $\mathrm{updatePrio}(Q,(v,\frac{|I(v)|+|O(v)|}{|U(v)|+\epsilon}))$%
		\ENDFOR
		\ELSE[$x$ is an entity]
		\IF{$I(x)\cup O(x)\neq\emptyset$}
		\STATE initEntity($x$, $I(x)$, $O(x)$)\COMMENT{initialize $\vec{x}$}\label{TDSM_init_entity}
		\FORALL[update triple sets of entities in $Q$]{$v\in Q\cap\Delta V^{+}$}
		\STATE $I(v)\leftarrow I(v)\cup\{(x,r,v)\in S^{(t+1)} | r\notin Q\}$%
		\STATE $O(v)\leftarrow O(v)\cup\{(v,r,x)\in S^{(t+1)} | r\notin Q\}$%
		\ENDFOR
		\FORALL[update triple sets of relations in $Q$]{$r\in Q\cap\Delta\mathcal{R}^{+}$}
		\STATE $I(r)\leftarrow I(r)\cup\{(h,r,x)\in S^{(t+1)} | x\notin Q\}\cup\{(x,r,t)\in S^{(t+1)} | t\notin Q\}$
		\ENDFOR
		\ENDIF
		\FORALL[update priority of entities in $Q$]{$v\in Q\cap\Delta V^{+}$}
		\STATE $U(v)\leftarrow U(v)\setminus (\{(x,r,v)\in S^{(t+1)} | r\notin Q\}\cup\{(v,r,x)\in S^{(t+1)} | r\notin Q\})$
		\STATE $\mathrm{updatePrio}(Q,(v,\frac{|I(v)|+|O(v)|}{|U(v)|+\epsilon}))$%
		\ENDFOR
		\FORALL[update priority of relations in $Q$]{$r\in Q\cap\Delta\mathcal{R}^{+}$}
		\STATE $U(r)\leftarrow U(r)\setminus (\{(h,r,x)\in S^{(t+1)} | h\notin Q\}\cup\{(x,r,t)\in S^{(t+1)} | t\notin Q\})$
		\STATE $\mathrm{updatePrio}(Q,(r,\frac{|I(r)|}{|U(r)|+\epsilon}))$%
		\ENDFOR
		\ENDIF
		\ENDWHILE
	\end{algorithmic}
\end{algorithm}

\subsubsection{Global Optima for Positive Triples}

\paragraph{TransH} \label{sec:global_optimum_transH}
As for \textsc{TransE}, there is also a direct way to calculate the global optimum for positive triples as the initial entity embedding for \textsc{TransH}.
However, since every relation has two assigned vectors now, which are cyclically dependent, there is no direct way for relations.
Therefore, we simply initialize relations with the average over all relation embeddings.
The initialization of an entity $v$ can be explained by means of \cref{TransH_entities_positives}, which shows the schematic geometric concept for a three-dimensional embedding space $\mathbb{R}^3$.
\begin{figure}[!b]
	\centering
	\includegraphics[width=0.8\textwidth]{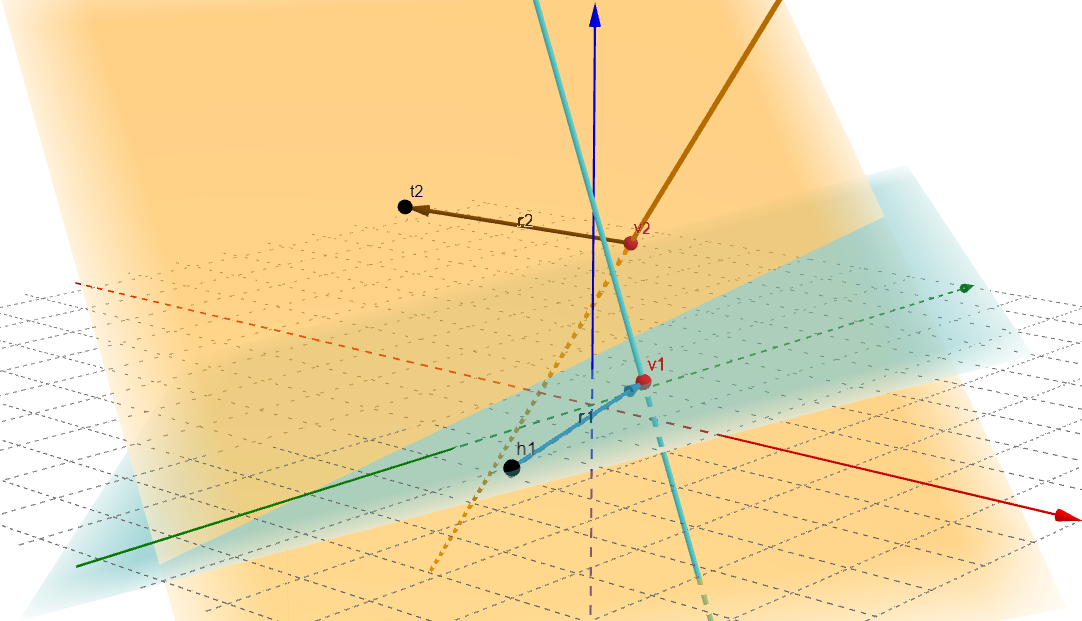}
	\caption[Sketch of entity initialization for \textsc{TransH}]{Sketch of the initialization of a new entity $v$ with the triples $(h,r_1,v)$ and $(v,r_2,t)$ for \textsc{TransH}. The blue hyperplane with its distance vector belongs to the relation $r_1$ and the yellow one to $r_2$. The projection of an entity representation $\vec{e}$ onto the relation-specific hyperplane of $r_i$ is denoted by $\vec{e_i}$. The embedding $\vec{v}$ is then initialized as the closest point to the orthogonal lines through $\vec{v_1}$ and $\vec{v_2}$.} \label{TransH_entities_positives}
\end{figure}
For the example of two triples $(h, r_1, v)$ and $(v, r_2, t)$ with already embedded relations $r_1$ and $r_2$ and entities $h$ and $t$, we first compute the projections of the embedded entities onto the relation-specific hyperplanes. The projection of $\vec{h}$ onto the hyperplane of $r_1$ is denoted in the figure by $\vec{h_1}$ and of $\vec{t}$ onto the hyperplane of $r_2$ by $\vec{t_2}$.
Given the embeddings of the relations on their hyperplanes, the optimal projections $\vec{v_1}$ and $\vec{v_2}$ of $\vec{v}$ can now be computed similarly to the initialization for \textsc{TransE} by adding the relation embedding on the projection of the head or subtracting it from the tail projection, respectively.
In general, for the sets $I(v)=\{(h_1, r_1, v), ..., (h_n, r_n, v)\}$ and $O(v)=\{(v, r_1', t_1), ..., (v, r_m', t_m)\}$ from alg.~\ref{TDSM_init}, we compute the optimal projections of $\vec{v}$ as follows:
\begin{equation}
	\begin{split}
		\vec{v}_i^\bot &= \vec{h_i}-\vec{w}_{r_i}^\top\vec{h_i}\vec{w}_{r_i}+\vec{r_i}\ \forall\ 1\leq i\leq n \\
		\vec{v}_{i+n}^\bot &= \vec{t_i}-\vec{w}_{r'_i}^\top\vec{t_i}\vec{w}_{r'_i}-\vec{r'_i}\ \forall\ 1\leq i\leq m
	\end{split} 
\end{equation}
The problem of finding the vector, whose projections on the hyperplanes are the closest to the optimal projections, boils down to finding the closest point to the lines defined by the optimal projections and the normal vectors of the hyperplanes:
\begin{equation}
	\begin{split}
		\vec{w}_i &= \vec{w}_{r_i}\ \forall\ 1\leq i\leq n \\
		\vec{w}_{i+n} &= \vec{w}_{r'_i}\ \forall\ 1\leq i\leq m
	\end{split}
\end{equation}
In \cref{TransH_entities_positives}, these lines are the orange and the blue one.
Without loss of generality, we can assume that the normal vectors of the hyperplanes are normalized.
Then the closest vector to these lines can be calculated by means of least-squares~\cite{line_intersection} as follows:
\begin{equation}
	\vec{v}=\left(\sum_{i=1}^{m+n}\vec{I}^{d\times d}-\vec{w}_i\vec{w}_i^\top\right)^{-1}\left(\sum_{i=1}^{m+n}(\vec{I}^{d\times d}-\vec{w}_i\vec{w}_i^\top)\vec{v}_i^\bot\right)
\end{equation}
This closest point is the global optimum in regards of \textsc{TransH}'s constraint $\vec{h}_\bot+\vec{r}\approx\vec{t}_\bot$, if $(h,r,t)$ holds, and is therefore the initialization value for the entity $v$.

\subsubsection{Incorporating Negative Evidence}\label{sec:negative_evidence}
By including negative samples in the initialization as well, we can calculate the ``positions'' in the embedding space, from which the new vector should be far away, similarly to the initialization with positive triples before. Formally, for each incoming triple $(h,r,v)\in I(v)$ and outgoing triple from $(v, r, t)\in O(v)$, we sample a certain number of corrupted triples $(h',r,v)\in S'_{(h,r,v)}$ and $(v,r,t')\in S'_{(v,r,t)}$ by replacing the head or the tail, respectively.
For each of these negative triples we can now calculate the tail or the head representations again by adding the relation embedding to the one of the head or subtracting it from the one of the tail.
However, the question is how to constrain the new vector representations to be far way from these ``worst positions''.
Therefore, we discuss both, a direct approach and an iterative one.

\paragraph{Direct Approach}
An intuitive and direct way is to include the negative of these vectors into the calculation of the global optimum.
This initialization mode is called ``neg'' from now on.
The results in comparison to the average initialization and the global optima for positive triples are given i.a. in \cref{TransE_FB15K_preInit_all_MR,TransE_FB15K_preInit_all_Hits@10} for different numbers of negative samples per positive one (given by the number after ``neg'').
While the Hits@10 proportion is decreasing more slowly over time than without any negative triples, the MR is increasing much faster.
For both metrics, this approach with one negative sample yields almost the same results as the average initialization.
As the Hits@10 proportion for the approach without any negative samples is decreasing drastically, we can observe that the ``worst positions'' calculated by means of sampled negative triples are close to the global optimum for the positive ones.
Hence, if we combine the negatives of these ``worst positions'' with the ``best positions'' for the positive triples, they are likely to eliminate each other resulting in an average like initialization.
However, if we increase the number of negative triples per positive one, this balance is no longer given and we can observe a massive performance loss for both metrics.
The new vector representations seem to be far away from the global optima in regards of their positive triples, because the negative samples pull them away.
Moreover, incorporating more than two negative samples per positive one affects the rank of other triples even more negatively than the original approach without any negative triples.
If the vectors for the ``worst positions'' have a small norm, including the negative of them into the arithmetic mean constrains the final embedding to be even closer to them instead of the desired opposite.
All in all, including the negatives of these vectors into the calculation of the global optimum seems to be a very hard constraint, because the new vector representations just have to be far away from the ``worst positions'' and not necessarily close to the negatives of them.

\begin{figure}
	\includegraphics[clip, trim=5cm 0.5cm 5.8cm 0.5cm, width=\textwidth]{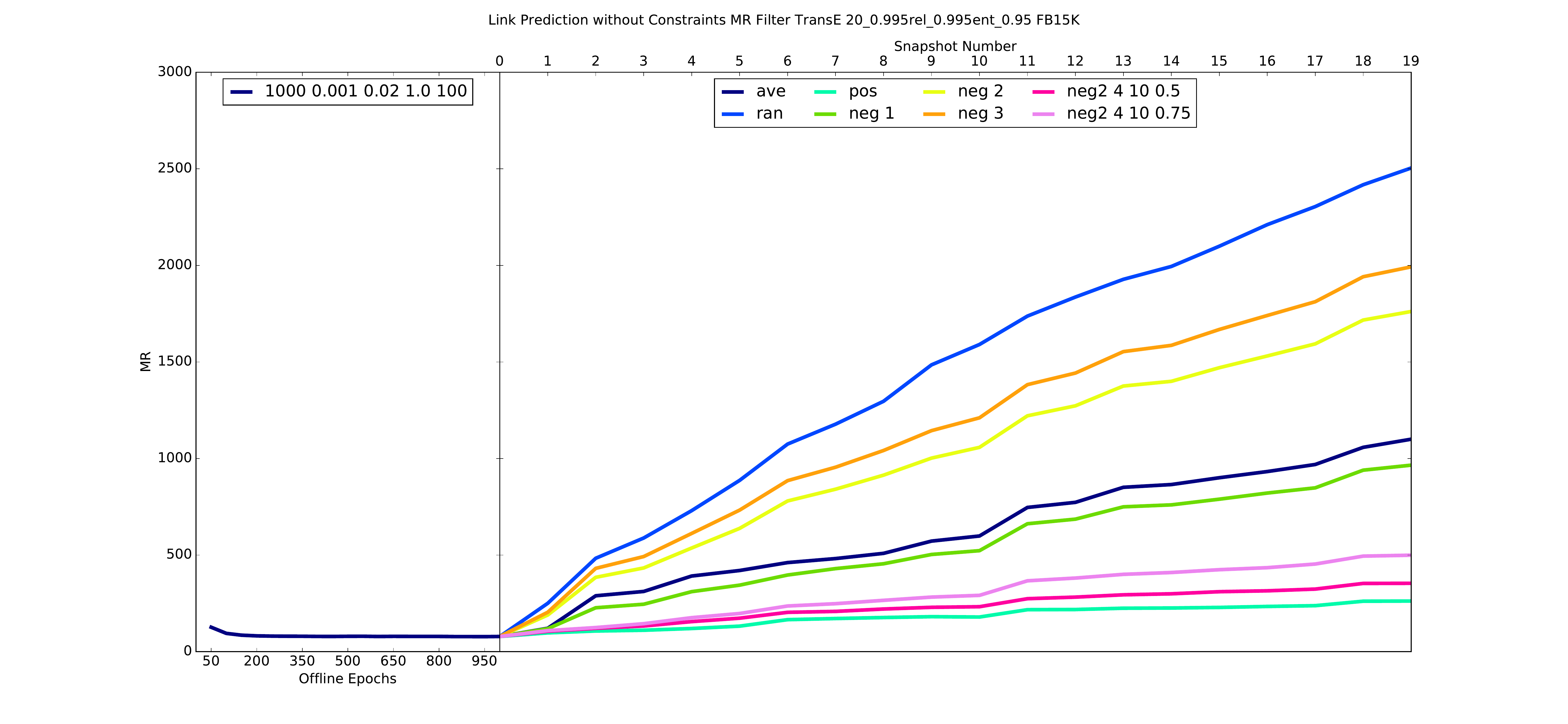}
	\caption[Link prediction performance in the MR metric of all pre-initialization techniques for \textsc{TransE}]{Link prediction performance in the MR metric of the direct (neg) and incremental (neg2) approaches of incorporating negative evidence in comparison to average (ave), random (ran) and global optima (pos) initialization for \textsc{TransE}}%
	\label{TransE_FB15K_preInit_all_MR}
\end{figure}
\begin{figure}
	\includegraphics[clip, trim=5cm 0.5cm 5.8cm 0.5cm, width=\textwidth]{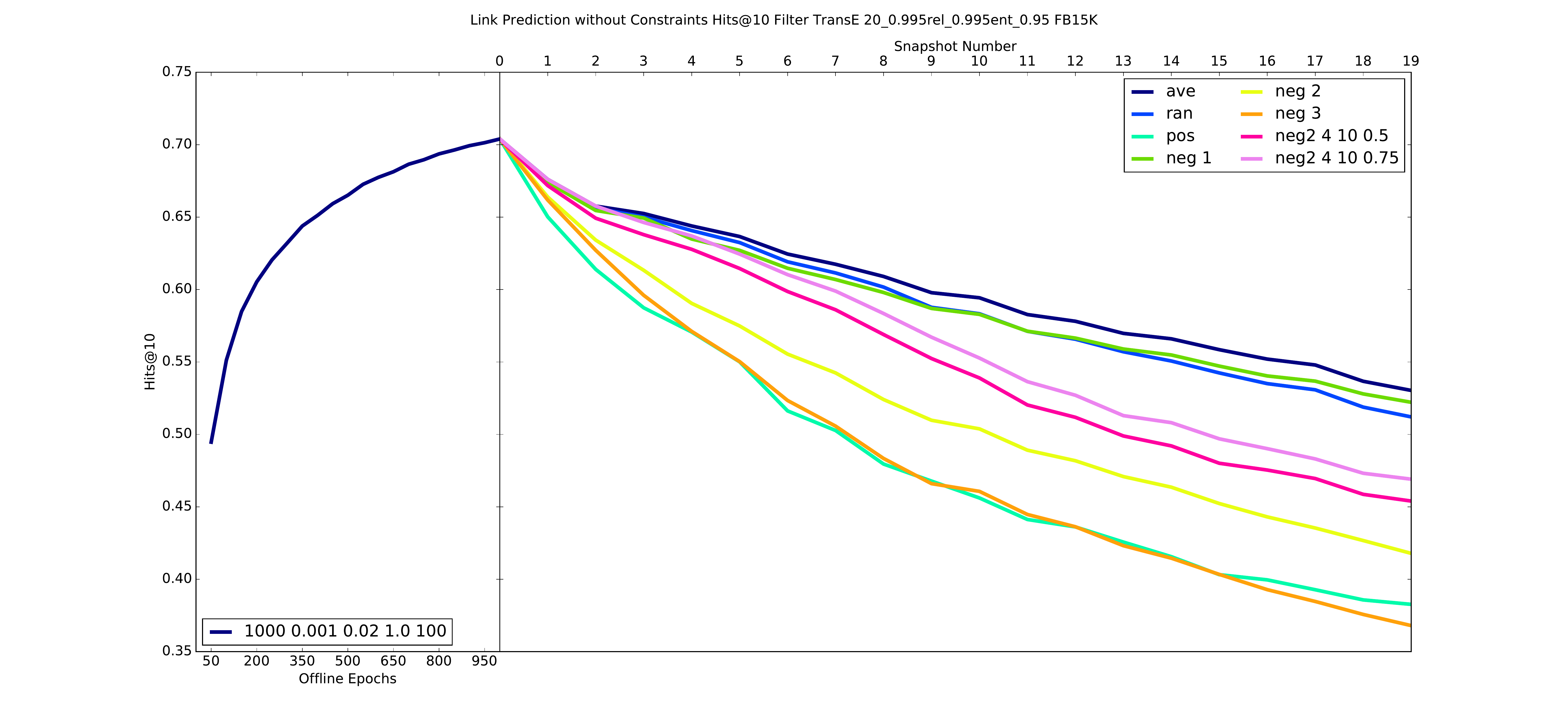}
	\caption[Link prediction performance in the Hits@10 metric of all pre-initialization techniques for \textsc{TransE}]{Link prediction performance in the Hits@10 metric of the direct (neg) and incremental (neg2) approaches of incorporating negative evidence in comparison to average (ave), random (ran) and global optima (pos) initialization for \textsc{TransE}}%
	\label{TransE_FB15K_preInit_all_Hits@10}
\end{figure}

\paragraph{Iterative Approach}
The iterative approach for \textsc{TransE} can be explained on basis of a simple abstraction in a two-dimensional space given in \cref{fig:negatives2_sketch}.
\begin{figure}
	\centering
	\begin{minipage}[t]{.48\textwidth}
		\centering
		\includegraphics[width=0.9\linewidth]{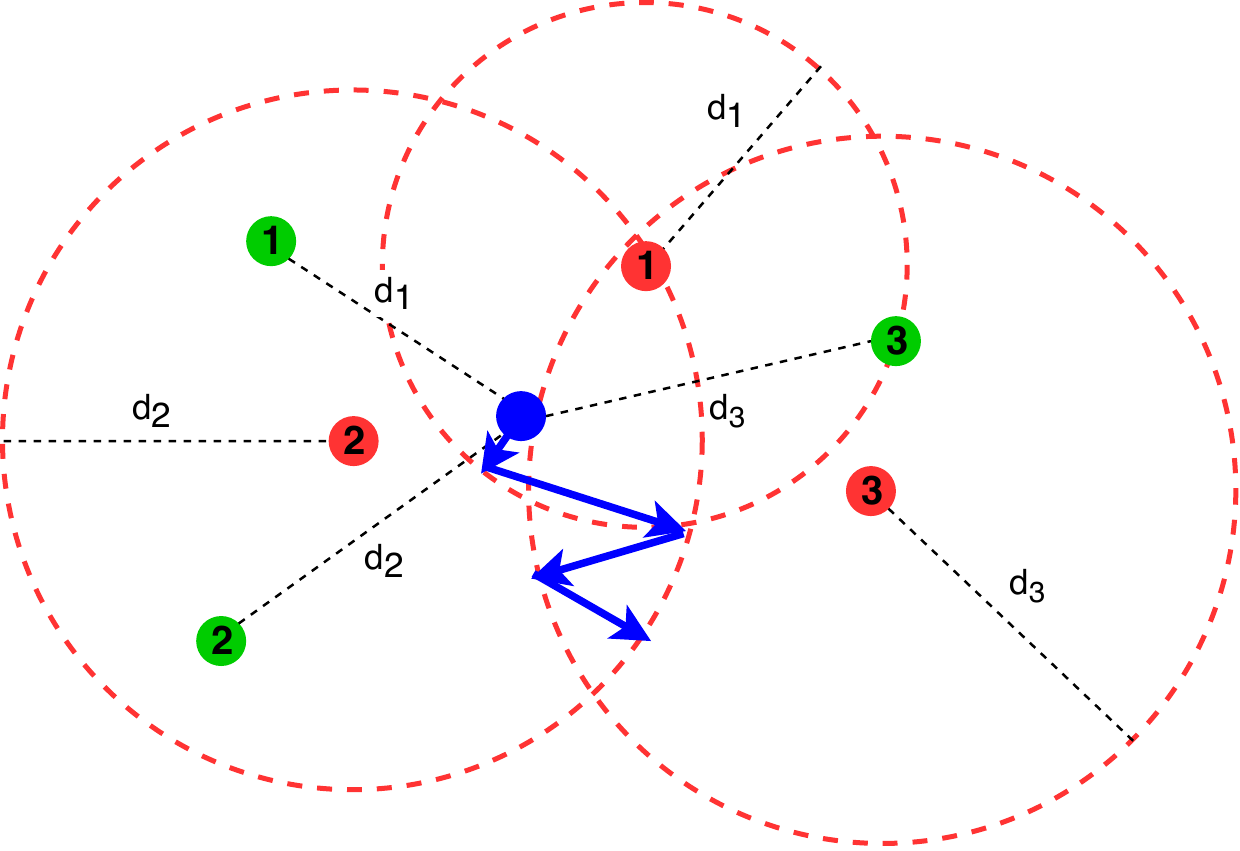}
		\captionof{figure}{Sketch of the general idea for the iterative approach of incorporating negative evidence in a two-dimensional embedding space}
		\label{fig:negatives2_sketch}
	\end{minipage}
	\hfill
	\begin{minipage}[t]{.48\textwidth}
		\centering
		\includegraphics[width=0.9\linewidth]{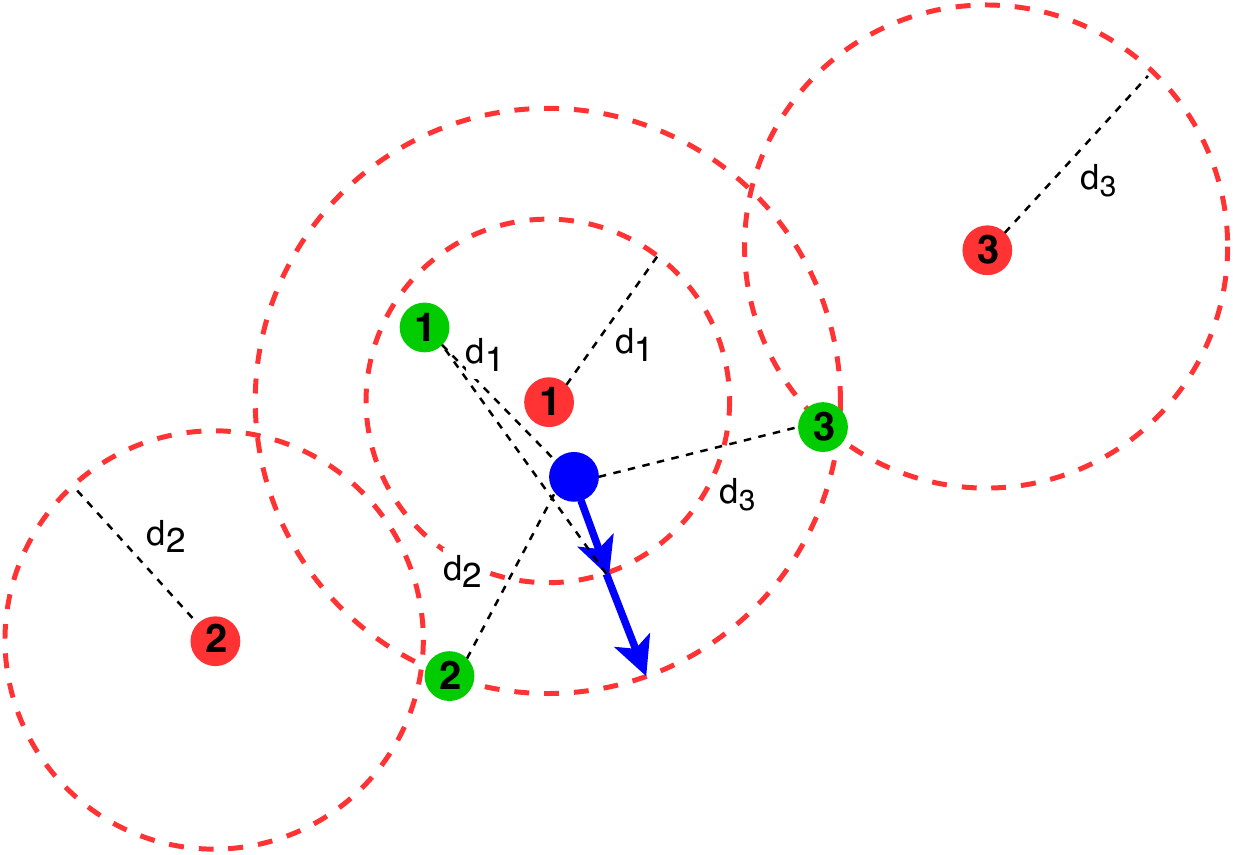}
		\captionof{figure}{Vicious cycle resulting in outliers because of an increasing radius}
		\label{fig:negatives2_outliers}
	\end{minipage}
\end{figure}
We do not consider positive triples and negative samples anymore but ``best'' (green) and ``worst'' (red) positions in the vector space, which are calculated as before.
Each red dot belongs to a green one (indicated by the same number), because the negative sample responsible for this certain worst position has been constructed from the positive triple responsible for the particular best position.
There is still the possibility to sample multiple negative triples per positive one.
At first, we initialize the entity or relation embedding (blue) as the nearest point to all best positions from \cref{main-sec:global_optima} in the main paper.
Now, the intention is that this vector representation should be closer to every best position than to their corresponding worst ones.
Therefore, we construct circles around the worst positions with the distance between the particular embedding to their corresponding best positions as the radii.
At each iteration, we consider one circle.
If the embedding is %
inside the circle, the intuitive solution would be to move it out of it with minimal effort (i.e. distance).
However, the radii have to be updated continuously, because the distances between the best positions and the embedding change after each step.
These updates provide the risk of a vicious cycle as given in \cref{fig:negatives2_outliers}, if the embedding is moved out of circles in the direction with minimal effort.
If the worst position is between its corresponding best position and the embedding, the embedding is constantly moved further away while staying in the updated radius all the time.
In experiments, the resulting outliers drifted so far away that the average length of all new entity embeddings in the $L_2$ norm increased above $80{,}000$.

Another solution is to move the embedding a step towards the corresponding best position. Hence, after one step the radius of the current circle around the worst position is smaller and the embedding is likely to escape it while keeping the step size small. The complete algorithm for the initialization of relations in \textsc{TransE} is given in alg.~\ref{alg:negatives2_iterative}.
\begin{algorithm}
	\caption{Iterative relation initialization with negative evidence for \textsc{TransE}}%
	\label{alg:negatives2_iterative}
	\begin{algorithmic}[1]
		\REQUIRE Relation $r\in\Delta\mathcal{R}^{+}$ with instances $I(r)$
		\STATE calculate start vector accordingly to \cref{main-transE_relations_positives} in the main paper.
		\STATE $epoch \leftarrow 0$
		\STATE $stable \leftarrow \mathrm{False}$
		\WHILE{$epoch < max\_epochs$ and not $stable$}
		\STATE $stable \leftarrow \mathrm{True}$
		\FORALL{$(h,r,t)\in I(r)$}
		\FORALL{$i=1, ..., num\_negs$}
		\STATE $(h',r,t') \leftarrow \mathrm{sample}(S'_{(h,r,t)})$ \COMMENT{sample a negative triple (see eq.~\cref{main-corruptedTriples} in the main paper)}
		\STATE $\vec{w} \leftarrow \vec{t'}-\vec{h'}$ \COMMENT{worst position $\vec{w}$}
		\STATE $\vec{d} \leftarrow \vec{t}-\vec{h}-\vec{r}$ \COMMENT{direction $\vec{d}$ for step}
		\IF{$||\vec{w}-\vec{r}||_{1/2} < ||\vec{d}||_{1/2}$}
		\STATE $\vec{r} \leftarrow \vec{r} + step\_ size \cdot \vec{d}$ \COMMENT{jump towards best position}
		\STATE $stable \leftarrow \mathrm{False}$
		\ENDIF
		\ENDFOR
		\ENDFOR
		\STATE $epoch \leftarrow epoch + 1$
		\ENDWHILE
	\end{algorithmic}
\end{algorithm}
One process step of all circles is called an epoch.
We introduce a parameter for the step size, which is multiplied with the direction vector for the particular step.
The algorithm terminates after the maximum number of epochs or if the particular embedding is not in any circle for a complete epoch, because the initialization vector is considered to be stable in this case.
This algorithm cannot just be extended to the initialization of entities for \textsc{TransE} but also for \textsc{TransH}, where the best and worst positions become lines instead of points and the circles around worst positions are cylinders, if we look at a three-dimensional abstraction of the embedding space.

The experimental results for \textsc{TransE} are also given in \cref{TransE_FB15K_preInit_all_MR,TransE_FB15K_preInit_all_Hits@10} for the two best configurations with a maximum number of ten epochs, a step size of $0.5$ or $0.75$ and four negative samples per positive one.
Both setups show a huge improvement regarding the Hits@10 in comparison to the global optima without negative evidence while increasing the MR to a much smaller extend than the average initialization.
The maximum number of epochs does not have a noteworthy impact on the performance.
Increasing the step size improves the Hits@10 result but raises the MR, because the vector is moved further away from the circles around the worst positions but also further away from the average of the best positions.
The same effect goes for the number of negative samples, as a larger number of circles implies more movement of the vector.

Another solution for the question of the direction, in which the new entity or relation embedding has to be moved, is the gradient.
By using gradient descent, we can find local optima considering positive as well as negative triples.

\subsubsection{Moving Towards Local Optima}
If there are three additional numbers $initNegs$, $initTimes$ and $initLR$ after the pre-initialization mode, the initialized embedding is trained with $initNegs$ negative samples per positive triple, at most $initTimes$ epochs and $initLR$ as the learning rate.
Furthermore, we stop early, if the loss function is zero, track back, if there was an epoch with a lower value of the loss function than in the end, and decrease the learning rate quite fast by multiplying it with $0.5$, if the loss function is not decreasing by $1\%$ for $5$ consecutive epochs.
The best setup for \textsc{TransE} on our validation data set includes one negative sample per positive one, a maximum of $50$ epochs and $0.001$ as the learning rate.

\paragraph*{Time Complexity}
\label{sec:time_complexity_local_optmimum}
We conclude this part by analyzing the complexity of our initialization.
Since the pre-initialization of taking the average over all entities and relations achieved the best performance on our validation data set, we do not consider the pre-initialization for our analysis.
\begin{mytheo}\label{theo:init_time_complex}
	The initialization of new entities and relations (without the pre-initialization) requires at most $\mathcal{O}(initTimes\cdot |\Delta E^{+}_{S}|\cdot(1+initNegs))$ 
	evaluations of the scoring function.
\end{mytheo}
\begin{proof}
	In the worst case the initialization of each entity and relation never stops early.
	Hence, we end up with $initTimes\cdot(|\Delta V^{+}|+|\Delta \mathcal{R}^{+}|)$ epochs for the whole initialization.
	In each of these epochs, the scoring function is evaluated for the triples $I(r)$ in the case of a relation $r$ and for the triples $I(v)\cup O(v)$ in the case of an entity $v$.
	However, since these are new entities and relations, we have $I(r)\subseteq\Delta E^{+}_{S}\ \forall r\in\Delta\mathcal{R}^{+}$ and $I(v)\cup O(v)\subseteq\Delta E^{+}_{S}\ \forall v\in\Delta V^{+}$.
	In the worst case all new triples of the training set consist of at least one new entity or relation.
	Then all triples of $\Delta E^{+}_{S}$ are considered in the initialization.
	Since a triple only consists of three entries, it can only be considered for the initialization of at most one new relation and two new entities.
	As a result, we can omit the number of added entities and relations and give $\mathcal{O}(initTimes\cdot |\Delta E^{+}_{S}|)$ as an upper approximation for the number of evaluations of the scoring function for positive triples.
	Finally, for each evaluation of the scoring function so far, we sample another $initNegs$ corrupted triples and evaluate the scoring function on them.
	This results in $\mathcal{O}(initTimes \cdot (|\Delta E^{+}_{S}| + |\Delta E^{+}_{S}| \cdot initNegs))$ evaluations, where $|\Delta E^{+}_{S}|$ can be factored out.
\end{proof}
Since the different pre-initialization modes do not employ gradient descent, the number of gradient steps for each entity and relation only depends on the KG and the parameters for the initialization.

\begin{mytheo}
	The initialization of new entities and relations requires at most $\mathcal{O}(initTimes\cdot(|\Delta V^{+}|+|\Delta\mathcal{R}^{+}|))$ gradient steps.
\end{mytheo}
\begin{proof}
	As described in the proof of theorem~\ref{theo:init_time_complex} there are at most $initTimes\cdot(|\Delta V^{+}|+|\Delta \mathcal{R}^{+}|)$ epochs for the whole initialization.
	Each epoch includes one gradient step for only one entity or relation representation, as the other embeddings are fixed.
\end{proof}
If we would not fix the embeddings of other entities and relations, the number of gradient steps could be $|V^{(t+1)}|+|\mathcal{R}^{(t+1)}|$ times larger making the initialization impractical for large-scale KGs. 

\subsection{Refreshing Old Embeddings}
\subsubsection{Change Specific Epochs}
\begin{algorithm}
	\caption{Change Specific Epoch} \label{alg:changespecs}
	\begin{algorithmic}[1]
		\STATE $T_{batch}\leftarrow\emptyset$ \COMMENT{initialize the set of pairs of triples}
		\FORALL{$(h',r,t')\in\Delta E_{S}^{-}$}\label{firstFor}
		\STATE $(h,r,t)\leftarrow$ sample($S^{(t+1)}_{(h',r,t')}$) \COMMENT{sample a correct triple (see \cref{main-correctTriples} in the main paper)} 
		\STATE $T_{batch}\leftarrow T_{batch}\cup\{((h,r,t), (h',r,t'))\}$
		\ENDFOR
		\FORALL{$(h,r,t)\in\Delta E_{S}^{+}$} \label{secFor}
		\STATE $(h',r,t')\leftarrow$ sample($S^{\prime (t+1)}_{(h,r,t)}$) \COMMENT{sample a corrupted triple (see \cref{main-corruptedTriples} in the main paper)}
		\STATE $T_{batch}\leftarrow T_{batch}\cup\{((h,r,t), (h',r,t'))\}$
		\ENDFOR \label{thirdEndFor}
		\STATE Update embeddings w.r.t. \[\sum_{((h,r,t), (h',r,t'))\in T_{batch}}\nabla\left[\mathrm{log}(1+\mathrm{exp}(-f_r(h,t)))+\mathrm{log}(1+\mathrm{exp}(f_r(h',t')))\right]\] or \[\sum_{((h,r,t), (h',r,t'))\in T_{batch}}\nabla\left[\mathrm{max}(0, \gamma - f_{r}(h,t) + f_{r}(h',t'))\right]\] \label{learn_update_online}
	\end{algorithmic}
\end{algorithm}
Besides the usual configuration as described in the paper, we provide the option to reduce the set of added triples by ignoring all triples with an added entity or relation, because we do not want to bias the embedding too much towards new triples, if the initialization of entities and relations is already well fitting.
If this option is turned on, we consider the subset $\{(h,r,t)\in\Delta E_S^{+} | r\notin\Delta\mathcal{R}^{+} \land h,t\notin\Delta V^{+}\}$  instead of 
$\Delta E^{+}_{S}$.

\paragraph*{Time Complexity}
Like for the initialization of new entities and relations, we briefly discuss the computational complexity of the change specific epochs.
While the general epochs include multiple backpropagations considering a certain batch of the complete KG at each step, the change specific epochs are comparable to only one of these steps.
\begin{mytheo}
	A change specific epoch requires at most $\mathcal{O}(|\Delta E^{+}_{S}|+|\Delta E^{-}_{S}|)$ evaluations of the scoring function.
\end{mytheo}
\begin{proof}
	The scoring function is evaluated for all triples insight the $T_{batch}$ (as part of a tuple), which are exactly $2\cdot(|\Delta E^{+}_{S}|+|\Delta E^{-}_{S}|)$, since we sample only one corrected or corrupted triple for each deleted or added one, respectively.
\end{proof}  
\begin{mytheo}
	A change specific epoch requires at most $\mathcal{O}(|V^{(t+1)}|+|\mathcal{R}^{(t+1)}|)$ gradient steps.
\end{mytheo}
\begin{proof}
	A change specific epoch includes only one backpropagation in the end.
	In the worst case, every entity and relation in the KG is part of an added or deleted triple.
	As a result, each of their vector representations has to be updated by one gradient step.
\end{proof}

\subsubsection{Online Method}
To update the resulting embedding of the offline method, the online method, given in alg.~\ref{TDSM_online}, is executed after each time step, in which the dynamic KG has changed.
\begin{algorithm}
	\caption{Updating TDSM (online)} \label{TDSM_online}
	\begin{algorithmic}[1]
		\REQUIRE Training set $S^{(t+1)}$ and validation set $W^{(t+1)}$ at time $t+1$, the changes $\Delta V^{+}$, $\Delta E_{S}^{+}$, $\Delta E_{S}^{-}$, $\Delta \mathcal{R}^{+}$, margin $\gamma$, number of general epochs $geNum$, number of change specific epochs $csNum$, validation steps $valid\_steps$
		\STATE initialization($\Delta V^{+}, \Delta\mathcal{R}^{+}$) \label{initialize_online}\COMMENT{initialize new objects with alg.~\ref{TDSM_init}}
		\STATE $type \leftarrow \mathrm{uniform\_mix}(geNum, csNum)$
		\FOR{$epoch=1, ..., geNum+csNum$}
		\IF{$type[epoch]=general$}
		\STATE perform one general epoch \COMMENT{see lines~\ref{start_general} to~\ref{TDSM_offline_optimization} in alg.~\ref{TDSM_offline}}
		\ELSIF{$type[epoch]=changespec$}
		\STATE perform one change specific epoch \COMMENT{see alg.~\ref{alg:changespecs}}
		\ENDIF
		\IF{$valid\_steps \mid epoch$}
		\STATE validate$(\Theta^{(t+1)}, W^{(t+1)})$ \COMMENT{validate every $valid\_steps$ epochs}
		\ENDIF
		\ENDFOR
	\end{algorithmic}
\end{algorithm}
Instead of one parameter for the maximum number of epochs as in the offline method, we now have two parameters determining the number of general and the number of change specific epochs.
Before starting the actual training procedure, 
we first determine the type (i.e. general or change specific epoch) for each epoch number.
Intuitively, mixing up both types of epochs should lead to a smoother change of the embedding.
Furthermore, the online method includes two learning rates, one for the change specific and the other one for the general epochs.
Hence, it is possible to set the learning rate for the change specific epochs accordingly to the change of the dynamic KG without touching the general learning rate.
As in the offline method, the current state of the embedding is validated on the validation set $W^{(t+1)}$ to allow early stopping, if the link prediction performance is not increasing anymore.

\paragraph*{Time Complexity}
\label{sec:time_complexity_online_method}
Before we get to the realization and implementation, we give a final analysis of the time complexity in terms of evaluations of the scoring function and the gradient steps.
\begin{mytheo}
	The online method requires at most $\mathcal{O}(initTimes\cdot|\Delta E^{+}_{S}|\cdot (1+initNegs)+csNum\cdot (|\Delta E^{+}_{S}|+|\Delta E^{-}_{S}|)+geNum\cdot |S^{(t+1)}|)$ evaluations of the scoring function.
\end{mytheo}
\begin{proof}
	We already proved the intermediate results for the initialization and the change specific epochs.
	There are $csNum$ change specific and $geNum$ general epochs in our online method.
	Each general epoch runs through the entire training set of the current time step $t+1$ and evaluates the scoring function for each triple (and a corrupted one).
\end{proof}
\begin{mytheo}
	The online method requires at most $\mathcal{O}(initTimes\cdot (|\Delta V^{+}|+|\Delta\mathcal{R}^{+}|)+(|V^{(t+1)}|+|\mathcal{R}^{(t+1)}|)\cdot (csNum+geNum\cdot num\_batch))$ gradient steps.
\end{mytheo}
\begin{proof}
	Again, we already proved the intermediate results for the initialization and the change specific epochs.
	Each general epoch first divides the training set into $num\_batch$ batches and performs a backpropagation for each one.
	In the worst case, every entity and relation in the KG is part of a triple in each batch.
	Then, each of their vector representations has to be updated by one gradient step.
\end{proof}

Note that the online method is an extension of the offline method, i.e., if we just throw away the embedding of the last snapshot, ignore the initialization of new entities and relations, set the number of change specific epochs to zero and of general epochs to $num\_epoch$ as used in alg.~\ref{TDSM_offline}, it boils down to the original offline method.
This means, the time complexity of the offline method is given by $\mathcal{O}(num\_epoch\cdot |S|)$ in terms of evaluations of the scoring function and $\mathcal{O}(num\_epoch\cdot num\_batch\cdot (|V|+|\mathcal{R}|))$ for the gradient steps.\footnote{The formalism is adjusted to alg.~\ref{TDSM_offline}. If the offline method is used to recalculate the embeddings for the snapshots of an dynamic KG, the training set and sets of entities and relations have to be marked with the time step.}
To obtain good results, the change of the dynamic KG between two consecutive time steps should be rather small\footnote{We make a choice for a reasonable size in \cref{main-sec:evaluation_on_real_datasets} in the main paper, but in general, this is up the user or the particular use case.}, i.e., $|\Delta E^{+}_{S}| << |S^{(t+1)}|$, $|\Delta E^{-}_{S}| << |S^{(t+1)}|$, $|\Delta V^{+}| << |V^{(t+1)}|$ and $|\Delta\mathcal{R}^{+}| << |\mathcal{R}^{(t+1)}|$.
Therefore, the number of general epochs mainly determines the runtime.
While the offline method starts from scratch, the online method continues the training of an already existing embedding.
As a result, we usually have $num\_epoch >> geNum$, which means that in practice a recalculation of the entire embedding using the offline method takes longer than an update step by the online method.
We will validate this statement for a large-scale dynamic KG in \cref{main-sec:evaluation_on_real_datasets} in the main paper.

\section{Evaluation on Real Datasets}
This section includes more details about the used datasets as well as the evaluation in terms of stability and scalability.
\subsection{Datasets}
\subsubsection{NELL5}
As already briefly described in \cref{main-sec:use_cases} in the main paper, NELL~\cite{NEL} is a system, which constructs a probabilistic KG automatically by extracting knowledge from the World Wide Web over time since 2010.
We use the extracted KG of high-confidence triples\footnote{http://rtw.ml.cmu.edu/rtw/resources} with $2,766,078$ triples over $833$ relations and $2{,}211{,}486$ entities.
As we can see for example at the triple per entity proportion, NELL's KG is really sparse.
The reason for that is that NELL extracts many entities with just one assigned category, which we translate into a triple of the relation ``type'', e.g. $\texttt{(empire\_state\_building, type, skyscraper)}$.
Sparsity is a basic problem regarding KG embedding.
If the KG is too sparse, there are too less constraints to learn an appropriate embedding.
Although we are more concerned to keep the embedding quality stable over time, the start embedding has to be at least reasonable enough to draw conclusions about the performance of our approach, when we update it.
Hence, we preprocess NELL's KG by iteratively excluding all entities and relations (as well as the triples including them) with less than a certain threshold of triples, in which they occur, until the resulting KG is stable.
The exact statistics about the change between the different snapshots are given in table~\ref{tab:nell5} of the appendix.

\begin{table}[!h]
\caption[Statistics of NELL5 $2\_20\_0.5\_0$]{The cardinality of changes between two consecutive snapshots of NELL5 $2\_20\_0.5\_0$. On average the training sets contain $181{,}606$ triples, the validation sets $9{,}978$ and the test sets $10{,}074$.} \label{tab:nell5}
\resizebox{\textwidth}{!}{
\begin{tabular}{|c||cc|cc|cc|cc|cc|}
\hline
\textbf{Step} & $|\Delta E^{+}_{S}|$ & $|\Delta E^{-}_{S}|$ & $|\Delta E^{+}_{W}|$ & $|\Delta E^{-}_{W}|$ & $|\Delta E^{+}_{T}|$ & $|\Delta E^{-}_{T}|$ & $|\Delta\mathcal{R}^{+}|$ & $|\Delta\mathcal{R}^{-}|$ & $|\Delta V^{+}|$ & $|\Delta V^{-}|$ \\
\hline
\hline
$0-1$ & $9{,}118$ & $9{,}046$ & $477$ & $510$ & $488$ & $527$ & $0$ & $11$ & $992$ & $438$ \\
$1-2$ & $9{,}081$ & $9{,}081$ & $492$ & $488$ & $510$ & $514$ & $0$ & $16$ & $431$ & $630$ \\
$2-3$ & $9{,}130$ & $9{,}112$ & $465$ & $488$ & $488$ & $483$ & $4$ & $2$ & $1{,}245$ & $1{,}256$ \\
$3-4$ & $9{,}095$ & $9{,}067$ & $480$ & $508$ & $508$ & $508$ & $6$ & $1$ & $773$ & $2{,}028$ \\
$4-5$ & $9{,}065$ & $9{,}075$ & $514$ & $490$ & $503$ & $517$ & $3$ & $1$ & $1{,}078$ & $757$ \\
$5-6$ & $9{,}086$ & $9{,}085$ & $498$ & $520$ & $499$ & $478$ & $5$ & $2$ & $1{,}139$ & $1{,}243$ \\
$6-7$ & $9{,}041$ & $9{,}102$ & $518$ & $499$ & $524$ & $482$ & $1$ & $2$ & $1{,}631$ & $1{,}394$ \\
$7-8$ & $9{,}100$ & $9{,}102$ & $491$ & $468$ & $492$ & $513$ & $0$ & $0$ & $523$ & $353$ \\
$8-9$ & $9{,}043$ & $9{,}071$ & $495$ & $533$ & $545$ & $479$ & $0$ & $3$ & $182$ & $170$ \\
$9-10$ & $9{,}090$ & $9{,}047$ & $501$ & $485$ & $492$ & $551$ & $1$ & $2$ & $102$ & $268$ \\
$10-11$ & $9{,}085$ & $9{,}081$ & $482$ & $500$ & $516$ & $502$ & $123$ & $8$ & $2{,}337$ & $455$ \\
$11-12$ & $9{,}052$ & $9{,}100$ & $505$ & $488$ & $526$ & $495$ & $18$ & $13$ & $1{,}288$ & $753$ \\
$12-13$ & $9{,}114$ & $9{,}144$ & $490$ & $487$ & $479$ & $452$ & $16$ & $37$ & $831$ & $1{,}526$ \\
$13-14$ & $9{,}068$ & $9{,}127$ & $485$ & $490$ & $530$ & $466$ & $7$ & $3$ & $541$ & $1{,}031$ \\
$14-15$ & $9{,}075$ & $9{,}055$ & $499$ & $519$ & $509$ & $509$ & $2$ & $5$ & $324$ & $1{,}202$ \\
$15-16$ & $9{,}033$ & $9{,}072$ & $549$ & $518$ & $500$ & $492$ & $4$ & $7$ & $248$ & $451$ \\
$16-17$ & $9{,}106$ & $9{,}054$ & $496$ & $503$ & $481$ & $526$ & $5$ & $2$ & $199$ & $338$ \\
$17-18$ & $9{,}025$ & $9{,}059$ & $525$ & $522$ & $533$ & $502$ & $3$ & $5$ & $135$ & $728$ \\
$18-19$ & $9{,}082$ & $9{,}033$ & $513$ & $512$ & $488$ & $538$ & $4$ & $7$ & $91$ & $1{,}585$ \\
\hline
\end{tabular}
}
\end{table}

\subsubsection{MathOverflow}
The MathOverflow data set\footnote{https://snap.stanford.edu/data/sx-mathoverflow.html} is a temporal interaction graph of the users from the MathOverflow forum\footnote{https://mathoverflow.net/} and has been constructed from the Stack Exchange Data Dump over a time period of $2{,}350$ days~\cite{Mathoverflow}.
Interactions between users, which are represented by the entities, are modeled by only three types of relations: ``answers a question of'' (a2q), ``comments on an answer of'' (c2a) and ``comments on a question of'' (c2q).
The KG contains $506{,}550$ facts about $24{,}818$ entities.
Since the KG is already dense enough, we can use the original data set instead of filtering it first.
The statistics about the changes of the dynamic KG are given in table~\ref{tab:MathOverflow}.

\begin{table}
\caption[Statistics of MathOverflow $2\_20\_0.5\_0$]{The cardinality of changes between two consecutive snapshots of MathOverflow $2\_20\_0.5\_0$. On average the training sets contain $129{,}083$ triples, the validation sets $7{,}177$ and the test sets $7{,}117$.} \label{tab:MathOverflow}
\resizebox{\textwidth}{!}{
\begin{tabular}{|c||cc|cc|cc|cc|cc|}
\hline
\textbf{Step} & $|\Delta E^{+}_{S}|$ & $|\Delta E^{-}_{S}|$ & $|\Delta E^{+}_{W}|$ & $|\Delta E^{-}_{W}|$ & $|\Delta E^{+}_{T}|$ & $|\Delta E^{-}_{T}|$ & $|\Delta\mathcal{R}^{+}|$ & $|\Delta\mathcal{R}^{-}|$ & $|\Delta V^{+}|$ & $|\Delta V^{-}|$ \\
\hline
\hline
$0-1$ & $6{,}232$ & $6{,}256$ & $378$ & $328$ & $326$ & $368$ & $0$ & $0$ & $605$ & $277$ \\
$1-2$ & $6{,}149$ & $5{,}534$ & $367$ & $288$ & $347$ & $300$ & $0$ & $0$ & $670$ & $229$ \\
$2-3$ & $6{,}130$ & $5{,}552$ & $320$ & $302$ & $329$ & $295$ & $0$ & $0$ & $607$ & $220$ \\
$3-4$ & $6{,}044$ & $6{,}062$ & $349$ & $325$ & $345$ & $343$ & $0$ & $0$ & $625$ & $247$ \\
$4-4$ & $6{,}447$ & $6{,}112$ & $354$ & $312$ & $336$ & $349$ & $0$ & $0$ & $722$ & $264$ \\
$5-5$ & $6{,}119$ & $6{,}101$ & $332$ & $314$ & $351$ & $337$ & $0$ & $0$ & $742$ & $246$ \\
$6-6$ & $5{,}624$ & $6{,}327$ & $326$ & $359$ & $294$ & $370$ & $0$ & $0$ & $701$ & $323$ \\
$7-7$ & $5{,}797$ & $6{,}205$ & $329$ & $332$ & $316$ & $334$ & $0$ & $0$ & $744$ & $328$ \\
$8-8$ & $5{,}786$ & $6{,}039$ & $320$ & $334$ & $327$ & $333$ & $0$ & $0$ & $710$ & $282$ \\
$9-10$ & $6{,}004$ & $6{,}001$ & $329$ & $333$ & $334$ & $337$ & $0$ & $0$ & $779$ & $292$ \\
$10-11$ & $5{,}898$ & $6{,}039$ & $321$ & $377$ & $307$ & $330$ & $0$ & $0$ & $717$ & $345$ \\
$11-12$ & $5{,}747$ & $6{,}432$ & $301$ & $350$ & $319$ & $353$ & $0$ & $0$ & $703$ & $387$ \\
$12-13$ & $5{,}807$ & $6{,}574$ & $312$ & $324$ & $330$ & $343$ & $0$ & $0$ & $702$ & $396$ \\
$13-14$ & $5{,}496$ & $6{,}126$ & $307$ & $334$ & $303$ & $358$ & $0$ & $0$ & $680$ & $468$ \\
$14-15$ & $5{,}677$ & $6{,}115$ & $319$ & $350$ & $311$ & $348$ & $0$ & $0$ & $690$ & $464$ \\
$15-16$ & $5{,}638$ & $6{,}322$ & $322$ & $348$ & $307$ & $342$ & $0$ & $0$ & $676$ & $465$ \\
$16-17$ & $5{,}749$ & $6{,}145$ & $314$ & $351$ & $312$ & $308$ & $0$ & $0$ & $793$ & $502$ \\
$17-18$ & $5{,}556$ & $6{,}337$ & $316$ & $336$ & $328$ & $360$ & $0$ & $0$ & $631$ & $500$ \\
$18-19$ & $5{,}755$ & $6{,}196$ & $281$ & $327$ & $329$ & $343$ & $0$ & $0$ & $735$ & $530$ \\
\hline
\end{tabular}
}
\end{table}

\subsubsection{Higgs Twitter}
The Higgs Twitter data set\footnote{https://snap.stanford.edu/data/higgs-twitter.html} consists of the messages posted on Twitter between the 1st and 7th July 2012 about the discovery of a new particle with the features of the elusive Higgs boson on the 4th July 2012~\cite{Higgs}.
To be more precise, the dynamic KG represents four different types of interactions between users: ``user A follows user B'' (FL), ``user A re-tweets a tweet from user B'' (RT), ``user A replies to a tweet of user B'' (RE) and ``user A mentions user B'' (MT).
Overall, the data set contains $15{,}367{,}315$ interactions between $456{,}626$ users.
However, the follower relation is responsible for $14{,}855{,}842$ ($96.67\%$) of these triples forming an underlying directed social network, which does not include any time stamps unlike the other relations.
As we need some kind of temporal order of the triples to obtain chronological snapshots, for each entity $v$, we compute the minimum time stamp $min_v$ of the timed triples, in which it occurs.
If the entity does not interact with any other entity in the form of a re-tweet, reply or mention, it is excluded from the dynamic KG. Hence, each entity has a minimum time stamp.
Now, given a triple $(h,follows,t)$, we assign the minimum of $min_h$ and $min_t$ as the time stamp.
Therefore, as soon as an entity has its first interaction with another one, all its direct predecessors and successors of the follower relation are or have been contained in the dynamic KG leading to a reasonable temporal order of triples while preventing extreme changes in the dynamic KG.
The preprocessed data set contains $10{,}473{,}241$ triples over $304{,}691$ entities.
The statistics about the changes of the training, validation and test sets are given in table~\ref{tab:higgs}.

\begin{table}
\caption[Statistics of Higgs Twitter $2\_20\_0.5\_0$]{The cardinality of changes between two consecutive snapshots of Higgs Twitter $2\_20\_0.5\_0$. On average the training sets contain $4{,}694{,}363$ triples, the validation sets $260{,}790$ and the test sets $259{,}978$.} \label{tab:higgs}
\resizebox{\textwidth}{!}{
\begin{tabular}{|c||cc|cc|cc|cc|cc|}
\hline
\textbf{Step} & $|\Delta E^{+}_{S}|$ & $|\Delta E^{-}_{S}|$ & $|\Delta E^{+}_{W}|$ & $|\Delta E^{-}_{W}|$ & $|\Delta E^{+}_{T}|$ & $|\Delta E^{-}_{T}|$ & $|\Delta\mathcal{R}^{+}|$ & $|\Delta\mathcal{R}^{-}|$ & $|\Delta V^{+}|$ & $|\Delta V^{-}|$ \\
\hline
\hline
$0-1$ & $233{,}234$ & $235{,}383$ & $13{,}106$ & $13{,}051$ & $13{,}148$ & $13{,}348$ & $0$ & $0$ & $1{,}287$ & $1{,}909$ \\
$1-2$ & $234{,}491$ & $235{,}609$ & $12{,}842$ & $12{,}992$ & $12{,}970$ & $13{,}157$ & $0$ & $0$ & $1{,}458$ & $1{,}319$ \\
$2-3$ & $234{,}067$ & $235{,}254$ & $13{,}069$ & $13{,}081$ & $12{,}947$ & $13{,}376$ & $0$ & $0$ & $1{,}813$ & $2{,}989$ \\
$3-4$ & $233{,}081$ & $235{,}422$ & $12{,}939$ & $13{,}260$ & $13{,}022$ & $12{,}959$ & $0$ & $0$ & $2{,}450$ & $1{,}842$ \\
$4-5$ & $234{,}033$ & $235{,}344$ & $12{,}979$ & $13{,}100$ & $12{,}972$ & $13{,}122$ & $0$ & $0$ & $1{,}924$ & $1{,}650$ \\
$5-6$ & $234{,}262$ & $235{,}278$ & $13{,}298$ & $13{,}022$ & $12{,}865$ & $13{,}074$ & $0$ & $0$ & $1{,}363$ & $2{,}159$ \\
$6-7$ & $234{,}618$ & $235{,}345$ & $12{,}969$ & $12{,}938$ & $12{,}922$ & $13{,}279$ & $0$ & $0$ & $1{,}714$ & $2{,}287$ \\
$7-8$ & $234{,}556$ & $235{,}437$ & $12{,}728$ & $13{,}011$ & $13{,}014$ & $12{,}983$ & $0$ & $0$ & $1{,}534$ & $3{,}026$ \\
$8-9$ & $234{,}620$ & $235{,}521$ & $13{,}019$ & $12{,}978$ & $12{,}963$ & $13{,}041$ & $0$ & $0$ & $1{,}878$ & $3{,}333$ \\
$9-10$ & $234{,}473$ & $235{,}452$ & $13{,}109$ & $13{,}009$ & $13{,}003$ & $12{,}938$ & $0$ & $0$ & $2{,}402$ & $3{,}688$ \\
$10-11$ & $234{,}309$ & $235{,}061$ & $13{,}069$ & $13{,}182$ & $13{,}200$ & $13{,}041$ & $0$ & $0$ & $2{,}780$ & $2{,}157$ \\
$11-12$ & $234{,}430$ & $235{,}124$ & $13{,}166$ & $13{,}079$ & $13{,}090$ & $13{,}014$ & $0$ & $0$ & $2{,}317$ & $2{,}851$ \\
$12-13$ & $234{,}717$ & $235{,}274$ & $12{,}977$ & $12{,}832$ & $12{,}981$ & $12{,}901$ & $0$ & $0$ & $2{,}490$ & $3{,}089$ \\
$13-14$ & $234{,}550$ & $234{,}964$ & $13{,}037$ & $13{,}238$ & $12{,}903$ & $12{,}950$ & $0$ & $0$ & $2{,}727$ & $2{,}871$ \\
$14-15$ & $234{,}259$ & $234{,}972$ & $13{,}140$ & $13{,}056$ & $13{,}036$ & $12{,}950$ & $0$ & $0$ & $2{,}547$ & $3{,}324$ \\
$15-16$ & $234{,}401$ & $235{,}084$ & $12{,}939$ & $13{,}104$ & $13{,}040$ & $12{,}824$ & $0$ & $0$ & $2{,}759$ & $3{,}040$ \\
$16-17$ & $234{,}136$ & $234{,}114$ & $12{,}946$ & $13{,}152$ & $12{,}986$ & $12{,}985$ & $0$ & $0$ & $2{,}768$ & $2{,}419$ \\
$17-18$ & $233{,}699$ & $234{,}140$ & $13{,}131$ & $12{,}801$ & $12{,}824$ & $12{,}882$ & $0$ & $0$ & $3{,}686$ & $2{,}051$ \\
$18-19$ & $232{,}654$ & $234{,}080$ & $12{,}718$ & $13{,}001$ & $12{,}909$ & $13{,}111$ & $0$ & $0$ & $3{,}548$ & $2{,}290$ \\
\hline
\end{tabular}
}
\end{table}

\subsection{Experimental Setup}\label{sec:experimental_setup}
To determine the best configuration of our approach, the link prediction performance as well as the stability of the online method is evaluated for different combinations of parameters.
However, before we discuss these combinations, we provide all parameters of the initial embedding trained by the offline method.
Therefore, we simply adopt the settings, which were used by Han et al.~\cite{openke} for embeddings of the static FB15K data set.
\footnote{In this paper we use standard SGD for all TD models and Adagrad for all SM models and most parameters are as was found in the work around OpenKE\cite{openke}. A further work on parameter settings provides more dataset specific training settings \cite{pykeen}, but in the current work it is less important to find optimal performance than it is to find a stable performance over time.}

The dimension of the embedding space is $d=100$ for each translational distance and semantic matching model except for \textsc{Rescal} with $d=50$.
Furthermore, we minimize the loss functions of translational distance models using plain stochastic gradient descent with a learning rate of $\lambda =0.001$, while for semantic matching models we use Adagrad~\cite{Adagrad} with a learning rate of $\lambda =0.1$.
As the weight for the $L_{2}$-penalty we use $w=0.02$ and additionally we decrease all learning rates in the offline and online method by multiplying it with $0.95$, if the loss function is not decreasing by $0.5\%$ for $20$ consecutive epochs to converge faster towards local minima.
The current state of the embedding is evaluated on the validation set after every $10$ epochs and if the link prediction performance in terms of Hits@10 is not increasing for ten consecutive validations, the offline or online method is stopped early.
The offline method runs for at most $1000$ epochs.

After introducing up to twelve new parameters for the complete online method, we need to have a consistent naming of the graphs for the combinations of parameters.
\Cref{fig:parameter} describes the order of parameters with their possible values and in which part of the online method they occur.
\begin{figure}[!b]
\includegraphics[width=\textwidth]{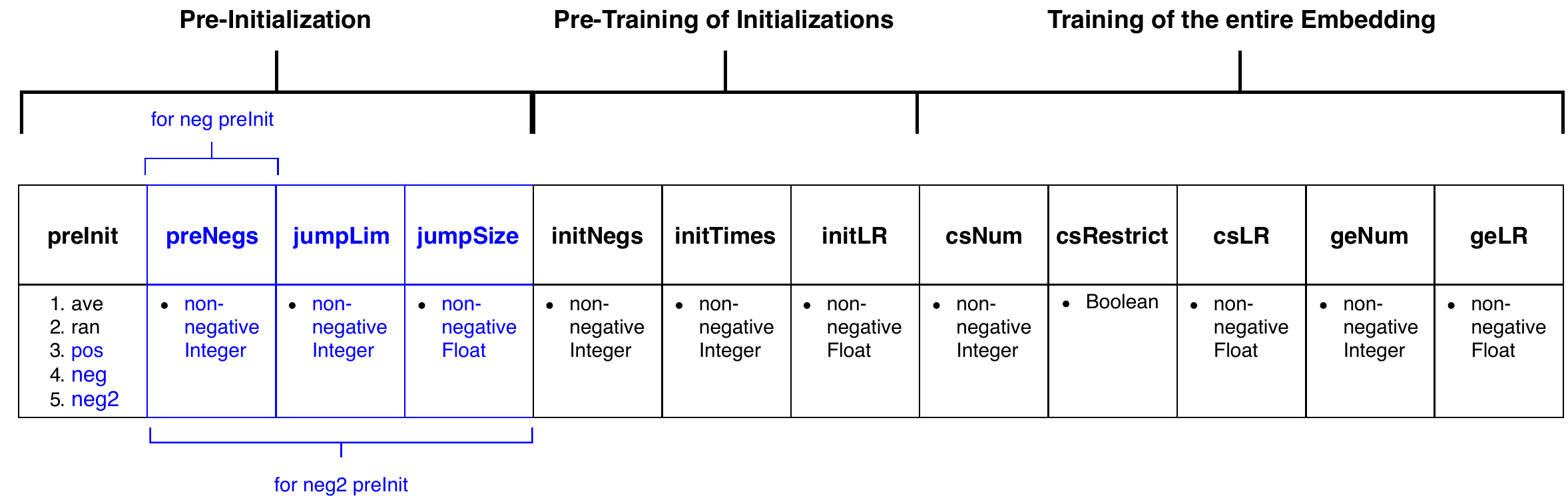}
\caption[Order of all hyperparameters for the online method with their possible values]{Parameter order with their possible values.  Blue marked parameters and values are only usable for \textsc{TransE} and \textsc{TransH}.} \label{fig:parameter}
\end{figure}
Furthermore, the meaning of the different parameter names can be summarized as follows:
\begin{enumerate}
\item \textbf{preInit}: method for the pre-initialization of new entities and relations as described in \cref{main-sec:global_optima,main-sec:negative_evidence} in the main paper
\item \textbf{preNegs}: number of negative triples per positive one, if either neg or neg2 is set as pre-initialization mode
\item \textbf{jumpLim}: maximum number of epochs in iterative approach of incorporating negative evidence
\item \textbf{jumpSize}: size of steps in iterative approach
\item \textbf{initNegs}: number of negative triples per positive one in pre-training of initializations as described in \cref{main-sec:local_optima} in the main paper
\item \textbf{initTimes}: maximum number of epochs in pre-training of initializations
\item \textbf{initLR}: learning rate in pre-training of initializations
\item \textbf{csNum}: number of change specific epochs in updates of old embeddings
\item \textbf{csRestrict}: flag for the restriction of change specific epochs to added triples with already embedded entities and relations
\item \textbf{csLR}: learning rate of change specific epochs
\item \textbf{geNum}: number of general epochs
\item \textbf{geLR}: learning rate of general epochs
\end{enumerate}

To reduce the number of combinations for the parameter tuning, we fix the parameters for pre-initialization as well as pre-training to the following combination, which showed the best results in the task of link prediction on the synthetic dynamic KG of FB15K.
\begin{equation}
\begin{split}
preInit &= ave \\
initNegs &= 1 \\
initTimes &= 50 \\
initLR &=\begin{cases}
 0.001 & \text{for translational distance models}\\
 0.1  & \text{for semantic matching models}\\
\end{cases}
\end{split}
\end{equation}
However, if we only regard two values for the remaining five parameters for each of the six embedding models, we still have $6\cdot 2^5 = 192$ combinations for a grid search on each data set.
To reduce this number even further, we choose \textsc{TransE} and \textsc{DistMult} as the representatives for translational distance and semantic matching models and stick to the synthetic dynamic KG of FB15K for now.
Moreover, we require the sum of change specific and general epochs to be $200$ in order to tune the ratio instead of the overall number of epochs and also set the learning rate of general epochs to $0.0002$ for \textsc{TransE} and $0.02$ for \textsc{DistMult}, as these settings showed the best link prediction performance in experiments, where only the general learning rate was tuned. 
We conduct a grid search with $csNum\in\{0, 20, 40\}$, $csRestrict\in\{False, True\}$ and $csLR\in\{0.0001, 0.0002, 0.0004\}$ for translational distance models or $csLR\in\{0.01, 0.02, 0.04\}$ for semantic matching models respectively to find the remaining three hyperparameters. 
The differences between the chosen values for the learning rate are so small, because they have to be seen in proportion to the learning rate of the general epochs.
Moreover, the change specific epochs strongly bias the embedding towards the changes of the dynamic KG, which should be used very carefully.
We will use the best configurations in comparison to completely recalculated embeddings for each time step on the real dynamic KGs afterwards.

\paragraph*{Best Configuration}
As we noticed that the restriction of change specific epochs to additions with already embedded entities and relations does not influence the link prediction results substantially, we only look at the configurations without the restriction.
Furthermore, we use the MRR metric as a good compromise of Hits@n and MR for the comparison in \cref{TransE_FB15K_changespecs_MRR,DistMult_FB15K_changespecs_MRR}.
\begin{figure}[!b]
\includegraphics[clip, trim=5cm 0.5cm 5.8cm 0.5cm, width=\textwidth]{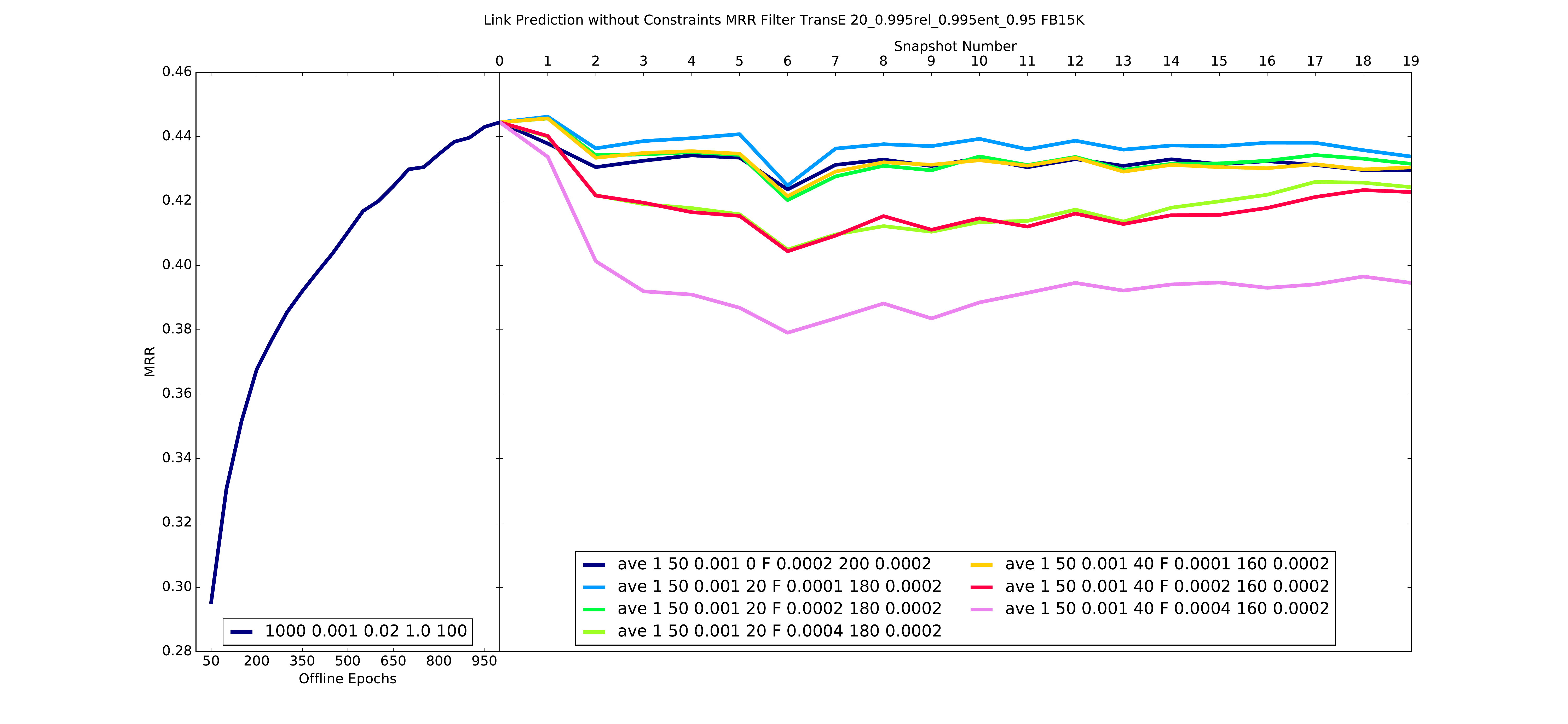}
\caption{Link prediction performance in the MRR metric of change specific epochs for \textsc{TransE}} \label{TransE_FB15K_changespecs_MRR}
\end{figure}
\begin{figure}
\includegraphics[clip, trim=5cm 0.5cm 5.8cm 0.5cm, width=\textwidth]{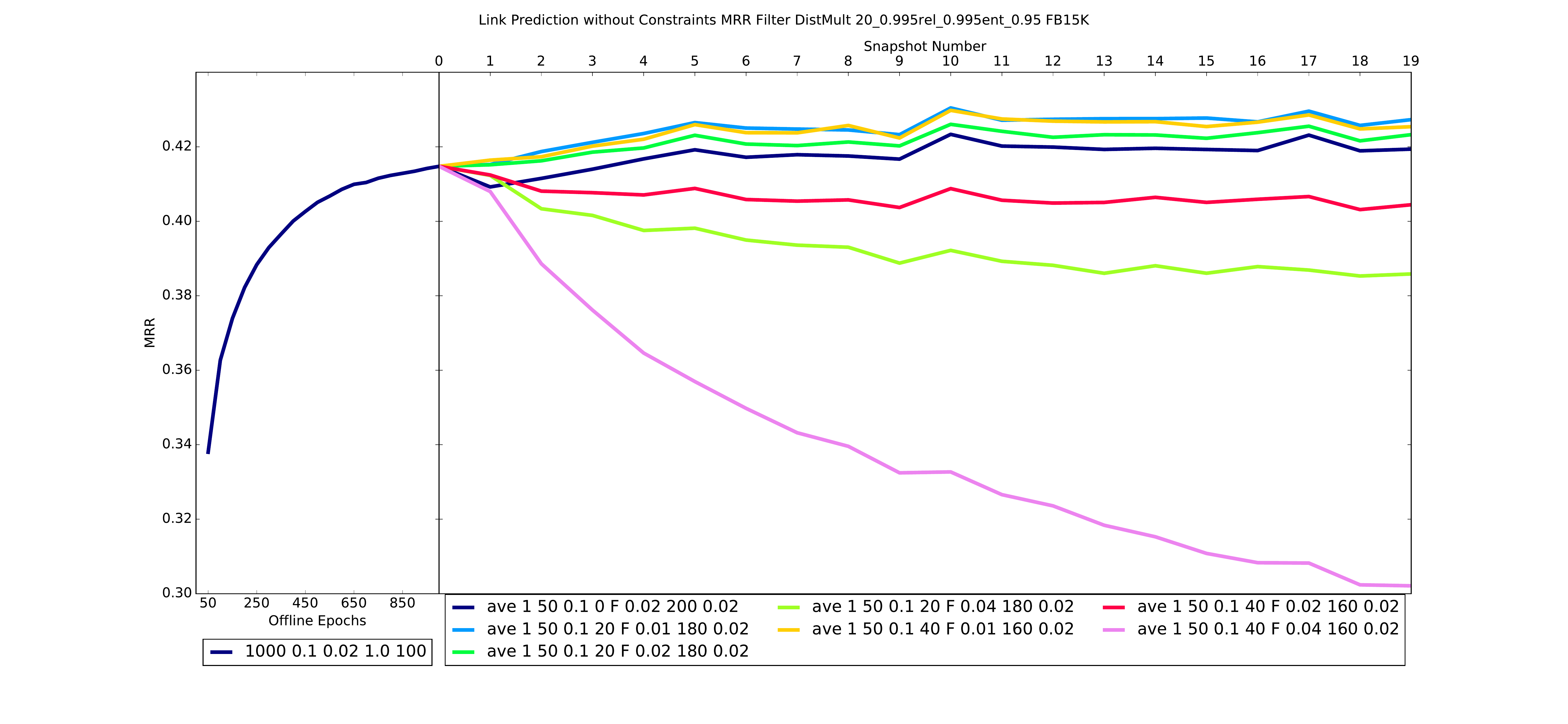}
\caption{Link prediction performance in the MRR metric of change specific epochs for \textsc{DistMult}} \label{DistMult_FB15K_changespecs_MRR}
\end{figure}
For both, \textsc{TransE} and \textsc{DistMult}, we can see that a too large number (e.g. $40$) of change specific epochs as well as a relatively high learning rate (twice the learning rate of general epochs) decreases the performance in comparison to the configuration without any change specific epochs.
This is due to the fact that the change specific epochs already have a really strong influence on the embedding, because they consider all additions and deletions at once instead of separating them into batches as done for general epochs.
However, if we choose only $20$ change specific epochs with half the learning rate of the general ones, the link prediction results are improving only slightly but consistently.
In the case of \textsc{DistMult}, this configuration even increases the MRR score of the initial embedding, while the setting without change specific epochs just keeps it more or less constant.
Although we replace $20$ general epochs by $20$ change specific epochs with only half the learning rate, the normalized mean change for the $L_2$-distance given in 
figures~\ref{TransE_FB15K_changespecs_L2_Rel} to~\ref{DistMult_FB15K_changespecs_L2_Ent} increases in comparison to the setup without any change specific epochs\footnote{The NMC score for snapshot $0$ in the right side plots results from the backtracking to the best embedding according to the validation results. It describes the change between the embedding at the end of the offline method and the best intermediate result, which becomes the initial embedding for the online method. Therefore, this score is excluded from any following calculations of arithmetic means, standard deviations or proportions.}.
\begin{figure}
\includegraphics[clip, trim=5cm 0.5cm 5.8cm 0.5cm, width=\textwidth]{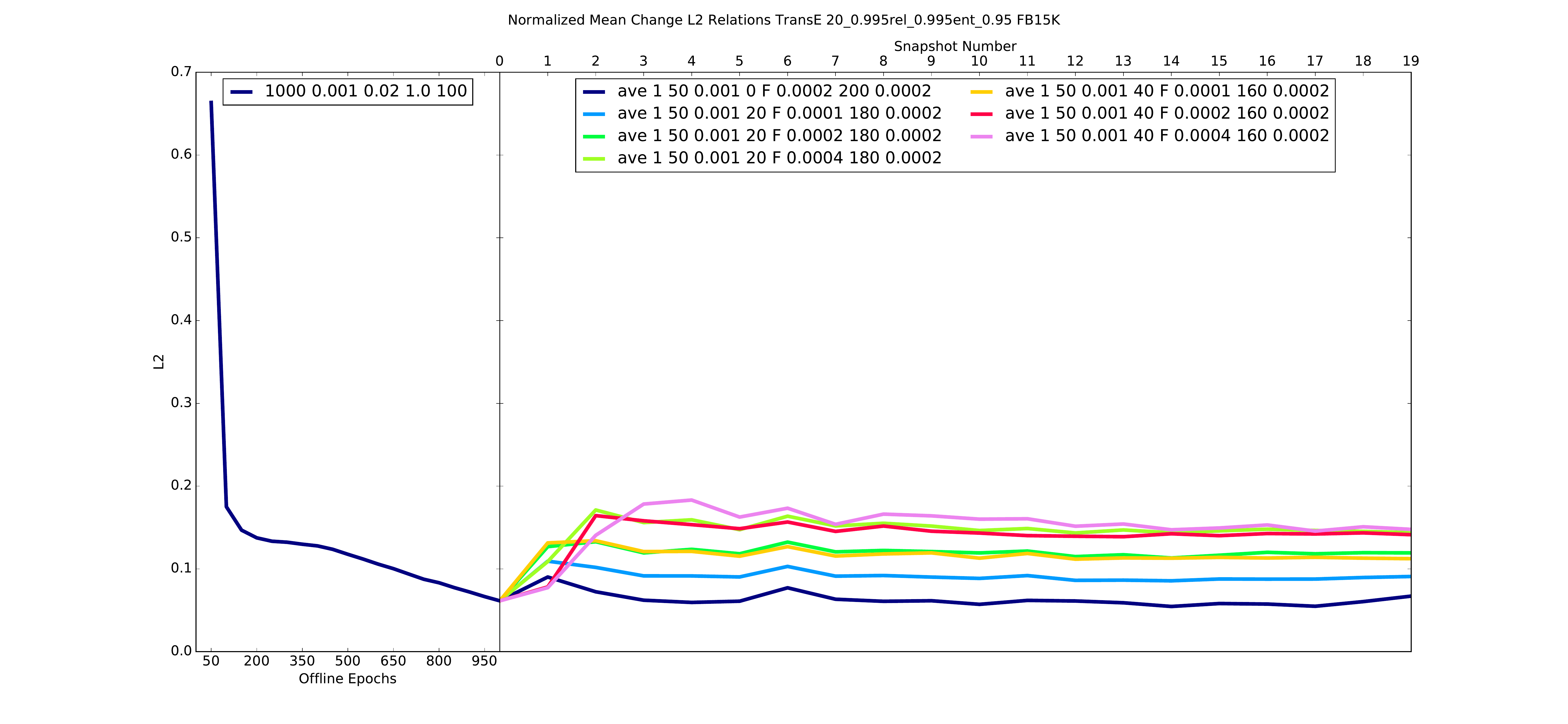}
\caption[Normalized Mean Change for relations using change specific epochs for \textsc{TransE} on FB15K]{Normalized Mean Change with the $L_2$-distance for relations using change specific epochs for \textsc{TransE} on %
the synthetic dynamic KG of FB15K}
\label{TransE_FB15K_changespecs_L2_Rel}
\end{figure}
\begin{figure}
\includegraphics[clip, trim=5cm 0.5cm 5.8cm 0.5cm, width=\textwidth]{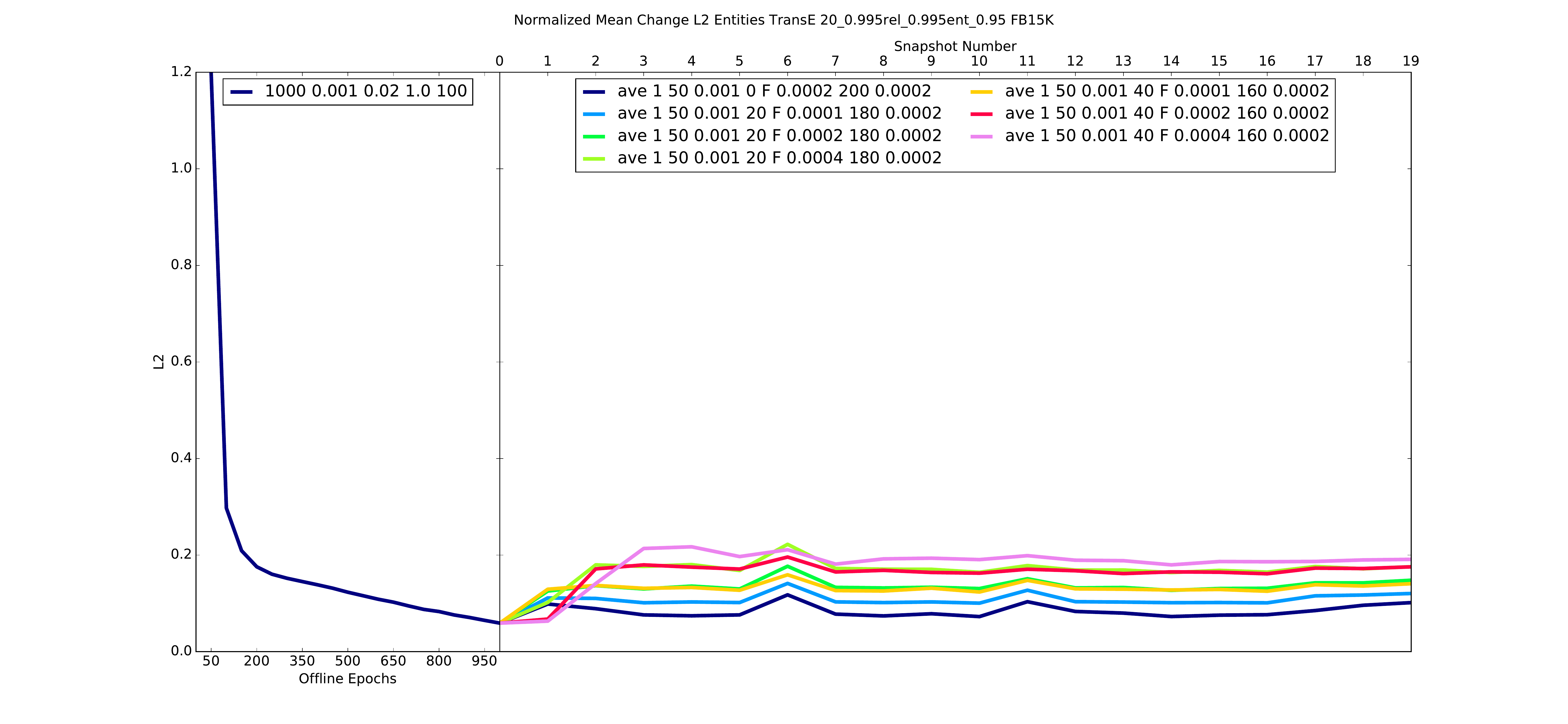}
\caption[Normalized Mean Change for entities using change specific epochs for \textsc{TransE} on FB15K]{Normalized Mean Change with the $L_2$-distance for entities using change specific epochs for \textsc{TransE} on %
the synthetic dynamic KG of FB15K}
\label{TransE_FB15K_changespecs_L2_Ent}
\end{figure}
\begin{figure}
\includegraphics[clip, trim=5cm 0.5cm 5.8cm 0.5cm, width=\textwidth]{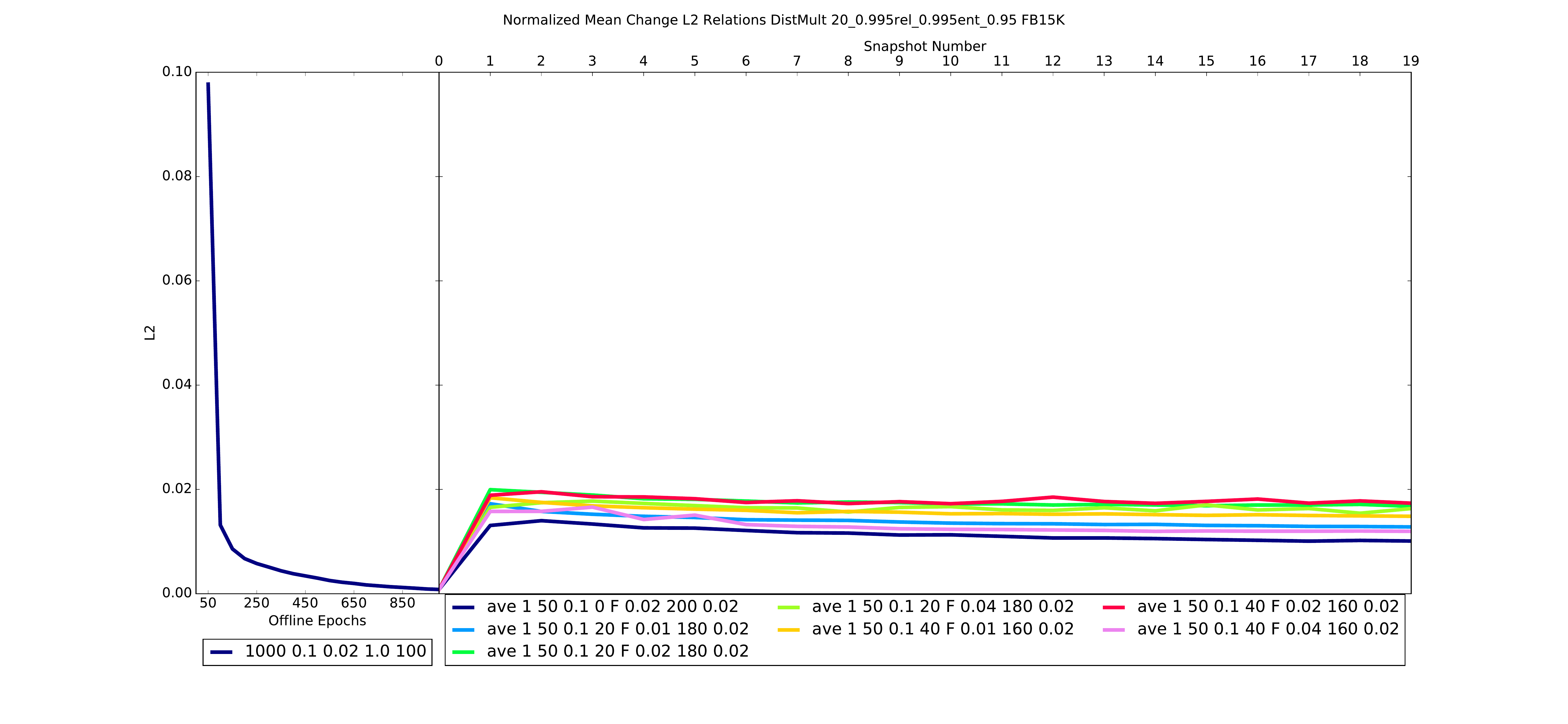}
\caption[Normalized Mean Change for relations using change specific epochs for \textsc{DistMult} on FB15K]{Normalized Mean Change with the $L_2$-distance for relations using change specific epochs for \textsc{DistMult} on %
the synthetic dynamic KG of FB15K}
\label{DistMult_FB15K_changespecs_L2_Rel}
\end{figure}
\begin{figure}
\includegraphics[clip, trim=5cm 0.5cm 5.8cm 0.5cm, width=\textwidth]{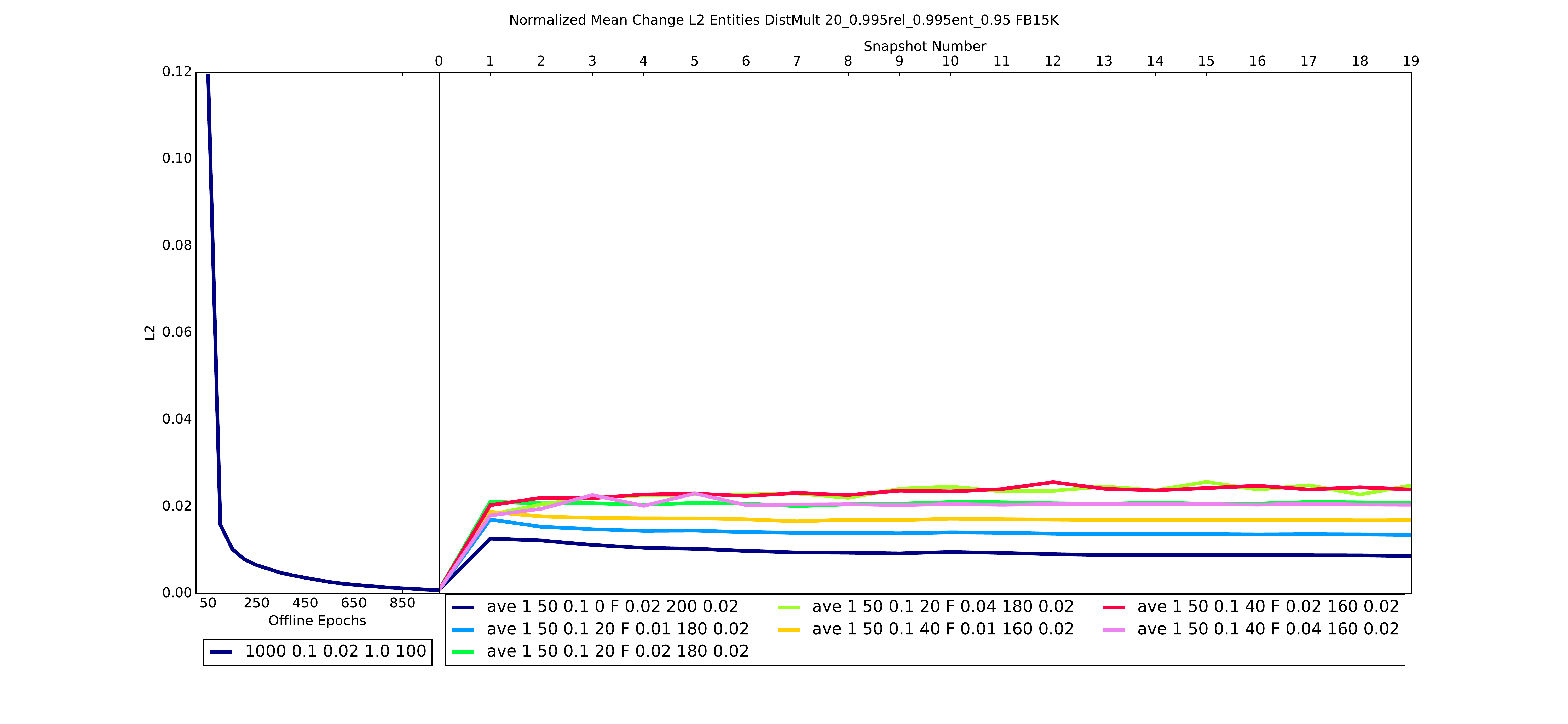}
\caption[Normalized Mean Change for entities using change specific epochs for \textsc{DistMult} on FB15K]{Normalized Mean Change with the $L_2$-distance for entities using change specific epochs for \textsc{DistMult} on %
the synthetic dynamic KG of FB15K}
\label{DistMult_FB15K_changespecs_L2_Ent}
\end{figure}
A possible explanation is that the change of the synthetic dynamic KG for FB15K (i.e. the change of the training set given in table~\ref{FB15K_20_0.995rel_0.995ent_0.95}) is much larger than one batch (ca. $4{,}757$ triples for a separation into $100$ batches) of a general epoch leading to a larger movement of the vector representations despite the relatively small learning rate.
However, since the impact of the change specific epochs on the link prediction performance as well as on the stability highly depends on the scale of change in the dynamic KG, these results just indicate the natural trade-off between embedding quality and stability, if the KG is changing substantially.
All in all, we consider the specified configuration with $20$ change specific epochs and half the learning rate of the general ones to be the best one, which will be used for the evaluation on the real data sets in the next sections.

\subsection{Computational Resources}
To run each of the experiments, we made use of a single node of a cluster. The servers we used have 2 Intel Xeon Platinum 8160 Processors (SkyLake --  2.1 GHz, 24 cores each totalling 48 cores) and 192 GB main memory. 
Further, we made use of a NVIDIA Volta 100 GPU with 16 GB HBM2 memory.
The servers run CentOS Linux 7 with kernel version 3.10.0-1160.15.2.el7.x86\_64.

\subsection{Results}
\subsubsection{Stability}\label{sec:stability_results}
Besides the measurement of the embedding quality, we employ the NMC score for an impression of the stability.
\paragraph{NELL}
The results for the embedding stability of our approach are given by the NMC scores (using the $L_2$-distance) separately for entities and relations as well as for translational distance and semantic matching models in \ref{TDs_NELL5_final_L2_Rel} to~\ref{SMs_NELL5_final_L2_Ent}.

\begin{figure}
\includegraphics[clip, trim=5cm 0.5cm 5.8cm 0.5cm, width=\textwidth]{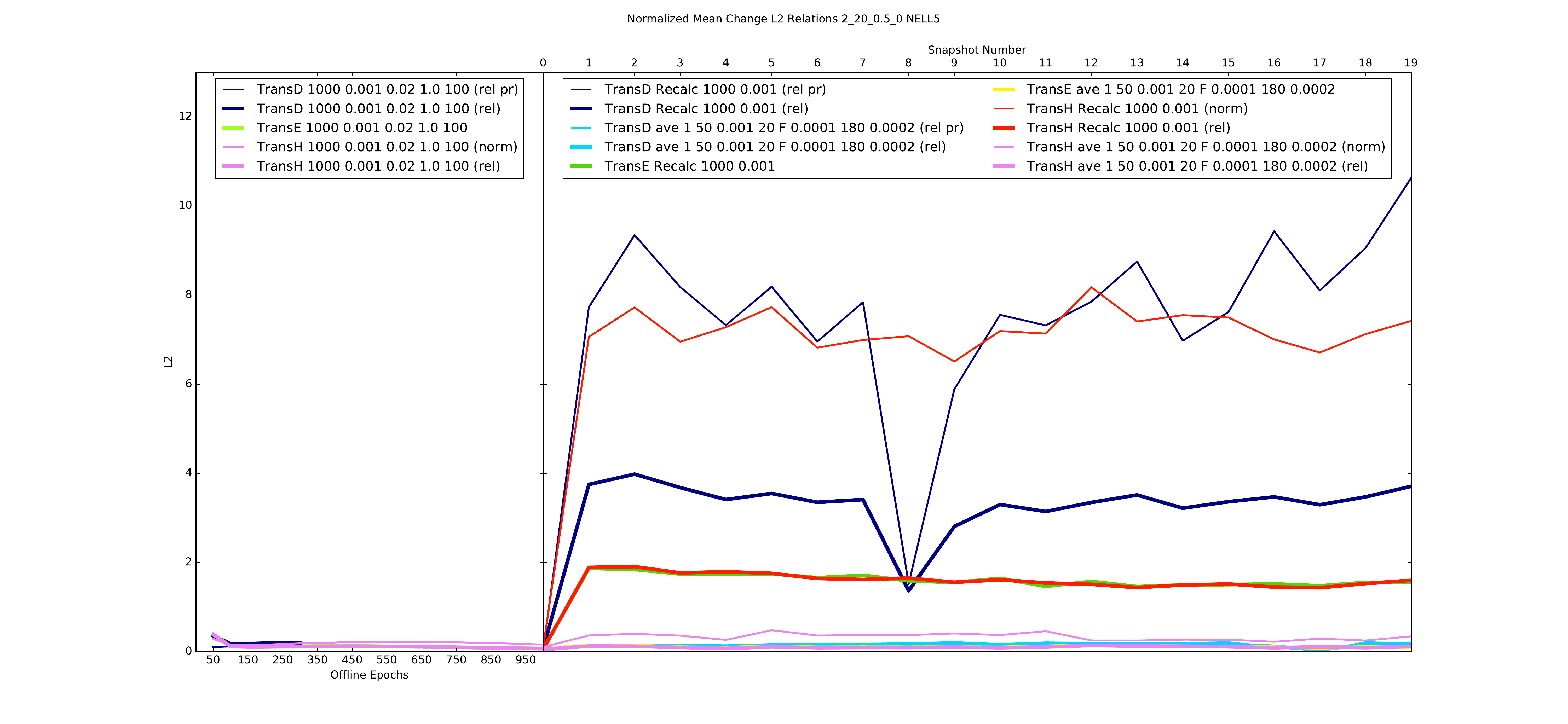}
\caption[Normalized Mean Change for relations using the online method in comparison to recalculations for translational distance models on NELL5]{Normalized Mean Change with the $L_2$-distance for relations using the online method in comparison to recalculations for translational distance models on %
NELL5}
\label{TDs_NELL5_final_L2_Rel}
\end{figure}
\begin{figure}
\includegraphics[clip, trim=5cm 0.5cm 5.8cm 0.5cm, width=\textwidth]{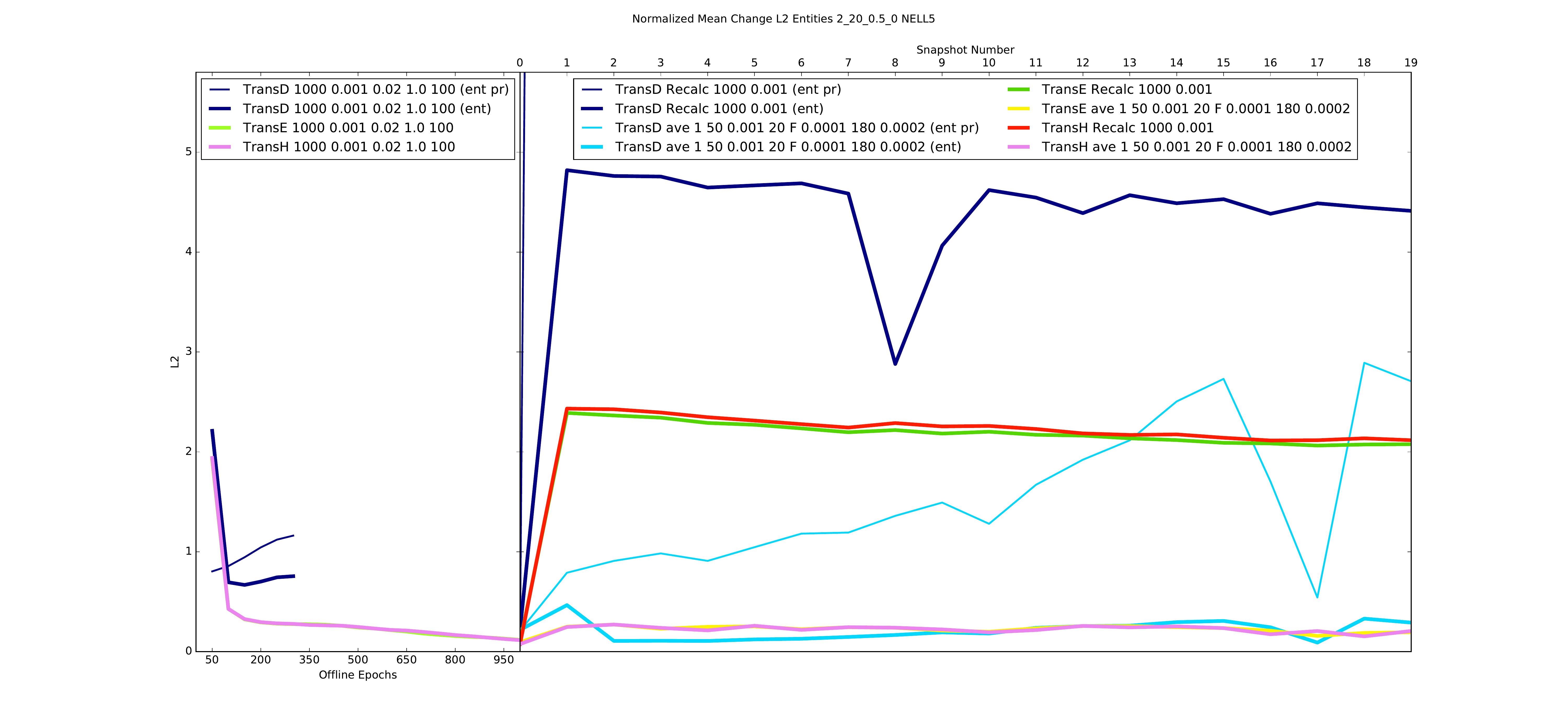}
\caption[Normalized Mean Change for entities using the online method in comparison to recalculations for translational distance models on NELL5]{Normalized Mean Change with the $L_2$-distance for entities using the online method in comparison to recalculations for translational distance models on %
NELL5. The NMC scores of the projection vectors for the recalculated embeddings of \textsc{TransD} have a much larger scale of $55.959$ on average with a standard deviation of $13.058$.}
\label{TDs_NELL5_final_L2_Ent}
\end{figure}
\begin{figure}
\includegraphics[clip, trim=5cm 0.5cm 5.8cm 0.5cm, width=\textwidth]{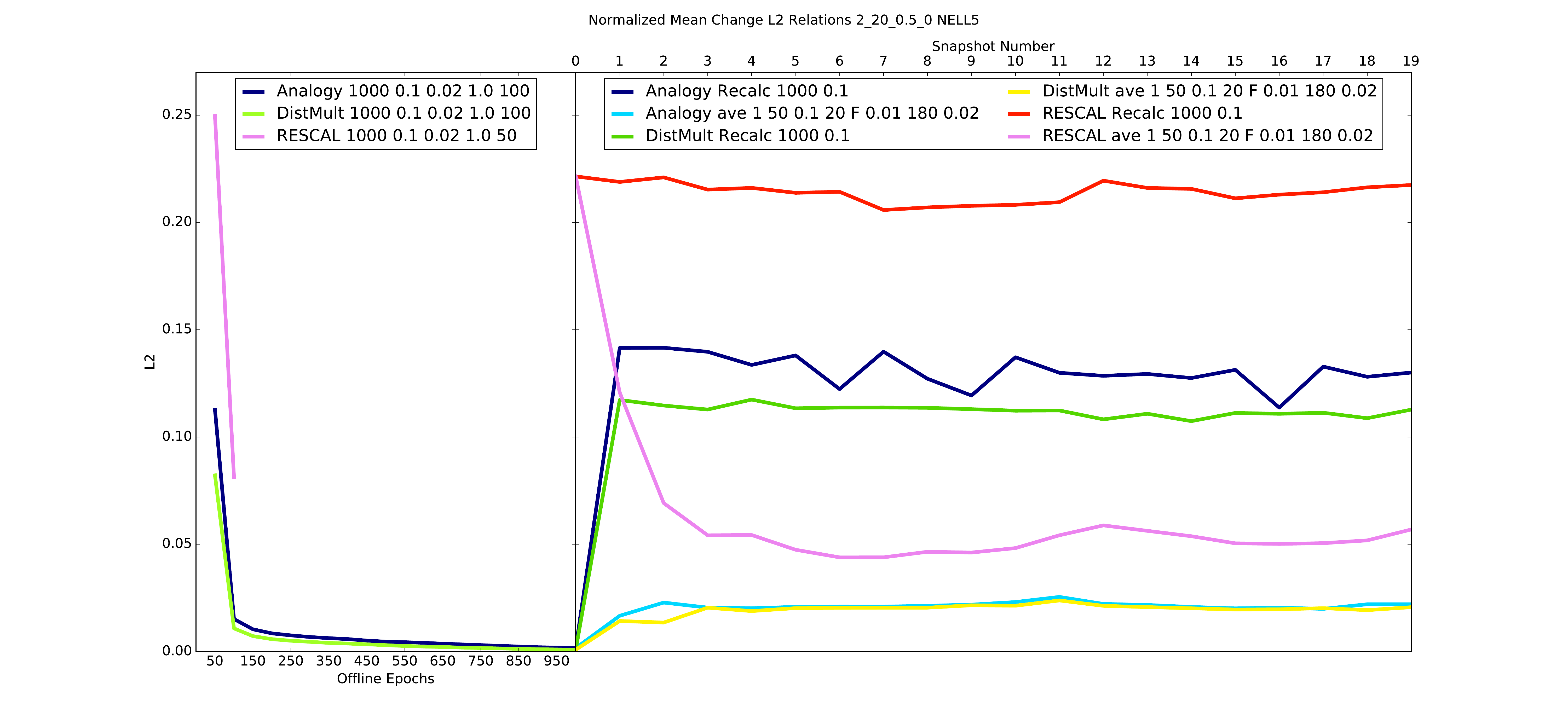}
\caption[Normalized Mean Change for relations using the online method in comparison to recalculations for semantic matching models on NELL5]{Normalized Mean Change with the $L_2$-distance for relations using the online method in comparison to recalculations for semantic matching models on %
NELL5}
\label{SMs_NELL5_final_L2_Rel}
\end{figure}
\begin{figure}
\includegraphics[clip, trim=5cm 0.5cm 5.8cm 0.5cm, width=\textwidth]{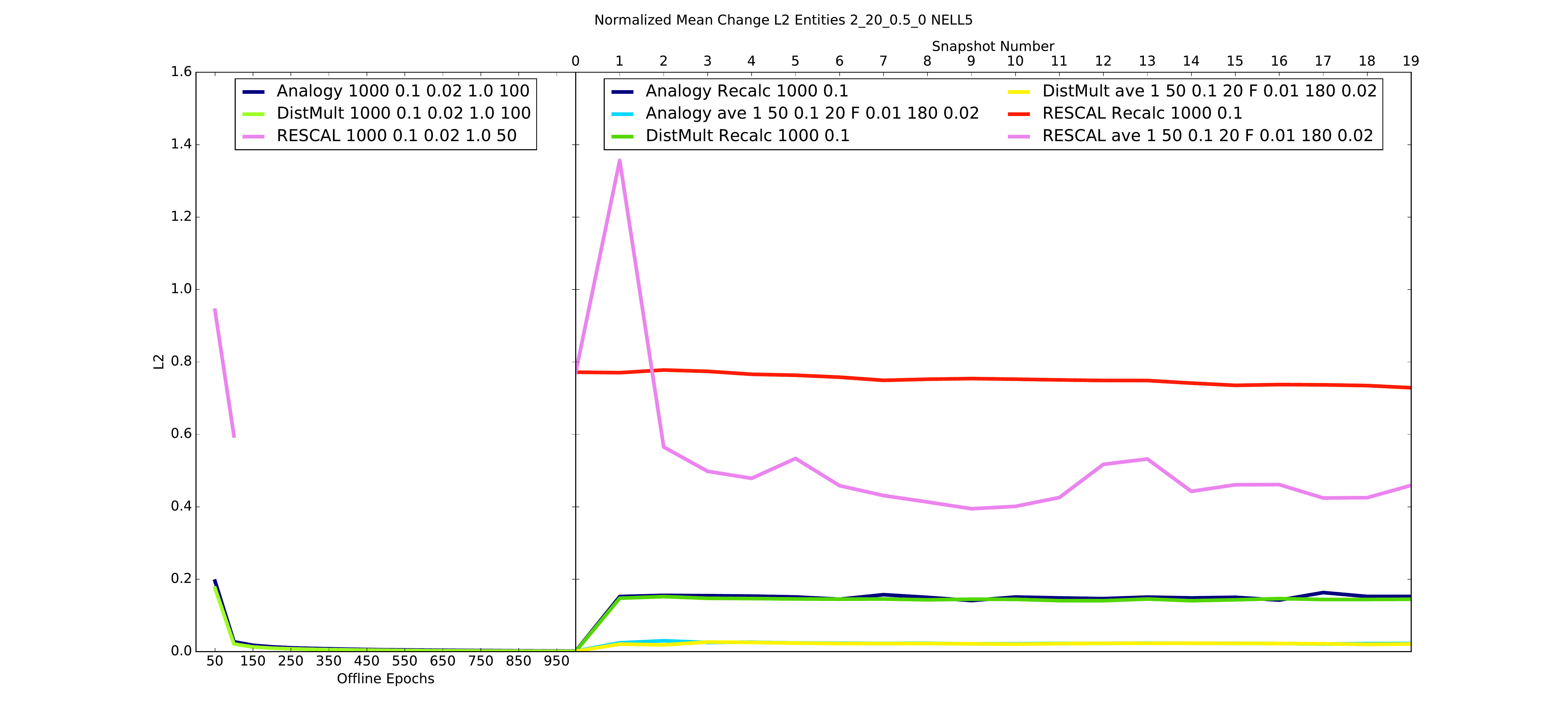}
\caption[Normalized Mean Change for entities using the online method in comparison to recalculations for semantic matching models on NELL5]{Normalized Mean Change with the $L_2$-distance for entities using the online method in comparison to recalculations for semantic matching models on %
NELL5}
\label{SMs_NELL5_final_L2_Ent}
\end{figure}

For the translational distance models we observe large scale differences between the NMC scores of the online method and the ones of recalculated embeddings.
In contrast to our approach, which continues the training with almost the same NMC score as in the final training phase of the initial embedding, simple recalculations result in average values of more than $10$ (for entities) and $15$ (for relations) times the scale for \textsc{TransE} and \textsc{TransH} and even $21$ and $22$ times the scale for \textsc{TransD}\footnote{We only focus on the stability for the actual entity and relation representations, not on the norm vectors of hyperplanes for \textsc{TransH} or the projection vectors for \textsc{TransD}.}.
This is not the case for the semantic matching models.
Although the online method still provides a better stability, the recalculated embeddings reach ``only'' $6$ times higher NMC scores for \textsc{DistMult} and \textsc{Analogy} and an even smaller proportion for \textsc{RESCAL}.
Furthermore, the overall values of this metric are smaller than for the translational distance models indicating less movement of the vectors during training.
\paragraph{MathOverflow}
Looking at the results for the embedding stability in terms of NMC scores (again with the $L_2$-distance) given in figures~\ref{TDs_MathOverflow_final_L2_Rel} to~\ref{SMs_MathOverflow_final_L2_Ent}, we can observe an overall higher stability of relation representations for translational distance models in comparison to the results for NELL5.
\begin{figure}
\includegraphics[clip, trim=5cm 0.5cm 5.8cm 0.5cm, width=\textwidth]{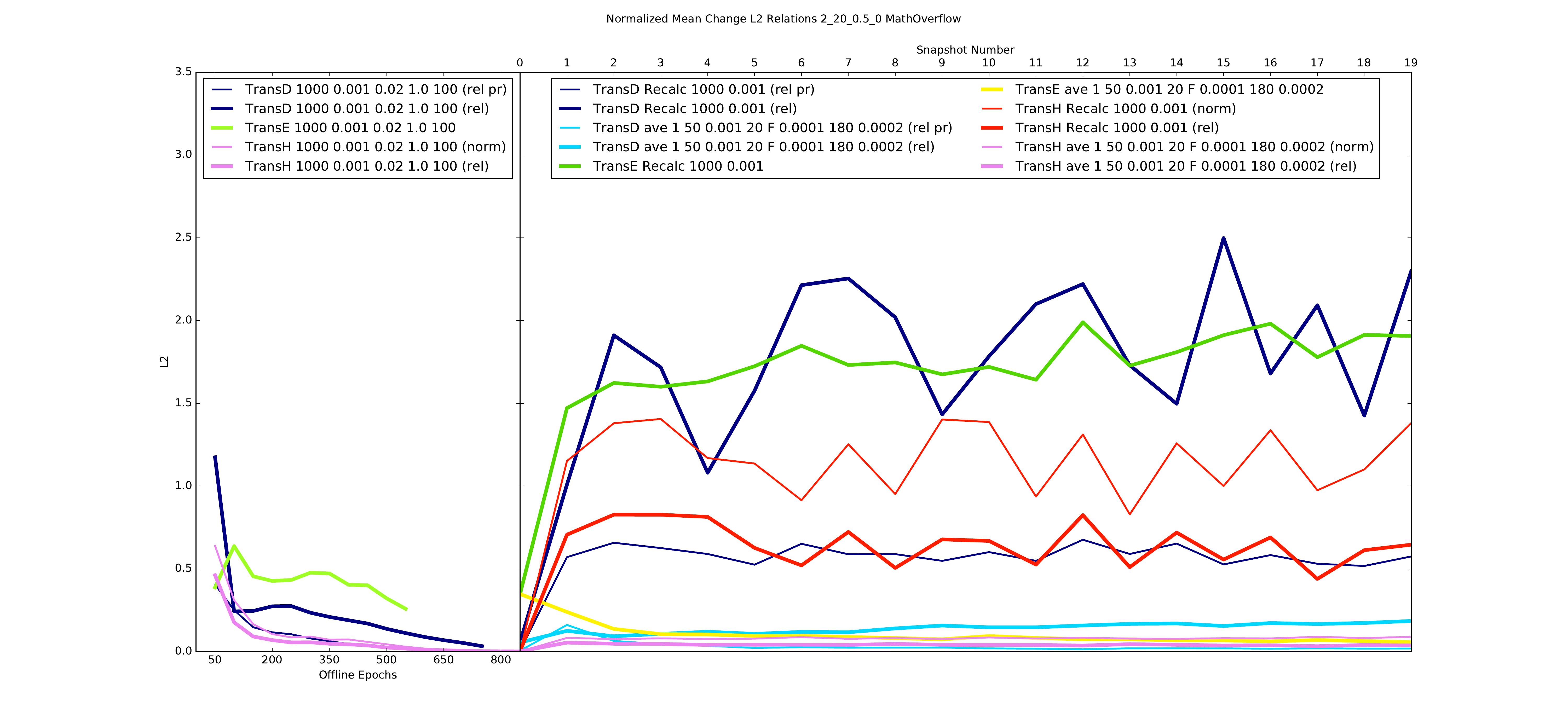}
\caption[Normalized Mean Change for relations using the online method in comparison to recalculations for translational distance models on MathOverflow]{Normalized Mean Change with the $L_2$-distance for relations using the online method in comparison to recalculations for translational distance models on %
MathOverflow}
\label{TDs_MathOverflow_final_L2_Rel}
\end{figure}
\begin{figure}
\includegraphics[clip, trim=5cm 0.5cm 5.8cm 0.5cm, width=\textwidth]{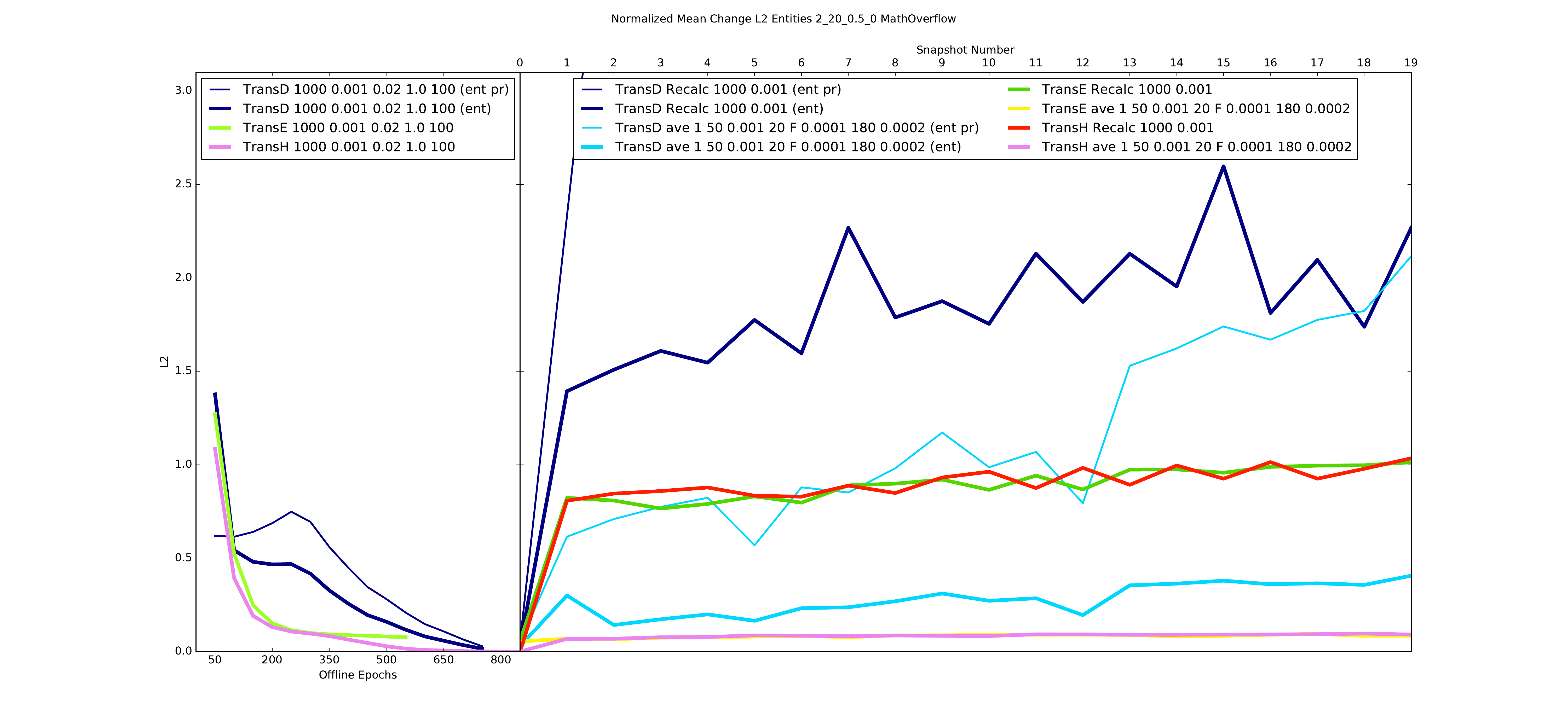}
\caption[Normalized Mean Change for entities using the online method in comparison to recalculations for translational distance models on MathOverflow]{Normalized Mean Change with the $L_2$-distance for entities using the online method in comparison to recalculations for translational distance models on %
MathOverflow. The NMC scores of the projection vectors for the recalculated embeddings of \textsc{TransD} have a much larger scale of $12.152$ on average with a standard deviation of $8.423$.}
\label{TDs_MathOverflow_final_L2_Ent}
\end{figure}
\begin{figure}
\includegraphics[clip, trim=5cm 0.5cm 5.8cm 0.5cm, width=\textwidth]{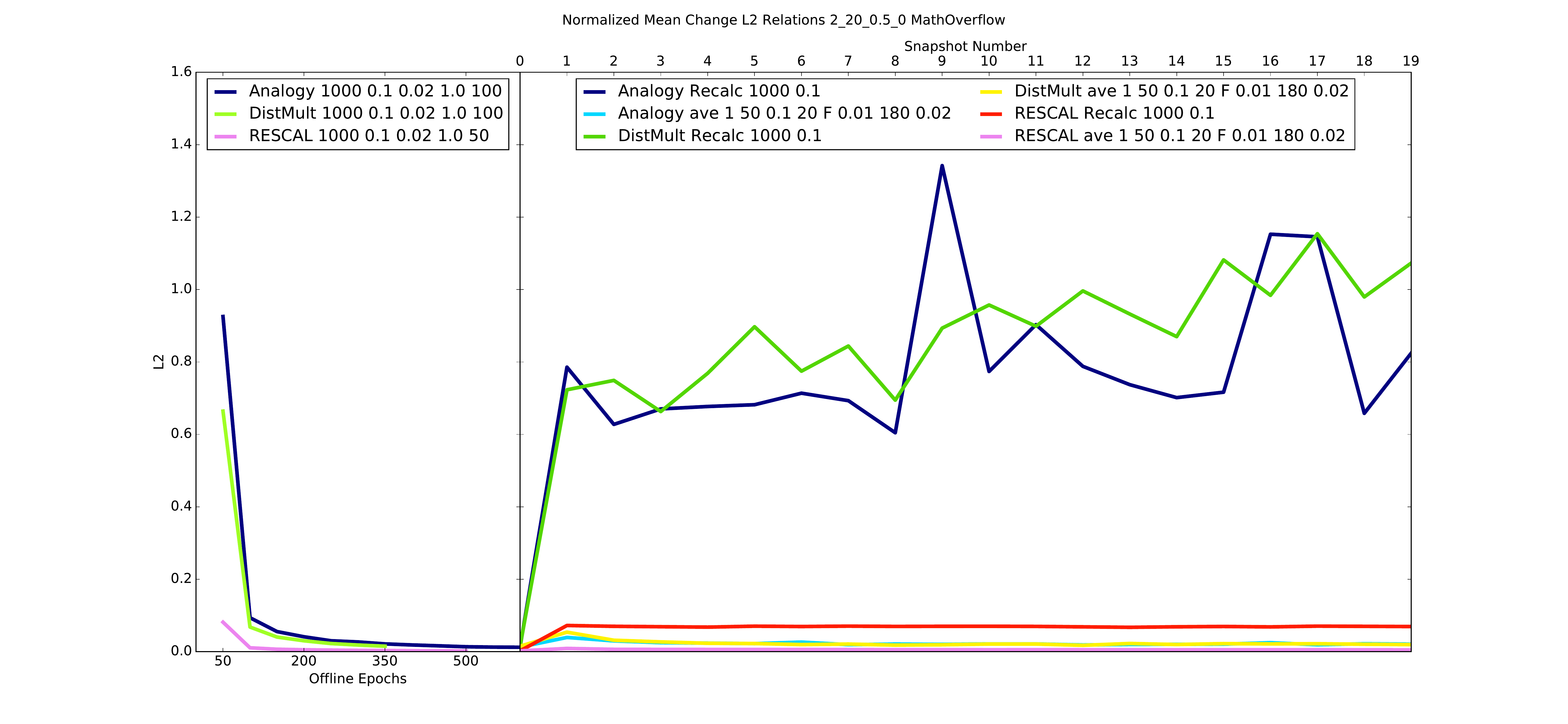}
\caption[Normalized Mean Change for relations using the online method in comparison to recalculations for semantic matching models on MathOverflow]{Normalized Mean Change with the $L_2$-distance for relations using the online method in comparison to recalculations for semantic matching models on %
MathOverflow}
\label{SMs_MathOverflow_final_L2_Rel}
\end{figure}
\begin{figure}
\includegraphics[clip, trim=5cm 0.5cm 5.8cm 0.5cm, width=\textwidth]{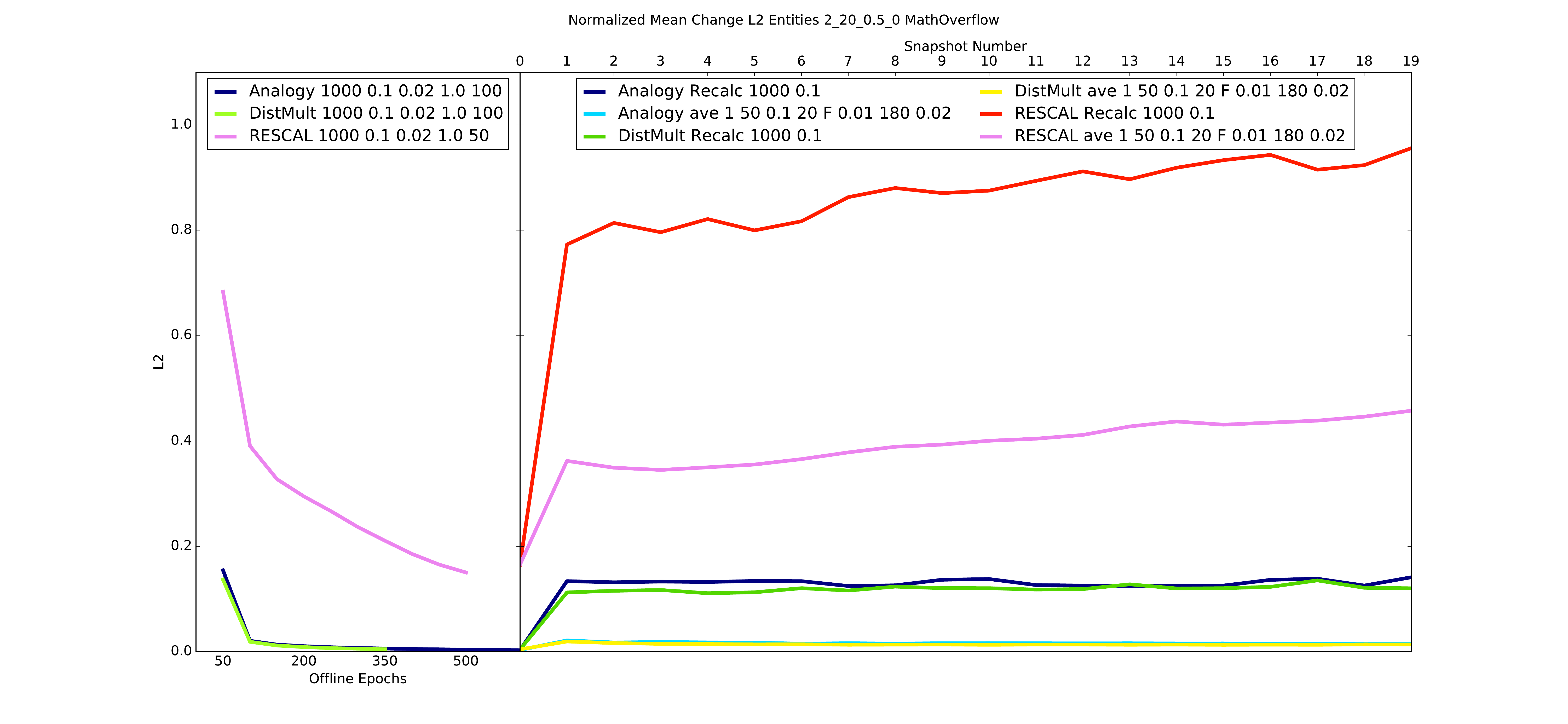}
\caption[Normalized Mean Change for entities using the online method in comparison to recalculations for semantic matching models on MathOverflow]{Normalized Mean Change with the $L_2$-distance for entities using the online method in comparison to recalculations for semantic matching models on %
MathOverflow}
\label{SMs_MathOverflow_final_L2_Ent}
\end{figure}
This can be explained with the different numbers of relations in the two data sets.
While there are $739$ relations in NELL5, the vector representations of the only three relations in MathOverflow are much more constrained leading to a smaller change over time.
As for NELL5, the updated embeddings of the translational distance models evolve almost stable, while the recalculated ones reach $11$ (for entities) and $19$ (for relations) times higher NMC scores for \textsc{TransE}, $11$ and $16$ as the proportions for \textsc{TransH} and $7$ and $13$ times higher ones for \textsc{TransD}.
For the semantic matching models, our online method provides the same stability on MathOverflow than on NELL5, but the proportions to the NMC scores of the recalculated embeddings are completely different for the relations this time.
For \textsc{DistMult} and \textsc{Analogy} the recalculated embeddings have on average $39$ and $35$ times higher NMC scores for relations than the updated embeddings.
This effect appears again due to the small number of relations in the MathOverflow data set.
Since all three relations are extreme many-to-many relations, their embeddings are constrained to be similar leading to a very small normalization factor in the NMC calculation.
Combined with a substantial distance between old and new vector representations, this results in very high NMC scores.

To evaluate the last one of our four desired properties, the scalability, we use the large Higgs Twitter data set and compare the runtimes besides the embedding quality.

\subsubsection{Scalability}\label{results_scalability}
We have an individual look at the Higgs Twitter dataset for the evaluation in terms of scalability.
Because of the large scale of the Higgs Twitter set, we had to make some adjustments to the general learning procedure to avoid exceeding memory and time limits.
\paragraph*{Adjusted experimental setup}
The training, validation and test sets include on average $4{,}694{,}363$, $260{,}790$ and $259{,}978$ triples.
Furthermore, each snapshot contains $277{,}002$ entities on average.
This means, to validate the current state of the embedding on the validation set in terms of link prediction, the scoring function has to be evaluated for about $2\cdot 260{,}790\cdot 277{,}002 \approx 1.445\cdot 10^{11}$ triples, since the head and the tail of each validation triple are replaced with each other entity in the snapshot.
As a result, the validation together with the test of the final embedding exceeds our time limits.
We decided to cancel the validation and therefore early stopping so that the online method runs for the complete number of epochs.
To reduce the memory requirements, we also had to increase the number of batches for the general epochs to $200$ instead of $100$.
Since \textsc{RESCAL} represents relations as matrices, we would further need to lower the batch size for this embedding model to reduce memory requirements, but then again the resulting runtime would exceed time limits.
Hence, we do not provide results for \textsc{RESCAL} on the Higgs Twitter data set.
Finally, as our stability metric requires the computation of a pairwise distance matrix between all entity embeddings, the calculation of the NMC scores exceeded memory limits, although the matrix is calculated and processed in chunks.
However, as we have already evaluated the stability on NELL5 and MathOverflow, we are more interested in the scalability of the online method in terms of embedding quality and especially the runtime.

\paragraph*{Embedding Quality}
All explained metrics for the measurement of the link prediction performance depend on the size of the given KG, because it is just more unlikely to calculate a small rank for a test triple in a large-scale data set than in a smaller one (especially the number of entities is responsible for this effect).
Hence, to obtain comparable results, we use the Hits@100 metric instead of the previously employed MRR.
The results of the translational distance models and semantic matching models (except \textsc{RESCAL}) are given in \cref{TDs_Higgs_final_Hits@100,SMs_Higgs_final_Hits@100}.
\begin{figure}
\includegraphics[clip, trim=5cm 0.5cm 5.8cm 0.5cm, width=\textwidth]{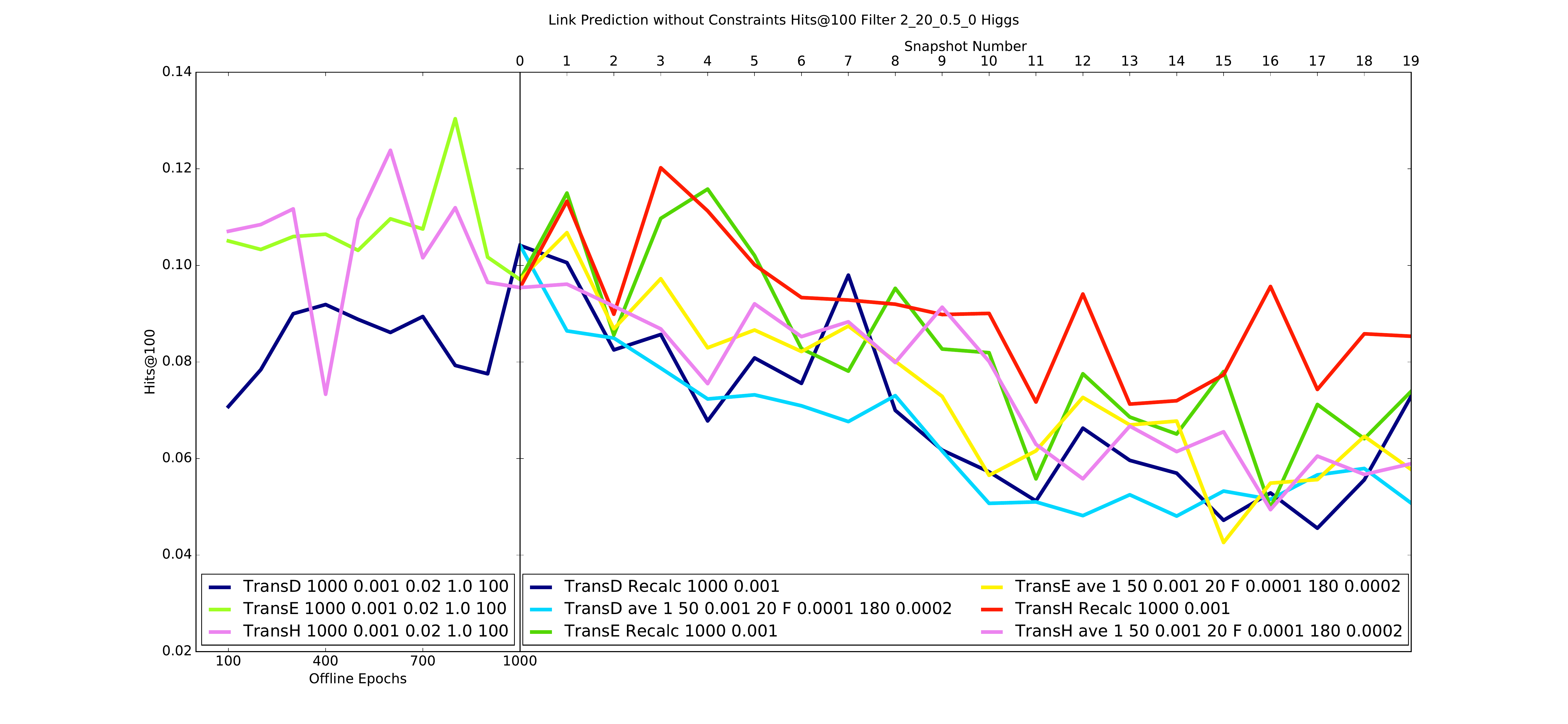}
\caption[Link prediction performance in the Hits@100 metric of the online method in comparison to recalculations for translational distance models on Higgs Twitter]{Link prediction performance in the Hits@100 metric of the online method in comparison to recalculations for translational distance models on %
Higgs Twitter}
\label{TDs_Higgs_final_Hits@100}
\end{figure}
\begin{figure}
\includegraphics[clip, trim=5cm 0.5cm 5.8cm 0.5cm, width=\textwidth]{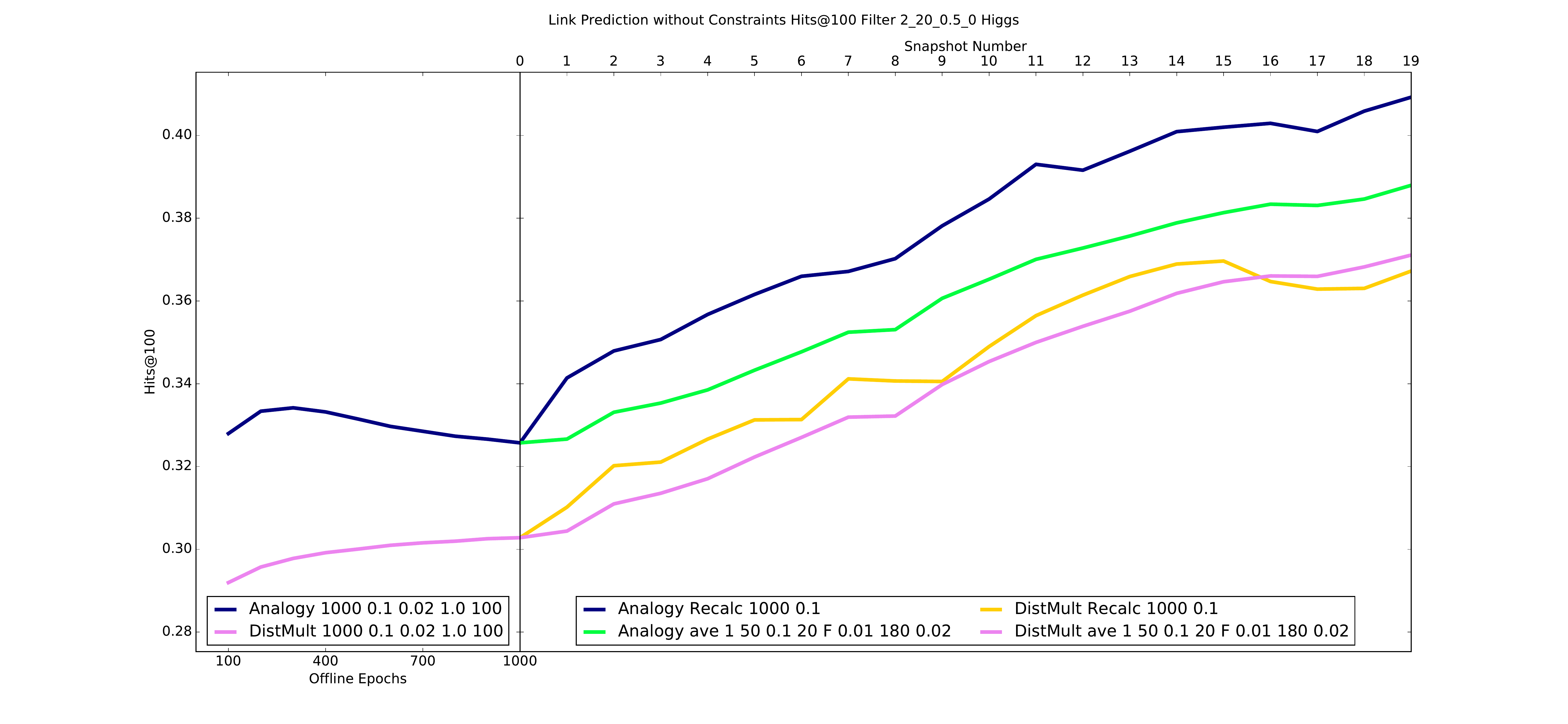}
\caption[Link prediction performance in the Hits@100 metric of the online method in comparison to recalculations for semantic matching models on Higgs Twitter]{Link prediction performance in the Hits@100 metric of the online method in comparison to recalculations for semantic matching models (except \textsc{RESCAL}) on %
Higgs Twitter}
\label{SMs_Higgs_final_Hits@100}
\end{figure}

First of all, we can observe huge performance differences between the translational distance and the semantic matching models.
On average, the recalculated embeddings of \textsc{DistMult} and \textsc{Analogy} achieve about $3.8 - 4.9$ and $4.2 - 5.4$ times higher Hits@100 proportions than \textsc{TransE}, \textsc{TransH} and \textsc{TransD}.
This could be due to the huge underlying directed network consisting of the follower relation. 
Furthermore, the embedding quality for the translational distance models is fluctuating from snapshot to snapshot, whereas for the semantic matching models there is a clearly visible positive trend due to the change of the dynamic KG. 
As before in the case of MathOverflow, the online method for \textsc{DistMult} and especially for \textsc{Analogy} is not completely able to catch up with the improving link prediction results of the recalculated embeddings. The bigger difference for \textsc{Analogy} could have been prevented by using validations and therefore backtracking to the best intermediate result for the initial embedding, as can be seen from the decreasing Hits@100 proportion during the later offline epochs.
For \textsc{DistMult} the average Hits@100 score of the recalculated embeddings is just $1.47\%$ better than for the online method.
Hence, we draw the conclusion that our approach is able to maintain comparable embedding quality for large-scale dynamic KGs.
Finally, we have a look at the runtime of the online method.

\paragraph*{Runtime}\label{sec:runtime}
In theory, the online method could be at most five times faster than the recalculations, as we use only a fifth of the epochs in the offline method.
These epochs mainly determine the runtime as already explained in the analysis of the time complexity.
For each snapshot we documented the runtime of the online method as well as of the recalculations without the time for the switch to the new snapshot or for the final link prediction test.
The results in terms of average runtime per snapshot and standard deviation are given in \cref{tab:higgs_runtimes}.

\begin{table}[!h]
\centering
\caption[Runtime of online method and recalculations on Higgs Twitter]{Average runtime $\varnothing$ and deviation $\sigma$ in seconds of the online method in comparison to recalculations for \textsc{TransE}, \textsc{TransH}, \textsc{TransD}, \textsc{DistMult} and \textsc{Analogy} on Higgs Twitter. The embeddings are trained by the NVIDIA Tesla V100 GPU with CUDA v.9.0 (driver 418.43) and cuDNN v.7.5.} \label{tab:higgs_runtimes}
\begin{tabular}{|c||cc|cc|}
\cline{2-5}
\multicolumn{1}{c|}{} & \multicolumn{2}{c|}{\textbf{Online Method}} & \multicolumn{2}{c|}{\textbf{Recalculations}} \\
\cline{2-5}
\multicolumn{1}{c|}{} & $\varnothing$ & $\sigma$ & $\varnothing$ & $\sigma$ \\
\hline
\textsc{TransE} & $1{,}981$ & $380$ & $6{,}135$ & $70$ \\
\hline
\textsc{TransH} & $3{,}254$ & $373$ & $13{,}122$ & $136$ \\
\hline
\textsc{TransD} & $3{,}056$ & $339$ & $12{,}053$ & $152$ \\
\hline
\textsc{DistMult} & $2{,}052$ & $307$ & $6{,}408$ & $94$ \\
\hline
\textsc{Analogy} & $4{,}090$ & $334$ & $16{,}698$ & $239$ \\
\hline
\end{tabular}
\end{table}

For the faster models \textsc{TransE} and \textsc{DistMult} the online method is three times faster than the recalculations.
While the other three embedding models \textsc{TransH}, \textsc{TransD} and \textsc{Analogy} took at most $4$ hours and $38$ minutes to recalculate an embedding from scratch, the online method is even four times faster with a maximum runtime of $1$ hour and $8$ minutes.
The larger standard deviation mainly results from an up to $15$ minutes slower runtime for the first snapshot because of hardware reasons, but also from variation in the changes of the dynamic KG.

Overall the online method is scalable to large real-world dynamic KGs and even saves a lot of time in comparison to recalculations.

\bibliographystyle{splncs04}
\bibliography{bibliography}